\documentclass[a4paper,11pt]{article}
\usepackage[left=30mm,top=30mm,right=30mm,bottom=30mm]{geometry}
\usepackage{etoolbox} %
\usepackage{booktabs}
\usepackage[table,xcdraw]{xcolor}
\usepackage[usestackEOL]{stackengine}
\usepackage[round]{natbib}
\usepackage[T1]{fontenc}
\usepackage[utf8]{inputenc}
\usepackage{bm}
\usepackage{bbm}
\usepackage{graphicx}
\usepackage{subcaption}
\usepackage{amsmath}
\usepackage{amsfonts}
\usepackage{mathtools}
\usepackage{xcolor}
\usepackage{float}
\usepackage{flafter}
\usepackage{hyperref}
\usepackage{authblk}
\usepackage[capitalise]{cleveref}
\usepackage{enumitem}
\usepackage{amssymb}
\usepackage{amsbsy}
\usepackage{amsthm}         %
\usepackage{caption}

\usepackage{algpseudocode}
\usepackage[ruled,vlined,linesnumbered]{algorithm2e}

\newtheorem{innercustomthm}{Theorem}
\newenvironment{customthm}[1]
  {\renewcommand\theinnercustomthm{#1}\innercustomthm}
  {\endinnercustomthm}

\newtheorem{innercustomlem}{Lemma}
\newenvironment{customlem}[1]
  {\renewcommand\theinnercustomlem{#1}\innercustomlem}
  {\endinnercustomlem}

\newtheorem{innercustomprop}{Proposition}
\newenvironment{customprop}[1]
  {\renewcommand\theinnercustomprop{#1}\innercustomprop}
  {\endinnercustomprop}

  \newtheorem{innercustomcorollary}{Corollary}
\newenvironment{customcorollary}[1]
  {\renewcommand\theinnercustomcorollary{#1}\innercustomcorollary}
  {\endinnercustomcorollary}

\newtheorem{theorem}{Theorem}
\newtheorem{lemma}{Lemma}
\newtheorem*{lemma*}{Lemma}
\newtheorem{proposition}{Proposition}

\newtheorem{corollary}{Corollary}
\newtheorem{definition}{Definition}

\newtheorem*{example*}{Example}

\newtheorem*{assumption*}{Assumption}
\newtheorem*{assumptions*}{Assumptions}

\newtheorem{model}{Model}

\newcommand{\indep}{\perp\kern-5.5pt\perp}
\newcommand{\notindep}{\not\!\indep}

\newcommand{\toverset}[2]{\overset{{\rm #1}}{#2}}

\newcommand{\baseline}{{\rm baseline}}

\DeclareMathOperator{\RV}{RV}
\DeclareMathOperator{\pa}{pa}
\DeclareMathOperator{\past}{past}
\DeclareMathOperator{\future}{future}

\hypersetup{
    colorlinks,
    breaklinks,
    linkcolor={black},
    citecolor={blue!50!black},
    urlcolor={blue!80!black},
    pdftitle={Granger Causality in Extremes},
    pdfauthor={Juraj Bodik and Olivier C. Pasche},
    pdfkeywords={Causal discovery, Granger causality, Extreme events, Time series, Structural causal models, Heavy tails}
}

\graphicspath{{Fig/}}

\newcommand{\citesupplement}{}

\title{\textbf{Granger Causality in Extremes}}

\author[1,2]{Juraj Bodik}
\author[3,4]{Olivier C. Pasche}

\affil[1]{Faculty of Business and Economics, University of Lausanne, Switzerland}
\affil[2]{Department of Statistics, UC Berkeley, California, USA}
\affil[3]{Research Institute for Statistics and Information Science, University of Geneva, Switzerland}
\affil[4]{Department of Industrial Engineering and Operations Research, Columbia University, New York, USA}

\date{}

\begin{document}

\pagenumbering{gobble}%

\maketitle
\begin{abstract}

Causal discovery in time series becomes especially important during extreme, highly volatile periods, yet state-of-the-art methods focus on causality within the body of the distribution and often overlook mechanisms that manifest only in extreme events. We propose a framework for Granger causality in extremes that infers causal links primarily from extreme events, using the causal tail coefficient. We establish equivalences between causality in extremes and other causal concepts, including (classical) Granger causality, Sims causality, and structural causality, under suitable assumptions. We prove other key properties of Granger causality in extremes and show that the framework is especially helpful in the presence of hidden confounders. Building on those properties, we propose a non-parametric inference method for detecting Granger causality in extremes from observational data; it handles non-linear and high-dimensional series, outperforms the considered state-of-the-art methods in almost all of our simulated settings, and uncovers interpretable structures in financial and extreme-weather applications. An open-source implementation of our methodology is provided.

\end{abstract}

\vspace*{6mm}

{%
}
    
\newpage

\pagenumbering{arabic}%

\section{Introduction}\label{s:intro}

Granger causality \citep{GRANGER1980329} is a widely employed statistical framework for formalizing causal relationships among two or more time series variables, used across a wide array of fields, including finance, economics, neuroscience, and climate science \citep{gujarati2009causality,Rubin,attanasio2013granger}.
It is well suited for empirical examinations of cause-and-effect associations, as it does not require the specification of a scientific model. 
However, Granger causality primarily measures the association between variables, and omitting relevant variables from the analysis can potentially lead to spurious causal inferences, which drew some criticism \citep{Granger_criticism}. 

A typical focus of causal methods is on the body of the distribution \citep[causality in the mean,][]{GRANGER1980329, PCalgorithm, Runge2019PCMCI}. However, many important causal questions lie beyond the range of observed values, where classical methods are not well suited. For example, what are the effects of extreme stock return \citep{Candelon2013} on other stocks, or of extreme precipitation \citep{barbero2018temperature} on river floods? With climate change driving more frequent extreme weather, understanding the impacts of extremes becomes more crucial. More generally, large interventions often differ from minor ones, and many causal mechanisms emerge only during extreme periods, beyond what causality in the mean can capture. Moreover, complex causal relationships often simplify in the extremes \citep{engelke2025extremesstructuralcausalmodels}, making them easier to study from that perspective.

In this work, we formally introduce Granger-type causality in extremes for the time series
\((\mathbf X,\mathbf Y)^\top=((X_t,Y_t)^\top,t\in\mathbb Z)\). We propose novel
definitions that characterize two forms of extreme causal effects, intuitively described as: 
\begin{itemize} %
    \item $X_t$ being extreme ``increases the probability of'' $Y_{t+i}$ being extreme,
    \item $X_t$ being extreme ``implies''  $Y_{t+i}$ being extreme,
\end{itemize}
for some $i\leq p \in \mathbb{N}$, where $p$ is refereed to as a ``max-lag''. 
Our framework generalizes prior work, such as \cite{Gnecco, Pasche, bodik}, by incorporating both Granger-type lagged causal effects and the potential presence of confounders, and by relaxing strong assumptions of regular variation and tail equivalence, allowing applicability to both heavy- and light-tailed variables. 
Approaches such as \cite{hong2009granger, Candelon2013, Mazzarisi_causality_in_tail} can also be viewed as special cases of this framework.

As theoretical contributions, we show formal equivalencies between the definitions of causality in extremes and classical definitions of causality. Additionally, we prove that hidden confounders---a key challenge in causal analysis---do not alter the results, provided certain tail assumptions are met and the sample size is large. Finally, we discuss a no-free-lunch theorem regarding testing of Granger causality and causality in extremes. 

Our practical contributions include a novel model-free method for detecting causality in extremes, able to handle complex multivariate time series. Along with an open-source implementation, we prove its consistency in large sample sizes. An empirical comparison with state-of-the-art methods \citep[such as PCMCI,][]{Runge2019PCMCI} highlights that our proposed approach is more accurate, faster and more robust across many practical settings. 

As an application of our framework, we analyzed a hydro-meteorological system in Switzerland and cryptocurrency returns. 
Our results identify coherent impacts of extreme precipitation on different regions of the river network and main drivers of extreme events in the cryptocurrency market.
For example, the latter could be practically useful to traders for anticipating extreme market movements and adjusting their strategies accordingly.

\subsection{Existing literature and notation}
The intersection between causality and extremes is a burgeoning research area, and only recently have some connections between causality and extremes been made. \cite{EngelkeGraphicalModels} propose graphical models within the context of extremes. \cite{Extremal_quantile_treatment_effect_Deuber_and_Engelke} have developed a method for estimating extremal quantiles of treatment effects. \cite{bodik2024extreme_treatment_effect} introduced the notion of extreme treatment effect in the potential outcomes framework.    \cite{Naveau} analyzed the effect of climate change on weather extremes. \cite{courgeau2021extreme} proposed a framework for extreme event propagation.  

We formalize the definition of Granger causality in extremes using a generalization of the so-called causal tail coefficient, first introduced by \cite{Gnecco} in the context of SCMs. For a pair of random variables $X_1, X_2$  with their respective distributions $F_1, F_2$, the causal (upper) tail coefficient of $X_1$ on $X_2$ is defined as
$$
\Gamma_{1,2} := \lim_{v\to 1^-}\mathbb{E}[ F_2(X_2)\mid F_1(X_1)>v ], 
$$
if the limit exists. This coefficient lies between zero and one and captures the influence of $X_1$ on $X_2$ in the upper tail. Intuitively, if $X_1$ has a monotonically increasing influence on $X_2$, we expect $\Gamma_{1,2}$ to be close to unity. Under strong assumptions on the tails of $X_1, X_2$ and their underlying causal structure, the values of $\Gamma_{1,2}$ and $\Gamma_{2,1}$ allow us to discover the causal relationship between $X_1$ and $X_2$ \citep[Theorem~1]{Gnecco}. \cite{Pasche} proposed an inference method that adjusts the causal tail coefficient for observed confounders and a permutation test strategy for causal discovery. \cite{bodik} modified the causal tail coefficient for stationary bivariate time series $(\textbf{X},\textbf{Y})^\top= ((X_t,Y_t)^\top, t\in\mathbb{Z})$ to
$$
\Gamma_{\textbf{X}\to \textbf{Y}}(p):=\lim_{v\to 1^-}\mathbb{E}[\max\{F_Y(Y_1), \dots, F_Y(Y_{p})\}\mid F_X(X_0)>v],
$$
where $p\in\mathbb{N}$ is the max-lag and $F_X, F_Y$ are marginal distributions of $\textbf{X},\textbf{Y}$ respectively. This coefficient allows discovering the causal relationship (in the Granger sense) between $\textbf{X},\textbf{Y}$ under strong assumptions on their tails \citep[Theorem~1]{bodik}. However, these results do not apply to time series systems that are confounded, normally distributed, or whose cause and effect variables have different tails. 

A different approach to causality in extremes is taken by \cite{hong2009granger, Candelon2013, Mazzarisi_causality_in_tail}. They transform the original time series $(\textbf{X},\textbf{Y})^\top = ((X_t,Y_t)^\top, t\in\mathbb{Z})$ into a binary-valued time series $(\tilde{\textbf{X}},\tilde{\textbf{Y}})^\top$, where $\tilde{X}_t := 1$ if $X_t$ exceeds a threshold, and $0$ otherwise (similarly for $\tilde{Y}_t$). Causal relations are then assessed using a parametric model for the resulting discrete time series.

Structural causal models \citep[SCMs,][]{Pearl_book, Elements_of_Causal_Inference} or Bayesian networks \citep{PCalgorithm}, are a prevalent approach for modeling causal relationships in non-temporal contexts, allowing for the explicit representation of causal mechanisms and counterfactual reasoning \citep{bodik2026retrospective}. Several methods have been devised to adapt SCMs for accurately capturing temporal and dynamic causal relationships \citep{white2010granger, Eicher}. While these developments offer deeper insights into causality in time-dependent systems, the integration of SCMs with time series data remains an ongoing area of research
\citep{TIMINOPeters, Runge2019PCMCI, Pamfil2020DYNOTEARSSL, assaad2022survey}.

In Section~\ref{section2}, we review classical definitions of causality, introduce our novel notions of ``causality in extremes'' and ``causality in tails'', and establish their equivalence with classical causal notions. Section~\ref{section3} offers a characterization of causality in extremes in the presence of hidden confounding. 
In Section~\ref{Section_estimation}, we propose a consistent inference procedure for discovering Granger causality in extremes.
Section~\ref{Section5} discusses the multivariate extension of our approach. 
Section~\ref{Section6} presents conclusions from simulation experiments, and Section~\ref{Section_application} discusses real-world applications to extreme causal discovery for hydro-meteorological systems and cryptocurrency returns. 
The supplementary material \citesupplement contains generalizations of the results presented in the main paper to non-unit causal lags and to both tails, theory about the limitations of general statistical tests for Granger and extremal causality, motivating the assumptions used in this paper, additional details about the simulations studies, and the mathematical proofs. 
Finally, we also provide an open-source implementation of all the methods discussed in this manuscript.

In this work, multivariate time series denoted, for example, $\mathbf{W} = (\mathbf{W}_t, t\in\mathbb{Z})$, comprise $d$-dimensional random vectors defined on a shared underlying probability space $(\Omega,\mathcal{F},\mathbb{P})$. For a collection of random variables, $\sigma(\cdot)$ denotes the smallest $\sigma$-algebra with respect to which those random variables are measurable. By a stationary time series we mean a strictly
stationary time series. A stationary series \(\mathbf W\) is called
ergodic if every time-shift-invariant event
\(A\in\sigma(\mathbf W_t:t\in\mathbb Z)\) satisfies \(
\mathbb P(A)\in\{0,1\}.
\) For a matrix $\textbf{A}\in\mathbb{R}^{d\times d}$, we define its norm as  $||\textbf{A}|| = \sup_{x\in\mathbb{R}^d, |x|=1}|\textbf{A}x|$. For random variables \(U\) and \(V\) defined on the same probability space, \(U \not\equiv V\) a.s. means that they are not
almost surely equal, equivalently \(\mathbb P(U\neq V)>0\). We say that $\textbf{W}$ is $1$-Markov, if the future values, given its current value, is independent of the past; that is, $\sigma(\mathbf W_s:s\ge t+1)
\perp\!\!\!\perp
\sigma(\mathbf W_s:s\le t-1)
\mid
\sigma(\mathbf W_t)$ for all $t\in\mathbb{Z}$ \citep{Markov_processes}. We use the notation $\past(t) = (t, t-1, t-2, \dots)$.

\section{From Granger causality via structural causality to causality in extremes}
\label{section2}

\subsection{Granger and structural causality}\label{section_granger_struct_def}

Granger causality is rooted in the fundamental axiom that ``the past and present can influence the future, but the future cannot influence the past'' \citep{GRANGER1980329}. For a bivariate process  $(\textbf{X},\textbf{Y})=( (X_t, Y_t)^\top, t\in\mathbb{Z})$, $\textbf{X}$ is considered to cause $\textbf{Y}$, if the knowledge of variable $X_t$ aids in predicting the future variable $Y_{t+1}$. While predictability on its own is essentially a statement about stochastic dependence, it is precisely the axiomatic imposition of a temporal ordering that allows interpreting such dependence as a causal connection. The notion of Granger causality can be formalized as follows.

\begin{definition}[Granger causality \citep{GRANGER1980329}]\label{DEF1}
Let \(\mathbf W=(\mathbf X,\mathbf Y,\mathbf Z)=((X_t,Y_t,\mathbf Z_t)^\top,\ t\in\mathbb Z)\) be a finite-dimensional stochastic process. We say that \(\mathbf X\) Granger-causes \(\mathbf Y\) at time \(t\), with respect to \(\mathbf Z\), if
\begin{equation}\label{Granger_first_definition}
    Y_{t+1}\notindep \mathbf X_{\past(t)}
    \mid \mathcal C_t^{-\mathbf X}, 
    \qquad 
    \mathcal C_t^{-\mathbf X}:=\sigma(\mathbf Y_{\past(t)},\mathbf Z_{\past(t)}).
\end{equation}
We write \(\mathbf X\toverset{G}{\to}\mathbf Y\mid\mathbf Z\) if there exists \(t\in\mathbb Z\) such that \eqref{Granger_first_definition} holds.

We simply write \(\mathbf X\toverset{G}{\to}\mathbf Y\) if the conditioning set \(\mathcal C_t^{-\mathbf X}\) is causally sufficient in the following sense: replacing \(\mathcal C_t^{-\mathbf X}\) by any admissible enlargement of the information available up to time \(t\) that does not contain \(\mathbf X_{\past(t)}\) does not change whether \(\mathbf X\) Granger-causes \(\mathbf Y\) at time~\(t\). 
\end{definition}

We emphasize that Granger causality is not causality in the interventionist or counterfactual sense. Rather, it is a predictive notion based on temporal ordering: if the past of \(\mathbf X\) improves the prediction of \(Y_{t+1}\) after conditioning on the available information set, then \(\mathbf X\) is said to Granger-cause \(\mathbf Y\). The causal terminology is justified only under the additional assumption of causal sufficiency (absence of hidden confounders). When such assumptions are not credible, \(\mathbf X\toverset{G}{\to}\mathbf Y\mid\mathbf Z\) should be interpreted only as evidence that \(\mathbf X\) is a potential, or ``prima facie,'' cause of \(\mathbf Y\), rather than as a definitive causal statement. This is analogous to observational causal inference in the potential-outcomes framework \citep{Rubin}, where causal relations are identified only under assumptions such as no hidden confounding.

More specialized notions than the one in Definition~\ref{DEF1} have also appeared in the literature \citep{CausalityInVariance}. We say that the process \(\mathbf X\) Granger-causes \(\mathbf Y\) in mean, respectively in variance, if
\begin{equation*}
    \mathbb{E}\left[Y_{t+1}\mid \mathcal{C}^{-\mathbf X}_{t}\right]
    \neq
    \mathbb{E}\left[Y_{t+1}\mid \mathcal{C}_{t}\right],
    \quad\text{respectively}\quad
    \operatorname{Var}\!\left(Y_{t+1}\mid \mathcal{C}^{-\mathbf X}_{t}\right)
    \neq
    \operatorname{Var}\!\left(Y_{t+1}\mid \mathcal{C}_{t}\right),
\end{equation*}
for some \(t\in\mathbb Z\), where \(\mathcal C_t:=\sigma(\mathbf X_{\past(t)},\mathbf Y_{\past(t)},\mathbf Z_{\past(t)})\) represents generated by the observed past of the system up to time $t$. If \(\mathbf X\) Granger-causes \(\mathbf Y\) in mean or in variance, then necessarily \(\mathbf X\toverset{G}{\to}\mathbf Y\). However, \(\mathbf X\) may Granger-cause \(\mathbf Y\) in variance without Granger-causing it in mean, as occurs in generalized autoregressive conditionally heteroskedastic models \citep[GARCH,][]{CausalityInVariance}.

A different concept of causality, known as ``structural causality,'' was introduced by \cite{white2010granger} as a time series analog to the Structural Causal Model (SCM). In this framework, $\textbf{X}$ and $\textbf{Y}$ are assumed to be generated structurally as:
\begin{equation}\label{structural_generation_lag}
\begin{split}
    X_t&=h_{X,t}(X_{t-1}, \dots, X_{t-q_x}, Y_{t-1}, \dots, Y_{t-q_x}, \textbf{Z}_{t-1}, \dots, \textbf{Z}_{t-q_x}, \varepsilon^X_t),\\
    Y_t&=h_{Y,t}(X_{t-1}, \dots, X_{t-q_y}, Y_{t-1}, \dots, Y_{t-q_y}, \textbf{Z}_{t-1}, \dots, \textbf{Z}_{t-q_y}, \varepsilon^Y_t),
\end{split}
\end{equation}
for all $t\in\mathbb{Z}$, where $h_{X,t}$ and $h_{Y,t}$ are measurable functions, and $q_x, q_y\in\mathbb{N}\cup\{\infty\}$ are called orders (lags) of $\textbf{X},\textbf{Y}$, respectively. Here, the process $\textbf{Z}$ encompasses all other variables in the system, and $\varepsilon^X_t$ and $\varepsilon^Y_t$ are the noise variables. Typically, we assume that $h_{X,t}$ are equal for all $t\in\mathbb{Z}$, in which case we omit the subscript $t$ and simply write $h_X$ (similarly for $h_Y$).
 
For clarity of the text, we simplify the notation by assuming $q_x=q_y=1$. Nonetheless, we relax this assumption in Supplement~\ref{Appendix_A}. 

\begin{definition}[Structural causality]\label{Definition_str}
Assume that $(\textbf{X},\textbf{Y}, \textbf{Z})$ are stationary time series, where $\textbf{Y}$ is structurally generated as 
\begin{equation*}%
\begin{split}
    Y_t&=h_{Y}(X_{t-1}, Y_{t-1}, \textbf{Z}_{t-1}, \varepsilon^Y_t),
\end{split}
\end{equation*}
for all $t\in\mathbb{Z}$, where \(h_Y\) is a measurable function and \((\varepsilon_t^Y)_{t\in\mathbb Z}\) are noise variables satisfying 
\begin{equation}\label{IndependenceOfNoise}
\varepsilon_t^Y
    \indep
    \sigma(X_s,Y_s,\mathbf Z_s:s<t),
    \qquad \text{for all }t\in\mathbb Z.    
\end{equation} 
We say that \(X\) does \textit{not} directly structurally cause \(Y\) if there exists a measurable
function \(\widetilde h_Y\) such that \(
    Y_t=\widetilde h_Y(Y_{t-1},\mathbf Z_{t-1},\varepsilon_t^Y)
\) a.s. for all $t\in\mathbb{Z}$. 
Otherwise, $\textbf{X}$ is said to directly structurally cause $\textbf{Y}$ (notation $\textbf{X}\toverset{str}{\to}\textbf{Y}$). 
\end{definition}

Definition~\ref{Definition_str} implicitly assumes causal sufficiency: the structural equation for \(Y_t\) contains all relevant lagged variables only through \((X_{t-1},Y_{t-1},\mathbf Z_{t-1})\). 

The definitions of Granger causality and structural causality are closely related. Under the assumption that $\textbf{X}$ and $\textbf{Y}$ are structurally generated as described in Definition~\ref{Definition_str}, Granger causality implies structural causality \citep[Proposition 1, Chapter 22.4]{berzuini2012causality}. The reverse implication is generally not true; however, the distinction between these definitions is not relevant for the purposes of the present paper. Generally, it is only relevant for counterfactual statements, and can be disregarded by considering the concept of ``almost sure structural causality'' \citep[Section 3.1]{white2010granger}.

\subsection{Causality in extremes}
\label{section_granger_cusality_in_tail}
Recall the two intuitive definitions of causality in extremes from Section~\ref{s:intro}, which can be reformulated, in this context, as: an extreme event for $X_t$ 
\begin{itemize}
    \item increases the probability of an extreme event for $Y_{t+1}$, given~$\mathcal{C}^{-\textbf{X}}_{t}$, or,
    \item implies an extreme event for $Y_{t+1}$, given~$\mathcal{C}^{-\textbf{X}}_{t}$.
\end{itemize}
These two intuitive notions are formalized in Definition~\ref{definiton_of_tail_causality}. For simplicity, we assume that (i) \(X_t\) and \(Y_t\) are supported on some neighborhood 
of infinity for all \(t\in\mathbb Z\), and (ii) the time series satisfy the 1-Markov property. We relax these assumptions in Section~\ref{section_extensions}. 

\begin{definition}[Causality in extremes]\label{definiton_of_tail_causality}
Let \(\mathbf W=(\mathbf X,\mathbf Y,\mathbf Z)=((X_t,Y_t,\mathbf Z_t)^\top,\ t\in\mathbb Z)\) be a finite-dimensional stochastic process satisfying the 1-Markov property. Let \(F\) be a cumulative distribution function satisfying \(F(x)<1\) for all \(x\in\mathbb R\), and assume that \(X_t\) and \(Y_{t+1}\) are supported on some neighborhood of infinity.

Define the causal tail coefficient adjusted for \(\textbf{Z}\) at time $t\in\mathbb Z$ by
\begin{equation}\label{CTC_definition}
\Gamma_{\textbf{X}\to\textbf{Y}\mid\mathcal{C}}^t:=  \lim_{v\to\infty} \mathbb{E}[ F(Y_{t+1})\mid X_t>v,\mathcal{C}^{-\textbf{X}}_t], \qquad  \text{ where }  \,\, \mathcal{C}_t^{-\mathbf X}:=\sigma(\mathbf Y_{\past(t)},\mathbf Z_{\past(t)}),
\end{equation}
provided that the limit exists a.s., and the baseline coefficient
\[
\Gamma_{\textbf{X}\to\textbf{Y}\mid\mathcal{C}}^{t, \baseline}:=\mathbb{E}[ F(Y_{t+1})\mid  \mathcal{C}^{-\textbf{X}}_t]. 
\]
We simply write  $\Gamma_{\mathbf X\to\mathbf Y\mid\mathcal{C}}$ and $   \Gamma_{\mathbf X\to\mathbf Y\mid\mathcal{C}}^{\baseline}$ when the system is stationary; that is, when the coefficient is time-invariant.

We say that the (upper) tail of \(\mathbf X\) causes \(\mathbf Y\) at time \(t\), adjusted for \(\textbf{Z}\), if
\begin{equation}
\label{tail_cause}
\Gamma_{\mathbf X\to\mathbf Y\mid \mathcal C}^t
\not\equiv
\Gamma_{\mathbf X\to\mathbf Y\mid \mathcal C}^{t,\baseline}
\quad \text{a.s.},
\end{equation}
We write $\mathbf X \toverset{tail}{\longrightarrow} \mathbf Y\mid \textbf{Z}$ if there exists \(t\in\mathbb Z\) such that \eqref{tail_cause} holds. 

We say that (upper) extreme in \(\mathbf X\) causes an extreme in \(\mathbf Y\) at time \(t\), adjusted for \(\textbf{Z}\), if
\begin{equation}
    \label{ext_cause}
        \Gamma_{\mathbf X\to\mathbf Y\mid \mathcal{C}}^t=1
    \qquad \text{a.s.}
\end{equation}
We write $   \mathbf X \toverset{ext}{\longrightarrow} \mathbf Y\mid \textbf{Z}$ if there exists \(t\in\mathbb Z\) such that  \eqref{ext_cause} holds. 

We simply write \(\mathbf X\toverset{tail}{\to}\mathbf Y\) or \(\mathbf X\toverset{ext}{\to}\mathbf Y\) if the conditioning set \(\mathcal C_t^{-\mathbf X}\) is causally sufficient in the sense that replacing \(\mathcal C_t^{-\mathbf X}\) by any admissible enlargement of the information available up to time \(t\) that does not contain \(\mathbf X_{\past(t)}\) does not change whether \(\mathbf X\) tail-causes or  extreme-causes \(\mathbf Y\). 
\end{definition}

An alternative but equivalent formulation of causality in extremes (following directly from Lemma~\ref{Observation}, in the Supplement) is: 
\begin{equation}\label{eq_42}
    \textbf{X}\toverset{ext}{\longrightarrow}\textbf{Y} \quad \iff  \quad \lim_{v\to\infty}P(Y_{t+1}>\tau\mid X_t>v, \mathcal{C}^{-\textbf{X}}_t) = 1 \quad \text{ for every }\tau\in\mathbb{R},
\end{equation}
for some $t\in\mathbb Z$. While the right side of \eqref{eq_42} might, perhaps, be easier to interpret, the definition of causality in extremes as in \eqref{CTC_definition} offers several advantages. While it also leads to a natural connection to the causality in tail and extensions to more general structures of time series, the main advantage is its flexibility in the choice of $F$. 
In their more restrictive setting, \cite{Gnecco} and \cite{bodik} considered $F$ as a marginal distribution function of $\textbf{Y}$. 
\cite{hong2009granger} and \cite{Mazzarisi_causality_in_tail} utilized the special case $F(x) = \mathbbm{1}(x > \tau)$ for a given threshold $\tau\in\mathbb{R}$, which leads to $ \mathbb{E}[ F(Y_{t+1})\mid X_t>v,\mathcal{C}^{-\textbf{X}}_t] = P( Y_{t+1} > \tau \mid X_t>v,\mathcal{C}^{-\textbf{X}}_t)$. 
However, this threshold-based choice does not satisfy \(F(x)<1\) for all \(x\in\mathbb R\); for such fixed thresholds, the equivalence results with Granger-type causality can fail in general. In practice, an appropriate choice for $F$ can lead to better performance when it comes to inference and avoids the need to choose the threshold $\tau$. 
Nevertheless, in the limit, the notions \( \textbf{X} \toverset{tail}{\longrightarrow} \textbf{Y} \) and \( \textbf{X} \toverset{ext}{\longrightarrow} \textbf{Y} \) remain invariant w.r.t.\ $F$ under very weak conditions (see Lemma \ref{lemma_o_invariance_about_F}). 
Thus, we do not specify \( F \) when stating ``tail/extreme of $\textbf{X}$ causes~$\textbf{Y}$.''

We also write $\Gamma_{\textbf{X}\to\textbf{Y}\mid\textbf{Z}}$ when we want to emphasize the exact variables we condition on; in particular, we use the notation $\Gamma_{\textbf{X}\to\textbf{Y}\mid\emptyset}$ when $\textbf{Z}$ is an empty set (replacing $\mathcal{C}^{-\textbf{X}}_t$ by $\sigma(\textbf{Y}_{\past(t)})$).

\subsection{Connections between the definitions}
\label{Section_Reformulating}

The notions \(\textbf{X}\toverset{ext}{\longrightarrow}\textbf{Y}\),
\(\textbf{X}\toverset{tail}{\longrightarrow}\textbf{Y}\), and
\(\textbf{X}\toverset{G}{\to}\textbf{Y}\) become equivalent under assumptions that
ensure that structural dependence remains visible in the tails. We now present these
assumptions.

\begin{assumptions*}
Assuming the structure from Definition~\ref{Definition_str}, we impose the following:
\begin{enumerate}[label=(A\arabic*), ref=A\arabic*]
    \item \label{AssumptionA1}
    Either
    \(
        \lim_{x\to\infty} h_Y(x,y,\mathbf z,e)=\infty
    \)
    for all admissible values of \((y,\mathbf z,e)\), or \(h_Y\)
    is constant in \(x\).

    \item \label{AssumptionA2}
    Either
    \(
        \lim_{|x|\to\infty} \big|h_Y(x,y,\mathbf z,e)\big|=\infty
    \)
    for all admissible values of \((y,\mathbf z,e)\), or \(h_Y\)
    is constant in \(x\).
\end{enumerate}
\end{assumptions*}

Assumptions A1 and A2 link ordinary structural dependence with dependence in extremes. They require that whenever \(X\) has a direct structural effect on \(Y\), this effect does not vanish in the tails: extreme values of \(X_t\) are transmitted into extreme values of \(Y_{t+1}\). Assumption~A1 formalizes this requirement for the upper tail, whereas Assumption A2 concerns both tails. 
Either Assumption A1 or A2 hold true in many classical models, such as vector autoregressive models \citep[VAR,][]{MVT_timeseries_SPRINGER} and GARCH, among others. However, they may fail in settings where the structural effect is dampened or bounded, for example when
\(Y_{t+1} = sin(X_{t})+\varepsilon_{t+1}^Y\). Note that \ref{AssumptionA1} and \ref{AssumptionA2} are automatically satisfied if $\textbf{X}\not\toverset{str}{\to} \textbf{Y}$.

\begin{proposition}\label{Proposition_CTC_iff_1}
If $\textbf{X}\toverset{ext}{\longrightarrow}\textbf{Y}$, then $\textbf{X}\toverset{tail}{\longrightarrow}\textbf{Y}$. Under Assumption \ref{AssumptionA1},  if $\textbf{X}\toverset{tail}{\longrightarrow}\textbf{Y}$, then $\textbf{X}\toverset{ext}{\longrightarrow}\textbf{Y}$. 
\end{proposition}

\begin{proposition}\label{Proposition_Granger_equivalent_tail}
 If $\textbf{X}\toverset{tail}{\longrightarrow}\textbf{Y}$, then $\textbf{X}\toverset{G}{\to}\textbf{Y}$. Under Assumption \ref{AssumptionA1},  if $\textbf{X}\toverset{G}{\to}\textbf{Y}$, then $\textbf{X}\toverset{tail}{\to}\textbf{Y}$. 
\end{proposition}

The proofs are presented in Section~\ref{Proof_of_propositions_1_and_2} of the supplementary material. 
Combining Propositions \ref{Proposition_CTC_iff_1} and \ref{Proposition_Granger_equivalent_tail}, $\textbf{X}\toverset{ext}{\longrightarrow}\textbf{Y}$ implies $\textbf{X}\toverset{G}{\to}\textbf{Y}$ as long as both definitions are well-defined.  
Moreover, under Assumption~\ref{AssumptionA1}, the other implication is also valid. This result is related to faithfulness. Structural causality is a property of the
structural equation, whereas Granger causality is a property of the observational
distribution. Thus, a structural dependence of \(Y_{t+1}\) on \(X_t\) may fail to
imply Granger causality if it is distributionally masked by changes only on null sets or cancellations, as in unfaithful SCMs. Assumptions A1/A2 play a related tail-specific role: they ensure that structural
dependence remains visible in the tails.

As a final note, Lemma~\ref{lemma_o_invariance_about_F} formalises the invariance of Definition~\ref{definiton_of_tail_causality} to the choice of $F$.

\begin{lemma}\label{lemma_o_invariance_about_F}
Under Assumption \ref{AssumptionA1}, the definition of  $\textbf{X}\toverset{tail}{\longrightarrow}\textbf{Y}$ is invariant with the choice of $F$. That is, for any distribution functions $F_1, F_2$ satisfying $F_i(x)<1$ for all $x\in\mathbb{R}$:
\begin{equation*}
    \begin{split}
        \lim_{v\to\infty} \mathbb{E}[ F_1(Y_{t+1})\mid X_t>v,&  \, \mathcal{C}^{-\textbf{X}}_t]\neq\mathbb{E}[ F_1(Y_{t+1})\mid  \mathcal{C}^{-\textbf{X}}_t]\\
        & \iff \\
        \lim_{v\to\infty} \mathbb{E}[ F_2(Y_{t+1})\mid X_t>v,&\,  \mathcal{C}^{-\textbf{X}}_t]\neq\mathbb{E}[ F_2(Y_{t+1})\mid  \mathcal{C}^{-\textbf{X}}_t]. 
    \end{split}
\end{equation*}
\end{lemma}

\subsection{Extensions}
\label{section_extensions}

Definition~\ref{definiton_of_tail_causality} focuses on upper-tail effects from \(X_t\) to \(Y_{t+1}\). We briefly describe four extensions that are useful in applications: non-unit causal lags, two-sided extremes, bounded support, and instantaneous extremal effects. The formulations are developed in detail in Supplement~\ref{Appendix_A}; here, we present only the main ideas.

\subsubsection{Non-unit causal lags}\label{lag_extention_in_section2}
In many time series, an extreme event in \(X_t\) may affect \(Y\) only after several time steps. For \(p\in\mathbb N\), define
\[
\Gamma^t_{\mathbf X\to \mathbf Y\mid\mathcal C}(p)
:=
\lim_{v\to\infty}
\mathbb E\!\left[
\max\{F(Y_{t+1}), \dots, F(Y_{t+p})\}
\,\middle|\,
X_t>v,\mathcal C^{-X_t}_t
\right],
\]
where \(
\mathcal C^{-X_t}_t
:=
\sigma(X_s:s<t,\;Y_s,Z_s:s\le t).
\) The corresponding baseline coefficient is defined analogously as  
\(
\Gamma^{t,\baseline}_{\mathbf X\to \mathbf Y\mid\mathcal C}(p)
:=
\mathbb E\!\left[
\max_{1\le j\le p} F(Y_{t+j})
\,\middle|\,
\mathcal C^{-X_t}_t
\right].
\) The max-lag \(p\) version of Definition~\ref{definiton_of_tail_causality} is obtained by replacing the unit-lag coefficient
\(\Gamma^t_{\mathbf X\to \mathbf Y\mid\mathcal C}\)
with
\(\Gamma^t_{\mathbf X\to \mathbf Y\mid\mathcal C}(p)\).
We write $\textbf{X}\toverset{tail(p)}{\longrightarrow}\textbf{Y}$ and $\textbf{X}\toverset{ext(p)}{\longrightarrow}\textbf{Y}$. 

When \(p\) is large enough to include the relevant causal delay in \eqref{structural_generation_lag}, the lagged coefficient recovers the corresponding main-text implications. In particular, Propositions~\ref{proposition_with_lag_1} and~\ref{proposition_with_lag_2} located in Supplement~\ref{Appendix_A_lag_greater_than_1} show that\footnote{Sims causality \citep{sims} is closely related to Granger causality;  see Supplement~\ref{appendix_lag_section2} for details.}
\[
\mathbf X\overset{{\rm ext}(p)}{\longrightarrow}\mathbf Y
\implies
\mathbf X\overset{{\rm tail}(p)}{\longrightarrow}\mathbf Y
\implies
\mathbf X\toverset{Sims}{\longrightarrow}\mathbf Y,
\]
with the converse implication valid an analogous condition to A1, for \(p\) large enough to cover the first relevant causal delay. The lagged analogue of Theorem~\ref{Theorem1} is discussed in Supplement~\ref{supplement_section3_lagged}, and the corresponding estimator and algorithmic adaptations are described in Supplement~\ref{supplement_section4_lagged}.

\subsubsection{Instantaneous extremal effects}
In practice, measurements are often collected in large steps, and hence, one may be interested in instantaneous extremal effects, where an extreme event in \(X_t\) is associated with an extreme event in \(Y_t\). Such effects are not Granger-causal in the usual sense and require an additional contemporaneous causal ordering or structural model to define a causal notion. Nevertheless, we define
\[
\Gamma^{t}_{\mathbf X\to \mathbf Y\mid\mathcal C}([0,p])
:=
\lim_{v\to\infty}
\mathbb E\!\left[
\max\{F(Y_t), \dots, F(Y_{t+p})\}
\,\middle|\,
X_t>v,\mathcal C^{-\{X_t,Y_t\}}_t
\right],
\]
where \(\mathcal C^{-\{X_t,Y_t\}}_t\) contains the admissible information available at time \(t\) excluding \(X_t\) and \(Y_t\). The corresponding causality in extremes is defined analogously with this modified coefficient. Formal definition is located in Appendix~\ref{Appendix_A_Instantaneous}. 

In Proposition~\ref{theorem_instantaneous}, located in Appendix~\ref{Appendix_A_Instantaneous}, we show that, under additive model assumptions and heavy-tailed innovations and $p=0$, $\Gamma^{t}_{\mathbf X\to \mathbf Y\mid\mathcal C}([0, p]) = 1$ if and only if $X_t$ enters the data-generating process of $Y_t$. Although this notion can be useful for systems with contemporaneous interactions, its causal interpretation requires assumptions substantially stronger than causal sufficiency. 

\subsubsection{Both tails}
The upper-tail formulation is appropriate when only large positive values are of interest. In many applications, however, both unusually large positive and unusually large negative values are relevant. In this case, we apply the same definition to \((|\mathbf X|,|\mathbf Y|)\). More precisely, we define
\[
\Gamma^t_{|\mathbf X|\to |\mathbf Y|\mid\mathcal C}(p)
:=
\lim_{v\to\infty}
\mathbb E\!\left[
\max_{1\le j\le p} F(|Y_{t+j}|)
\,\middle|\,
|X_t|>v,\mathcal C^{-X_t}_t
\right].
\]
The corresponding baseline coefficient is defined analogously, without the conditioning. We write
\(\mathbf{X}\toverset{tail\pm(p)}{\longrightarrow}\mathbf{Y}\) if this coefficient differs from its baseline, and
\(\mathbf{X}\toverset{ext\pm(p)}{\longrightarrow}\mathbf{Y}\) if it is equal to one almost surely. This extension captures extremal up-movements and down-movements simultaneously.

Supplement~\ref{Appendix_A_both_tails} contains counterparts of Propositions~\ref{Proposition_CTC_iff_1}, \ref{Proposition_Granger_equivalent_tail} and Lemma \ref{lemma_o_invariance_about_F} for causality in both tails. We also discuss the extension of the results from Section~\ref{section3} in Supplement~\ref{appendix_both_tails_section3} and estimator from Section~\ref{Section_estimation} in Supplement~\ref{appendix_both_tails_section4}  . 

\subsubsection{Bounded support}
The assumption that \(X_t\) and \(Y_{t+1}\) are supported near infinity is mainly a notational convenience. If \(X\) or \(Y\) has a finite upper endpoint, extremes are interpreted as values approaching the corresponding endpoint. As formalized in Supplement~\ref{Appendix_A_support}, this can be handled by replacing the limit in \eqref{CTC_definition} from \(v\to\infty\) to \(v\uparrow r_X\), where \(r_X\) is the right endpoint of the support of \(X\), and by choosing \(F_Y\) so that \(F_Y(y)\uparrow1\) as \(y\uparrow r_Y\).

\section{Robustness of causality in extremes to hidden confounders under regular variation}
\label{section3}
Causality in extremes offers a significant advantage over causality in the mean in terms of robustness to hidden confounders, at the population level. We demonstrate that, under certain assumptions about the tails of the time series,
\begin{equation}\label{OhYeah}%
\Gamma_{\textbf{X}\to\textbf{Y}\mid\mathcal{C}} = 1 \iff  \Gamma_{\textbf{X}\to\textbf{Y}\mid \emptyset} = 1.
\end{equation}
This is particularly valuable in high-dimensional datasets, where there are numerous potential confounders for $\textbf{X}$ and $\textbf{Y}$, making it challenging to distinguish between true causality and correlation induced by a hidden confounder. Equation (\ref{OhYeah}) allows us to focus solely on the coefficient $\Gamma_{\textbf{X}\to\textbf{Y}\mid \emptyset}$ without the need to condition on the potentially high-dimensional confounders. Implication ``$\Longrightarrow$'' in (\ref{OhYeah}) is relevant for testing non-causality, while ``$\Longleftarrow$'' is pertinent for testing causality.

The goal of this section is to establish the assumption for the validity of  (\ref{OhYeah}). We demonstrate that the implication ``$\Longrightarrow$'' is valid under mild assumptions, whereas the ``$\Longleftarrow$'' implication requires stronger assumptions regarding the tails of the variables. 

A similar robustness property has previously been leveraged in \cite{Gnecco} and \cite{bodik}, whose frameworks depend on the distributions having heavy tails and implicitly rely on their coefficient's robustness to confounding. We demonstrate that this population-level robustness to confounders extends to more general stochastic recurrence equations, beyond previously studied linear frameworks. 

\subsection{Preliminaries for regular variation and SRE}\label{Section_preliminaries_RV_and_SRE}

A dominant framework in the literature for modeling tails of random variables is regular variation framework \citep{RegularVariationBook, Heavy_tailed_time_series}.  A real random variable $X$ is regularly varying with tail index $\theta>0$, if its distribution is in the form  $F_X(x)=1-x^{-\theta} L(x)$ for some slowly varying function $L$, i.e., a function satisfying $\lim_{x\to\infty}\frac{L(c x)}{L(x)} = 1$ for every $c > 0$. This property is denoted by $X\sim \RV(\theta)$.  Regular variation describes that a tail decays polynomially (note that the tail of normal distribution decays exponentially). Smaller $\theta$ implies heavier tails; in particular, the k-th moment of $X$ does not exist when $\theta\leq k$. We say that random variables $X,Y$ have compatible tails, if $\lim_{x\to\infty}\frac{P(X>x)}{P(Y>x)}\in (0, \infty)$.
For real functions $f,g$, we write $f(x)\sim g(x) \iff \lim_{x\to\infty}\frac{f(x)}{g(x)}=1$.

We consider the stochastic recurrence equation \citep[SRE,][]{SRE}
\begin{equation}\label{generalSRE}
    \textbf{W}_t = \textbf{A}_t\textbf{W}_{t-1} + 
 \textbf{B}_t, \qquad t\in\mathbb{Z},
\end{equation}
where $(\textbf{A}_t, \textbf{B}_t)$ is an i.i.d.\ random sequence, $\textbf{A}_t$ are $d\times d$ matrices and $\textbf{B}_t$ are $d$ dimensional vectors. This model of time series is quite general, with VAR(1) or ARCH(1) models as special cases. 
Under mild contractivity assumptions $\mathbb{E}\log||\textbf{A}_t|| < 0$ and $\mathbb{E}\log^+|\textbf{B}_t| < \infty$ (where \(
\log^+ x := \max\{\log x, 0\}
\), see Chapter 2.2 in \cite{SRE}), the sequence $ \textbf{W}_t$ is strictly stationary, ergodic, can be rewritten as $\textbf{W}_0 = \sum_{i=0}^\infty \pi_{i-1}\textbf{B}_{-i}$, where $\pi_i = \textbf{A}_{0}\textbf{A}_{-1}\dots \textbf{A}_{-i}$ with a convention that $\pi_{-1}$ is an identity matrix, and  
satisfy a distributional equality 
\begin{equation*} 
\tilde{\textbf{W}}\toverset{d}{=} \tilde{\textbf{A}} \tilde{\textbf{W}} +  \tilde{\textbf{B}}, \qquad \tilde{\textbf{W}}\indep (\tilde{\textbf{A}},\tilde{\textbf{B}}),
\end{equation*}
where $(\tilde{\textbf{A}},\tilde{\textbf{B}})\toverset{d}{=}({\textbf{A}_1},{\textbf{B}_1})$ and $\tilde{\textbf{W}}\toverset{d}{=}\textbf{W}_0$ are generic elements. 

In the univariate case ($d=1$), the distribution of $\tilde{W}$ is regularly varying under mild assumptions on the distribution of $(A,B)$ \citep{kesten1973random}. In the literature, these assumptions typically mainly include one of the following two. 

\begin{assumption*}[Grey assumption with index $\alpha$]
There exists $\alpha>0$ such that $\mathbb{E}|A|^\alpha < 1, \mathbb{E}|A|^{\alpha+\nu}<\infty$ for some $\nu>0$ and such that
\begin{equation*}
    P(B>x)\sim p_\alpha x^{-\alpha} l(x) \;\; \text{ and } \;\; P(-B>x)\sim q_\alpha x^{-\alpha} l(x) 
\end{equation*}
with $p_\alpha, q_\alpha\geq0$, $p_\alpha\neq 0$, $ p_\alpha+q_\alpha=1$, where $l(x)$ is a slowly varying function.  
\end{assumption*}
\begin{assumption*}[Kesten-Goldie assumption with index $\alpha$]
There exists $\alpha>0$ such that $\mathbb{E}|A|^\alpha = 1, \mathbb{E}|A|^{\alpha}\log^+|A|<\infty$ and $\mathbb{E}|B|^\alpha<\infty$. Moreover, $P(Ax+B=x)<1$ for every $x\in\mathbb{R}$ and the conditional law of $\log|A|$ given $\{A\neq 0\}$ is non-arithmetic. 
\end{assumption*}
The Grey assumption is typically of interest in VAR models, whereas the Kesten-Goldie assumption is pertinent in GARCH models \citep{Pedersen}. 
\subsection{Causality in extremes under regular variation}\label{Section_causality_under_RV}
To demonstrate (\ref{OhYeah}), we adopt the assumption that our time series adhere to the SRE model (\ref{generalSRE}), denoted as follows:

\begin{equation}\label{sre2}
\textbf{W}_t = \begin{pmatrix}
Z_t  \\
X_t  \\
Y_t \\
\end{pmatrix}
, 
\textbf{A}_t = \begin{pmatrix}
A^z_{1,t} & A^z_{2,t} & A^z_{3,t} \\
A^x_{1,t} & A^x_{2,t} & A^x_{3,t} \\
A^y_{1,t}& A^y_{2,t} & A^y_{3,t} \\
\end{pmatrix}
, \textbf{B}_t=
\begin{pmatrix}
B^z_{t} \\
B^x_{t}  \\
B^y_{t}  \\
\end{pmatrix}, \varepsilon_t^\cdot = (A^\cdot_{1,t},A^\cdot_{2,t},A^\cdot_{3,t},B^\cdot_t)^\top.
\end{equation}

For simplicity, we assume that \(\textbf{Z}\), which represents a potential hidden cause of \(\textbf{X}\), of \(\textbf{Y}\), or of both, is univariate.
Additionally, we assume that $B^x_t, B^y_t$ are supported on some neighborhood of infinity for all $t\in\mathbb Z$. We operate under the following assumptions:
\begin{itemize}
\item[(B1)] $\mathbb{E}[\log||\textbf{A}_t||] < 0$ and $\mathbb{E}[\log^+|\textbf{B}_t|] < \infty$,  \hfill (stationarity and ergodicity assumption)

\item[(B2)] $\varepsilon_t^{z}, \varepsilon_t^{x}, \varepsilon_t^{y}$ are independent for all $t\in\mathbb{Z}$, \hfill (no instantaneous causality)

\item[(B3)]  $ \lim_{v\to\infty }\mathbb P\left(Z_t>-av\mid X_t>v,Y_{past(t)}\right)= 1$ for every $a>0$, \hfill (upper-tails only)
 
\item[(B4)] $A_{j,t}^i \overset{a.s.}{>} 0$ for all $t\in\mathbb{Z}$ and $j=1,2,3$, $i=z,x,y$ satisfying $P(A_{j,t}^i = 0) \neq 1$,

\item[(B5)] $A_{j,t}^i$ has a density function absolutely continuous with respect to Lebesgue measure for all $t\in\mathbb{Z}$, $j=1,2,3$, $i=z,x,y$ satisfying $P(A_{j,t}^i = 0) \neq 1$.  
\end{itemize}

Condition (B1) ensures stationarity and ergodicity of the time series; in particular,
it implies that \(\Gamma_{\mathbf X\to\mathbf Y\mid\mathcal C}\) does not depend on \(t\). Condition (B2) rules out instantaneous causality. (B3) rules out the case where an extreme value of $X_t$ is systematically accompanied by an extreme negative value of $Z_t$.   Note that (B4) together with (B2) implies Assumption A1, while (B5) together with (B2) implies
Assumption A2.

Under these assumptions, Theorem~\ref{Theorem1} is the main result of this section:
it shows that the implication ``\(\implies\)'' in \eqref{OhYeah} holds under relatively
weak assumptions, whereas the reverse implication holds for regularly varying processes.

\begin{theorem}\label{Theorem1}
Consider the SRE model (\ref{sre2}), satisfying (B1), (B2), (B4). 
\begin{itemize}
    \item Under (B3), $\Gamma_{\textbf{X}\to\textbf{Y}\mid\mathcal{C}} = 1 \implies \Gamma_{\textbf{X}\to\textbf{Y}\mid \emptyset} = 1. $
\item If the pairs $(A^x_{1,t},B^x_t)^\top,(A^x_{2,t},B^x_t)^\top,(A^x_{3,t},B^x_t)^\top$ satisfy the Grey assumption with index $\alpha_x$, and $\limsup_{u\to\infty}\frac{P(X_t>u\mid Y_{past(t)})}{P(B_t^x>u)}\overset{a.s.}{<}\infty$, then $\Gamma_{\textbf{X}\to\textbf{Y}\mid\mathcal{C}} = 1 \;\Longleftarrow\; \Gamma_{\textbf{X}\to\textbf{Y}\mid \emptyset} = 1. $
\end{itemize}
\end{theorem}
The proof is given in Supplement~\ref{proof_of_thm1}. The condition
\(
\limsup_{u\to\infty}\frac{P(X_t>u\mid Y_{past(t)})}{P(B_t^x>u)}<\infty
\)
ensures that some extreme events originate in \(X_t\), and its tail is not entirely determined by the tails of
$\textbf{W}_{past(t)}$. The computation of this limit in stochastic recurrence equations has been extensively studied; see
\citet[Theorem~4.4.24]{SRE} and \cite{Ewa_Damek}.

Theorem~\ref{Theorem1} shows that, under Grey-type tail assumptions, the adjustment
process \(\mathbf Z\) can be ignored when assessing causality in extremes from
\(\mathbf X\) to \(\mathbf Y\), provided that the extreme behavior
of \(X_t\) is sufficiently driven by its own innovation \(B_t^x\). In the SRE representation, this corresponds to settings where the
contribution of the confounding pathway through \(A^y_{1,t}B_t^z\) is not heavier-tailed
than the contribution of the direct pathway through \(A^y_{2,t}B_t^x\). Whether analogous
results hold under the Kesten--Goldie assumption, instead of the Grey assumption, remains
an open problem.

\section{Estimation and causal discovery}
\label{Section_estimation}

We introduce a family of estimators of ${\Gamma}_{\textbf{X}\to\textbf{Y}\mid\mathcal{\textbf{Z}}}$ and a classification procedure that outputs either $\textbf{X}\toverset{ext}{\to}\textbf{Y}$ or $\textbf{X}\toverset{ext}{\not\to}\textbf{Y}$ from data. We denote by  $\textbf{Z} = (\textbf{Z}_t, t\in\mathbb{Z})$ a vector of other relevant time series (possible confounders) with dimension $\dim(\textbf{Z}_t) = d\in\mathbb{N}$.  We assume that we observe $n\in\mathbb{N}$ time steps of the series $(x_1, y_1, \textbf{z}_1)^\top, \dots, (x_n, y_n, \textbf{z}_n)^\top$. 
\begin{definition}\label{Gamma_hat_original}
Let $F$ be a distribution function from Definition~\ref{definiton_of_tail_causality}. We propose a general covariate-adjusted estimator of the form
  \begin{equation}\label{Gamma_hat_equation}
\hat{\Gamma}_{\textbf{X}\to\textbf{Y}\mid\mathcal{\textbf{Z}}}:=\frac{1}{|S|}\sum_{t\in S} F(y_{t+1}),
  \end{equation}
 where several choices for the set $S\subseteq\{1, \dots, n-1\}$ are described below. 
\end{definition}

One possible choice, leading to an unadjusted estimator, is $S_0:= \{t\in \{1, \dots, n-1\} : x_t\geq\tau_n^X\}$, where $\tau_n^X\in\mathbb{R}^{\mathbb{N}}$ is a sequence satisfying $\tau_n^X\to\infty$ and $|S_0|\to\infty$ in probability. In practice, this can be achieved by taking $\tau_n^X=x_{(n-k_n+1)}$, the $k_n$-th largest value of $x_1, \dots, x_n$, where $k_n$ is any sequence satisfying
\begin{equation}
 \label{k_deleno_n}
 k_n\to\infty, \; \frac{k_n}{n}\to 0, \quad \text{ as } n\to\infty.
\end{equation}
The estimator introduced in \cite{bodik} can be viewed as this special case.

\subsection{Conditioning on confounders being non-extreme}
Although we have shown the population coefficient $\Gamma_{\textbf{X}\to\textbf{Y}\mid\emptyset}$ to be robust to confounders, in the sense of~\eqref{OhYeah}, confounding effects can still have undesirable impacts on finite sample estimation and testing, as shown by~\cite{Pasche} in the i.i.d.\ case. 
We present alternative choices for the set $S$ in \eqref{Gamma_hat_equation} with the aim of removing the confounding influence of $\textbf{Z}$ in the extremes and enhancing the efficacy of the estimator in scenarios with different tail behaviors.
The general idea of the sets we propose in the definitions below is to condition on  $X_t$ being extreme, while we condition on all other relevant variables not being extreme. 
This ensures that an extreme event in $Y_{t+1}$ is indeed caused by an extreme event in $X_{t}$, and it is not caused by a common confounder $\textbf{Z}_{t}$ or $Y_{t}$. 
\begin{definition}\label{Def_estim_S1}
  Let \begin{equation*}
      S_{1} := \{t\in\{1, \dots, n-1\}: X_t\geq \tau_n^X, 
      \begin{pmatrix}
    Y_t \\
    \textbf{Z}_t 
\end{pmatrix}   \leq  \boldsymbol{\tau} \},
  \end{equation*}
where $\boldsymbol{\tau} = (\tau_Y, \boldsymbol{\tau}_Z)^{\top}\in\mathbb{R}^{1+d}$ is a fixed constant and $\tau_n^X\to\infty$ is a sequence such that $|S_1|\to\infty$ in probability. The inequality $ \begin{pmatrix} Y_t \\ \textbf{Z}_t \end{pmatrix} \leq \boldsymbol{\tau} $ is understood element-wise.
\end{definition}
\begin{definition}\label{definition_of_S2}
  We denote by $B_{w_0}(r) = \{w : ||w-w_0||_{\infty}<r\}$ the ball with center $w_0$ and radius $r\in\mathbb{R}^+$. 
  Let \begin{equation*}
      S_2 := \{t\in\{1, \dots, n-1\}: X_t\geq \tau_n^X, 
      \begin{pmatrix}
    Y_t \\
    \textbf{Z}_t 
\end{pmatrix}   \in {B}_{(y_0, \textbf{z}_0)}(r_n) \},
  \end{equation*}
where $(y_0, \textbf{z}_0)\in\mathbb{R}^{1+d}$ lies the support of $(Y_0, \textbf{Z}_0)$ with non-zero density and $\tau_n^X\to\infty, r_n\to 0$ are sequences such that $|S_2|\to\infty$ in probability.
\end{definition}

\begin{theorem}\label{Theorem2v2}
Consider a data-generating process as in Definition~\ref{Definition_str}. 
Assume that the process \((\mathbf X,\mathbf Y,\mathbf Z)\) is stationary 
and ergodic. Assume that the relevant finite-dimensional distributions are 
absolutely continuous with respect to Lebesgue measure and have continuous 
densities.  Assume that the structural function \(h_Y\) satisfies 
Assumption~\ref{AssumptionA1} and is continuously differentiable in 
\((y,\mathbf z)\) on a neighbourhood of \((y_0,\mathbf z_0)\), with gradient 
uniformly bounded on that neighbourhood.

Then, the estimator \(\hat{\Gamma}_{\textbf{X}\to\textbf{Y}\mid \textbf{Z}}\) defined in equation \eqref{Gamma_hat_equation} with \(S \equiv S_2\),  is consistent in the sense that   \begin{equation*}
   \hat{\Gamma}_{\textbf{X}\to\textbf{Y}\mid \textbf{Z}} \overset{P}{\to} \Gamma_{\textbf{X}\to\textbf{Y}\mid\mathcal{C}_0}, \quad \text{ as }n\to\infty,
  \end{equation*}
where $\Gamma_{\textbf{X}\to\textbf{Y}\mid\mathcal{C}_0} = \lim_{v\to\infty} \mathbb{E}[ F(Y_{t+1})\mid X_t>v,Y_t = y_0, \textbf{Z}_t = \textbf{z}_0]$, provided that the limit exists. Additionally, assuming second-order assumptions presented in Supplement~\ref{proof_of_thm2v2}, the following holds:
\[
    \frac{\sqrt{|S_2|}}{\widehat\sigma_n}\,
    \left(
        \widehat\Gamma_{\mathbf X\to\mathbf Y\mid\mathbf Z}
        -
        \Gamma_{\mathbf X\to\mathbf Y\mid\mathcal C_0}
    \right)
    \overset{d}{\to} N(0,1), \quad \text{ where }\quad
        \widehat\sigma_n^2
    :=
    \frac1{|S_2|}
    \sum_{t\in S_2}
    \left\{
        F(Y_{t+1})
        -
        \widehat\Gamma_{\mathbf X\to\mathbf Y\mid\mathbf Z}
    \right\}^2.
\]
\end{theorem}

\begin{theorem}\label{Theorem3}
Consider a time series following the non-negative SRE model \eqref{sre2} that satisfies Assumptions (B1), (B2), and (B4). Then, the estimator \(\hat{\Gamma}_{\textbf{X}\to\textbf{Y}\mid \textbf{Z}}\) defined in equation (\ref{Gamma_hat_equation}), with \(S \equiv S_1\) satisfies
\begin{equation}
\label{consistency_theorem2}
\hat{\Gamma}_{\textbf{X}\to\textbf{Y}\mid\textbf{Z}} \overset{P}{\to} 1 \text{   as }n\to\infty\iff  \Gamma_{\textbf{X}\to\textbf{Y}\mid\mathcal{C}} = 1.
\end{equation}
\end{theorem}
Proof of Theorems~\ref{Theorem2v2} and~\ref{Theorem3} are presented in Supplements~\ref{proof_of_thm2v2} and~\ref{proof_of_thm3}. 

As for the unadjusted estimator, a practical choice of the hyperparameter 
\(\tau_n^X\) is the \(k_n\)-th largest value of \(X_t\) among the indices
\(
    \widetilde S_1
    :=
    \left\{
    t\in\{1,\dots,n-1\}:
    \begin{pmatrix}
        Y_t\\
        \mathbf Z_t
    \end{pmatrix}
    \leq \boldsymbol{\tau}
    \right\},
\)
where \(k_n\) is any sequence satisfying \eqref{k_deleno_n}, and
\(\boldsymbol{\tau}\) is chosen as a high quantile. Further details are
discussed in Section~\ref{Section_choice_of_hyperparameters}.

\subsection{Causal discovery in extremes}
\label{section_algorithm}
In this section, we propose a procedure that takes the data $(x_1, y_1, \textbf{z}_1)^\top, \dots, (x_n, y_n, \textbf{z}_n)^\top$ and outputs  $\textbf{X}\toverset{ext}{\to}\textbf{Y}$ or  $\textbf{X}\toverset{ext}{\not\to}\textbf{Y}$. Intuitively, it relies on two key values for the estimator: 
\begin{itemize}
    \item if $\textbf{X}\toverset{ext}{\to}\textbf{Y}$, then $\hat{\Gamma}_{\textbf{X}\to\textbf{Y}\mid\textbf{Z}}\approx 1$, 
    \item if $\textbf{X}\toverset{ext}{\not\to}\textbf{Y}$, then $\hat{\Gamma}_{\textbf{X}\to\textbf{Y}\mid\textbf{Z}}\approx  {\Gamma}_{\textbf{X}\to\textbf{Y}\mid\mathcal{C}}^{\baseline}<1$.
\end{itemize}
In order to distinguish between these two cases, we rely on an estimate  $\hat{\Gamma}_{\textbf{X}\to\textbf{Y}\mid\textbf{Z}}^{\baseline} := \frac{1}{|\tilde{S}|}\sum_{t\in \tilde{S}} F(y_{t+1})$. If $\hat{\Gamma}_{\textbf{X}\to\textbf{Y}\mid\textbf{Z}}$ is closer to $1$ than to $\hat{\Gamma}_{\textbf{X}\to\textbf{Y}\mid\mathcal{C}}^{\baseline}$, we output $\textbf{X}\toverset{ext}{\to}\textbf{Y}$. Otherwise we output $\textbf{X}\toverset{ext}{\not\to}\textbf{Y}$. 
Algorithm~\ref{Algorithm1} details this procedure. 

\begin{algorithm}[tbph] %
  \KwData{$(x_1, y_1, \textbf{z}_1)^\top, \dots, (x_n, y_n, \textbf{z}_n)^\top$.}
  \KwOut{either $\textbf{X}\toverset{ext}{\to}\textbf{Y}$ or  $\textbf{X}\toverset{ext}{\not\to}\textbf{Y}$.}
  \vspace{4pt}
Obtain an estimate $\hat{\Gamma}_{\textbf{X}\to\textbf{Y}\mid\textbf{Z}}$ using (\ref{Gamma_hat_equation}) and either set $S_1$ or $S_2$\;
Compute $\hat{\Gamma}_{\textbf{X}\to\textbf{Y}\mid\textbf{Z}}^{\baseline} := \frac{1}{|\tilde{S}|}\sum_{t\in \tilde{S}} F(y_{t+1})$ using $\tilde{S}$ as either $\tilde{S}_1$ or $\tilde{S}_2$\;
\leIf{$\hat{\Gamma}_{\textbf{X}\to\textbf{Y}\mid\textbf{Z}} > \frac{1+\hat{\Gamma}_{\textbf{X}\to\textbf{Y}\mid\textbf{Z}}^{\baseline}}{2}$}{\Return $\textbf{X}\toverset{ext}{\to}\textbf{Y}$}{\Return $\textbf{X}\toverset{ext}{\not\to}\textbf{Y}$}

  \caption{Discovery of causality in extremes}
  \label{Algorithm1}
\end{algorithm}

The consistency of Algorithm~\ref{Algorithm1} follows directly from Theorems~\ref{Theorem2v2} and~\ref{Theorem3}, as demonstrated in the following Lemma~\ref{Classification_lemma}. The proof is presented in Supplement~\ref{proof_of_classicication_lemma}.  
\begin{lemma}\label{Classification_lemma}
Let the assumptions from Theorem~\ref{Theorem2v2} hold. Then, Algorithm~\ref{Algorithm1} with $S=S_2$ is consistent; that is, the output is correct with probability tending to one as $n\to\infty$. 

Let the assumptions from Theorem~\ref{Theorem3} hold. Then, there exists $\boldsymbol{\tau}_0\in\mathbb{R}^{1+d}$ such that for all  $\boldsymbol{\tau}\leq \boldsymbol{\tau}_0$, Algorithm~\ref{Algorithm1} with $S=S_1$ and with hyper-parameter $\boldsymbol{\tau}$ gives the correct output with probability tending to one as $n\to\infty$. 

\end{lemma}

\subsection{Testing for tail causality}
\label{Section_testing}
We develop a statistical test of the hypothesis $H_0^{tail}: \textbf{X} \toverset{tail}{\not\to} \textbf{Y}$ as follows. Using bootstrapping (described below), we construct $\alpha$-confidence intervals for ${\Gamma}_{\textbf{X}\to\textbf{Y}\mid\textbf{Z}} - \Gamma^{\baseline}_{\textbf{X}\to \textbf{Y}\mid\textbf{Z}}$, $\alpha\in (0,1)$, using the estimators described in Section~\ref{Section_estimation}. We reject the null hypothesis $H_0^{tail}$ if zero lies outside of this interval. Since $\textbf{Y}$ has support in a neighborhood of infinity, the baseline coefficient $\Gamma^{\baseline}_{\textbf{X}\to \textbf{Y}\mid\textbf{Z}}$ is strictly within the open interval $(0,1)$, ensuring the hypothesis test is well posed. 

Computing confidence intervals for an estimand, solely based on its estimator is a classical statistical problem \citep{Vaart_1998}. Out of all procedures for its estimation, we opt for using the moving block bootstrap technique \citep{kinsch1989jackknife, Vaart_1998}. As opposed to classical bootstrap, consecutive observation blocks are resampled with replacement to preserve the time series' temporal dependencies.

Data is split into $n - b + 1$ overlapping blocks of length $b$. Then from these $n - b + 1$ blocks, $n/b$ blocks will be drawn at random with replacement. Then aligning these $n/b$ blocks in the order they were picked, will give the bootstrap observations. The length $b$ is typically chosen as $b=\sqrt{n}$.  In the bootstrap observation, we compute $\tilde{\hat{\Gamma}}_{\textbf{X}\to\textbf{Y}\mid\textbf{Z}}$. 
Repeating this procedure $B\in\mathbb{N}$ times, we end up with $B$ estimations. Denoting the sample $\alpha$-quantile of these $B$ estimations by $\hat{\zeta}_\alpha^B$, the resulting block-bootstrap interval is $(\hat{\zeta}_{\alpha/2}^B, \hat{\zeta}_{1-\alpha/2}^B)$. 
See Algorithm~\ref{Algorithm_testing_appendix} in Supplement~\ref{SupplSect_blockBootstrapAlgo} and the code supplement for more details. 

It has been widely recognized that confidence intervals  $(\hat{\zeta}_{\alpha/2}^B, \hat{\zeta}_{1-\alpha/2}^B)$ maintain the correct confidence level as $B,n\to\infty$ under some regularity assumptions \citep{davison1997bootstrap, dehaan2024bootstrapping}. This has primarily been demonstrated through extensive simulation studies rather than theoretical proofs, which can be challenging even for simple statistics.

\section{Multivariate extension: estimating full causal graph}
\label{Section5}
\label{section_algorithm2_estimating_full_causal_graph}
One is often interested not only in the causal relation between $\textbf{X}$ and $\textbf{Y}$, but in a causal graph involving a collection of time series $\textbf{X}^1, \dots, \textbf{X}^m$, where $m\in\mathbb{N}$. 
We define the summary graph $\mathcal{G}=(V, \mathcal{E})$, where the vertices $V=\{1, \dots, m\}$ correspond to the respective series $\textbf{X}^1, \dots, \textbf{X}^m$, and an edge $(i, j)\in \mathcal{E}$ exists if and only if $\textbf{X}^i \toverset{ext}{\to} \textbf{X}^j$. 
An example of a summary graph is shown in Figure~\ref{Crypto_graphs}. Under Assumption \ref{AssumptionA1}, this summary graph $\mathcal{G}$ aligns with the classical Granger summary graph.

One approach to estimating $\mathcal{G}$ involves determining the presence of a direct causal link $\textbf{X}^i \toverset{ext}{\to} \textbf{X}^j$, while considering the influence of all other time series, for each distinct pair $i,j\in\{1, \dots, m\}$. However, a large number of time series $m$ may diminish statistical power.

In lieu of this, we propose a faster and more efficient algorithm leveraging the property~(\ref{OhYeah}). As demonstrated in Section~\ref{section3},  under relatively mild assumptions $\Gamma_{\textbf{X}\to\textbf{Y}\mid\emptyset} < 1 \implies \Gamma_{\textbf{X}\to\textbf{Y}\mid \textbf{Z}} < 1$. Consequently, we first initiate our analysis with a simple pairwise examination before accounting for the influence of the other time series in a second step. This procedure is detailed in Algorithm~\ref{Algorithm2}. 

\begin{algorithm}[tbph] %
  \KwData{$(x_1^1, \dots, x_1^m)^\top, \dots, (x_n^1, \dots, x_n^m)^\top$.}
  \KwOut{Summary graph $\hat{\mathcal{G}}$.}
  \vspace{4pt}
Start with a complete graph $\hat{\mathcal{G}}$, where a directed edge connects each pair of vertices (each vertex represents one distinct time series)\; 
\SetKwBlock{StepOne}{Step 1 (Pairwise):}{endstep}
\StepOne{
\ForAll{$i,j\in\{1,\ldots,m\} : i\neq j$}{
Determine if $\Gamma_{\textbf{X}^i\to\textbf{X}^j\mid \textbf{Z}} = 1$ given $\textbf{Z} = \emptyset$%
\tcp*{{\scriptsize using Algorithm~1 or  Section~\ref{Section_testing}}}
\lIf{$\Gamma_{\textbf{X}^i\to\textbf{X}^j\mid \textbf{Z}} < 1$}{remove edge $(i,j)$ from $\hat{\mathcal{G}}$} 
} 
${\hat{\mathcal{G}}^{\rm P}}\gets \hat{\mathcal{G}}$\; 
} 
\SetKwBlock{StepTwo}{Step 2 (Multivariate):}{endstep}
\StepTwo{
\ForEach{edge $(i,j)$ in ${\hat{\mathcal{G}}^{\rm P}}$}{
Determine if $\Gamma_{\textbf{X}^i\to\textbf{X}^j\mid \textbf{Z}} = 1$ given $\textbf{Z}=\pa_{{\hat{\mathcal{G}}^{\rm P}}}(i)\cap \pa_{{\hat{\mathcal{G}}^{\rm P}}}(j)$%
\tcp*{{\scriptsize where $\pa_{{\hat{\mathcal{G}}^{\rm P}}}(i)$ denotes the parents of~$i$ (set of vertices with an incoming edge to $i$ in ${\hat{\mathcal{G}}^{\rm P}}$)}}
\lIf{$\Gamma_{\textbf{X}^i\to\textbf{X}^j\mid \textbf{Z}} < 1$}{remove edge $(i,j)$ from $\hat{\mathcal{G}}$} 
} 
} 
\Return $\hat{\mathcal{G}}$\; 
\caption{Extreme causality: summary graph estimator}
  \label{Algorithm2}
\end{algorithm}

To determine whether $\Gamma_{\textbf{X}^i\to\textbf{X}^j\mid \textbf{Z}} = 1$, either Algorithm~\ref{Algorithm1} or the test procedure from Section~\ref{Section_testing} can be employed. Our primary focus lies on Algorithm~\ref{Algorithm1}. 

\begin{lemma}
\label{lemma_path_diagram}
Let \((\textbf{X}^1, \dots, \textbf{X}^m)\) be a collection of time series. Assume that, for each distinct pair \(i, j \in \{1, \dots, m\}\), Algorithm~\ref{Algorithm1} is consistent and that $\Gamma_{\textbf{X}^i \to \textbf{X}^j \mid \mathcal{C}} = 1 \implies \Gamma_{\textbf{X}^i \to \textbf{X}^j \mid \emptyset} = 1.$
Note that these conditions are satisfied under the assumptions of Lemma~\ref{Classification_lemma} and Theorem~\ref{Theorem1}. Then, Algorithm~\ref{Algorithm2} is consistent, meaning that \(P({\hat{\mathcal{G}}} = \mathcal{G}) \to 1\) as \(n \to \infty\).

Furthermore, if, for each distinct pair $i, j\in{1, \dots, m}$, $\Gamma_{\textbf{X}^i\to\textbf{X}^j\mid\mathcal{C}} = 1 \iff  \Gamma_{\textbf{X}^i\to\textbf{X}^j\mid \emptyset} = 1,   $
then $P({\hat{\mathcal{G}}^{\rm P}} = \mathcal{G})\to 1$ as $n\to\infty$, and Step 2 of the algorithm is asymptotically not necessary. 
\end{lemma}
The proof of Lemma~\ref{lemma_path_diagram} is presented in Supplement~\ref{proof_of_lemma_path_diagram}. 

Algorithm~\ref{Algorithm2} is highly efficient, with a time complexity of $O(m^2\, n\,\log(n))$. The term $n\log(n)$ accounts for the time complexity of Algorithm~\ref{Algorithm1}, as computing $\hat{\Gamma}_{\textbf{X}\to\textbf{Y}\mid\textbf{Z}}$ requires a sorting algorithm, while the $m^2$ term arises from iterating over each pair of $i$ and $j$. To our knowledge, our algorithm lies among the most efficient algorithms for causal discovery.

\section{Simulations}
\label{Section6}

\subsection{Hyperparameter analysis}
\label{Section_choice_of_hyperparameters}

In the estimation of  $\hat{\Gamma}_{\textbf{X}\to\textbf{Y}\mid\mathcal{\textbf{Z}}}$, we need to make specific practical choices of several hyperparameters. We discuss the values we use in our computations, which could be reasonable default choices. However, the optimal choice might vary depending on the specific characteristics of each time series. 

\begin{itemize}
    \item  $F$: One needs to choose the cumulative distribution function $F$ in the definition of ${\Gamma}_{\textbf{X}\to\textbf{Y}\mid\mathcal{C}}$ in \eqref{CTC_definition}. Although the choice of $F$ is not important for theoretical properties of ${\Gamma}_{\textbf{X}\to\textbf{Y}\mid\mathcal{C}}$ demonstrated in this paper, it may affect the finite sample properties of its estimator. A natural choice for $F$ is the empirical marginal distribution function of $Y$ denoted as $\hat{F}_Y$. However, we opt for     \[  \hat{F}_Y^{truc}(t):= \begin{cases} \hat{F}_Y(t) & \text{if } t \geq \operatorname{median}(Y) \\ 0 & \text{if } t < \operatorname{median}(Y). \end{cases} \]
Simulations in Section~\ref{Section_choice_F_{}} suggest that the choice $ \hat{F}_Y^{truc}(t)$ leads to better finite sample properties. We also experimented with various alternatives, including $F(x) = \mathbbm{1}(x > \tau)$ for large  $\tau\in\mathbb{R}$, which induces causality-in-high-quantile \citep{Candelon2013}. However, all considered alternatives resulted in inferior finite sample behavior. 
    \item $S$: Sets $S_1$ or $S_2$ are equivalent if the supports of  $Y$ and  $\textbf{Z}$ are bounded from below (which is true in most of our simulations setups and in the application) and $(y_0, \textbf{z}_0)$ is chosen as the lower endpoint. In such a case, choosing an optimal $\boldsymbol{\tau}$ and an optimal $r$ are equivalent tasks. In the other cases, we use the set $S_1$ when causality in the upper tail only is of interest and $S_2$ for causality in both tails.  
    \item  $\tau^X_k$ (or equivalently $k_n$) : We choose \(\tau_n^X\) as the \(k_n\)-th largest value of \(X_t\) among the indices
\(
    \widetilde S_1
    :=
    \left\{
    t\in\{1,\dots,n\}:
    \begin{pmatrix}
        Y_t\\
        \mathbf Z_t
    \end{pmatrix}
    \leq \boldsymbol{\tau}
    \right\}
\).  If the presence of a strong hidden confounder is suspected, $k_n = n^{\frac{1}{2}}$ appears to be a reasonable choice, as in \cite{bodik}. If one does not suspect strong hidden confounding $k_n=n^{\frac{1}{3}}$ yields better results. This is concluded from the simulations in Section~\ref{Section_choice_k_n}. 
    \item $\tau_Y$: We choose $\tau_Y$ to be a $q_Y\in (0,1)$ quantile of $Y$. The choice  leads to a bias-variance trade-off, as smaller  $\tau_Y$ leads to more strict conditioning while reducing the effective sample size. We choose $q_Y=0.8$, as this choice is optimal under a specific autoregressive data-generating process, as discussed in Simulations~\ref{Section_choice_tau_Y}. However, under large auto-correlation in $Y$, larger quantiles $q_Y$ may lead to a better finite sample behavior. 
    \item $\boldsymbol{\tau}_Z$: Recall that we assume a $d$-dimensional confounder $\textbf{Z}\in\mathbb{R}^d$, and we denote $\boldsymbol{\tau}_Z = (\tau_Z^1, \dots, \tau^d_Z)$. We select each $\tau_Z^i$ to represent the $q_Z^i\in (0,1)$ quantile of $Z_i$. The optimal choice of $q^i_Z$ depends on the strength of the confounding effect of $Z_i$: the stronger the confounding effect, the smaller the optimal $q^i_Z$. As discussed in Simulations~\ref{Section_choice_tau_Z}, a quantile at level $0.9$ appears to be a suitable choice in the univariate case, while we opt for a quantile at level $1- \frac{0.2}{d}$ whenever $d>1$ to prevent the effective sample size from becoming too small after conditioning on confounders being non-extreme. It is important to note that we should decrease $q^i_Z$ when a strong confounding effect of $Z_i$ is expected. 
    \item $p$: Causal lag from Section~\ref{lag_extention_in_section2}. Increasing the lag relaxes the assumptions regarding the structure of (\ref{structural_generation_lag}), albeit at the cost of reducing statistical power. 
    The selection of an appropriate lag presents a common challenge in time series analysis \citep{LagSelection,Runge2019PCMCI}, for which classical approaches such as analyzing auto-correlation plots or extremograms \citep{Extremogram} are available. Alternatively, conclusions can be drawn across a range of lag choices.
\end{itemize}

\subsection{Comparative performance study}
\label{Section_Sim_perform}

We assess the performance of our methodology through a series of comparative simulations. 
We generate time series data with various choices for parameters of interest: 1) the number of variables $m$ in the randomly generated underlying causal graph, 2) the sample size $n$, 3) heavy-tailed versus non-heavy-tailed noise variables, and 4) a VAR versus a GARCH dependence model. 
We compare our methodology to the state-of-the-art causality methods \citep{assaad2022survey}. Following the Tigramite package \citep{runge2023overview_causal}, we use the PCMCI method \citep{Runge2019PCMCI}, with the independence tests ``RobustParCorr'' and ``GPDC'', which we believe are the most appropriate. 
For each method and dataset, we measure the estimated causal graph's error as the number of edge additions or removals required to obtain the true graph, standardised between 0 and 1 by dividing by $m(m-1)$. 
For each combination of data parameters, the time series were generated as follows.

\textbf{Step 1:} We generated a random graph $\mathcal{G}$ with $m \in \mathbb{N}$ vertices, where each edge is present independently with probability $\frac{1}{m}$. We defined $\delta^{\mathcal{G}}_{j,i} = 1$ if $(j,i) \in \mathcal{G}$ and $\delta^{\mathcal{G}}_{j,i} = 0$ otherwise (i.e., $\delta^{\mathcal{G}}_{j,i} = 1$ if there is a directed edge $j \to i$ in $\mathcal{G}$).

This graph-generating mechanism produces sparse graphs, as the expected number of directed edges is \(m-1\), whereas the total number of possible directed edges is \(m(m-1)\). 
Since the average error is normalised by \(m(m-1)\), it is increasingly dominated by absent edges as \(m\) grows. 
As it is typically easier to detect the absence of an edge than the existence of an edge, we expect methods with low relative false-positive rates to display lower average errors with large $m$. 
Interpreting the results in terms of the difference between the approaches for each individual $m$ might, thus, be more relevant than interpreting them as a function of $m$.

\textbf{Step 2 (VAR case):} We initialized $X^1_1, \dots, X^m_1 = 0$ and iteratively generated the series for each $t \in \{2, \dots, n\}$ and $i \in \{1, \dots, m\}$ as follows:
\begin{equation*}
X_{t+1}^i = 0.1X_t^i + \sum_{j \neq i} \delta^{\mathcal{G}}_{j,i} 0.5 X_t^j + \varepsilon_t^i,
\end{equation*}
where $\varepsilon_t^i \toverset{iid}{\sim} \mathcal{N}(0,1)$ in the non-heavy-tailed case and $\varepsilon_t^i \toverset{iid}{\sim} {\rm Pareto}(1)$ in the heavy-tailed case. 
An inbetween noise distribution $\varepsilon_t^i \toverset{iid}{\sim} {\rm Pareto}(2)$ was also considered; see Section~\ref{appendix_simulations_Pa2} of the supplementary material.
The constants 0.1 and 0.5 were chosen to ensure that the time series remains stationary and does not explode, even for random graphs $\mathcal{G}$ with $m=20$ vertices.

\textbf{Step 2 (GARCH case):} We initialized $X^1_1, \dots, X^m_1 = 0$ and iteratively generated the series for each $t \in \{2, \dots, n\}$ and $i \in \{1, \dots, m\}$ as follows:
\begin{equation*}
X_{t+1}^i =  \bigg(0.1 + \sum_{j \neq i} \delta^{\mathcal{G}}_{j,i} 0.5 (X_t^j)^2\bigg)^{1/2}\varepsilon_t^i,
\end{equation*}
where $\varepsilon_t^i \toverset{iid}{\sim} \mathcal{N}(0,1)$ in the non-heavy-tailed case and $\varepsilon_t^i \toverset{iid}{\sim} {\rm Cauchy}$ in the heavy-tailed case. We chose 0.1 as the auto-correlation constant to prevent exponential increases in the time series, and 0.5 for the effect strength as it did not affect the stationarity.

\textbf{Step 3:} We repeat the experiments for each combination of data parameters by generating 100 instances of the time series according to steps 1 and 2 and estimating $\mathcal{G}$ for each of those 100 repetitions.

The code and instructions to reproduce the study are available as supplementary material.

\begin{figure}[!bt]
\centering
\includegraphics[width=1\textwidth]{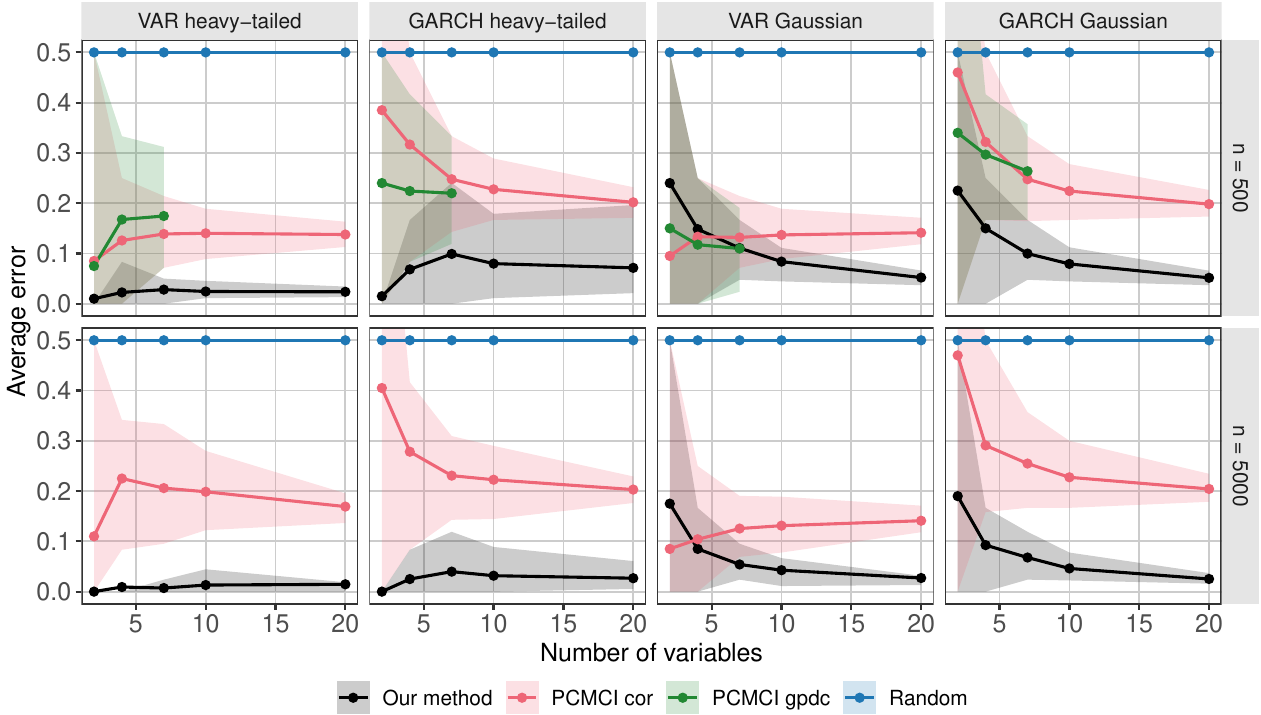}
\caption{Comparison of the average model errors between our approach and the competitors for different numbers of variables (x-axis), data processes (columns) and sample sizes (rows).
The variability bands show the 10--90\% inter-quantile spread across repetitions.
The ``random'' algorithm generates a random graph with each edge present with probability $\frac{1}{2}$. Due to time complexity constraints, PCMCI with GPDC independence test is estimated only for $n=500, m\leq 7$. 
}
\label{Sim_6_results}
\end{figure}

Figure~\ref{Sim_6_results} summarises the results. 
Overall, our causality-in-extremes approach shows robust performance across all settings. 
It achieves significantly lower average error than the other state-of-the-art methods in all cases except in the low-dimensional VAR Gaussian setting, although the relative performance seems to improve with sample size in the latter case. %
In the Gaussian cases, our method seems to outperform competitors by a larger margin in high dimensions, which could be a valuable property, as high-dimensional settings are typically harder for most methods. 
Again, the relative error between methods should be compared for each value of $m$, rather than their evolution with $m$, due to the nature of the evaluation metric and the relative edge sparsity for graphs with large $m$ values.
Furthermore, the error variability of our method across repetitions is, also, comparatively much smaller, in most cases.

For noise distributions less heavy-tailed than $\text{Pareto}(1)$ and Cauchy, additional experiments seem to indicate that the performance of Algorithm 2 is close to the $\text{Pareto}(1)$ and Cauchy cases. 
The results for the VAR structure with $\text{Pareto}(2)$ noise distribution, presented in Section~\ref{appendix_simulations_Pa2} and Figure~\ref{Sim_6_results_Pa2} of the supplementary material, are almost identical to the $\text{Pareto}(1)$ results. 
Overall, it, thus, seems that our causality-in-extremes approach significantly outperforms the considered competitors in the scenarios considered in this study.

Our approach also offers a significant advantage in computational efficiency. For a dataset with \( n = 500 \) and \( m = 20 \), our algorithm estimates the causal graph in about 5.96 seconds (Intel Core i5-6300U 2.5 GHz, 16 GB RAM), compared to PCMCI with RobustParCorr at 13.34 seconds and to PCMCI with GPDC at over an hour of compute time.

\section{Application to real-data scenarios}
\label{Section_application}

\subsection{Causality in extreme hydrological events}
\label{section_river_application}

We apply our methodology to infer the causal relationship between extreme precipitation and extreme river discharge.
We analyze discharge data recorded by the Swiss Federal Office for the Environment (\url{hydrodaten.admin.ch}), which were studied and provided by the authors of \cite{Pasche, engelke2021sparse, Pasche2024}, with preliminary insights. Precipitation data are sourced from the Swiss Federal Office of Meteorology and Climatology, MeteoSwiss (\url{gate.meteoswiss.ch/idaweb}).

\begin{figure}[tb]
\centering
\includegraphics[width=0.95\textwidth]{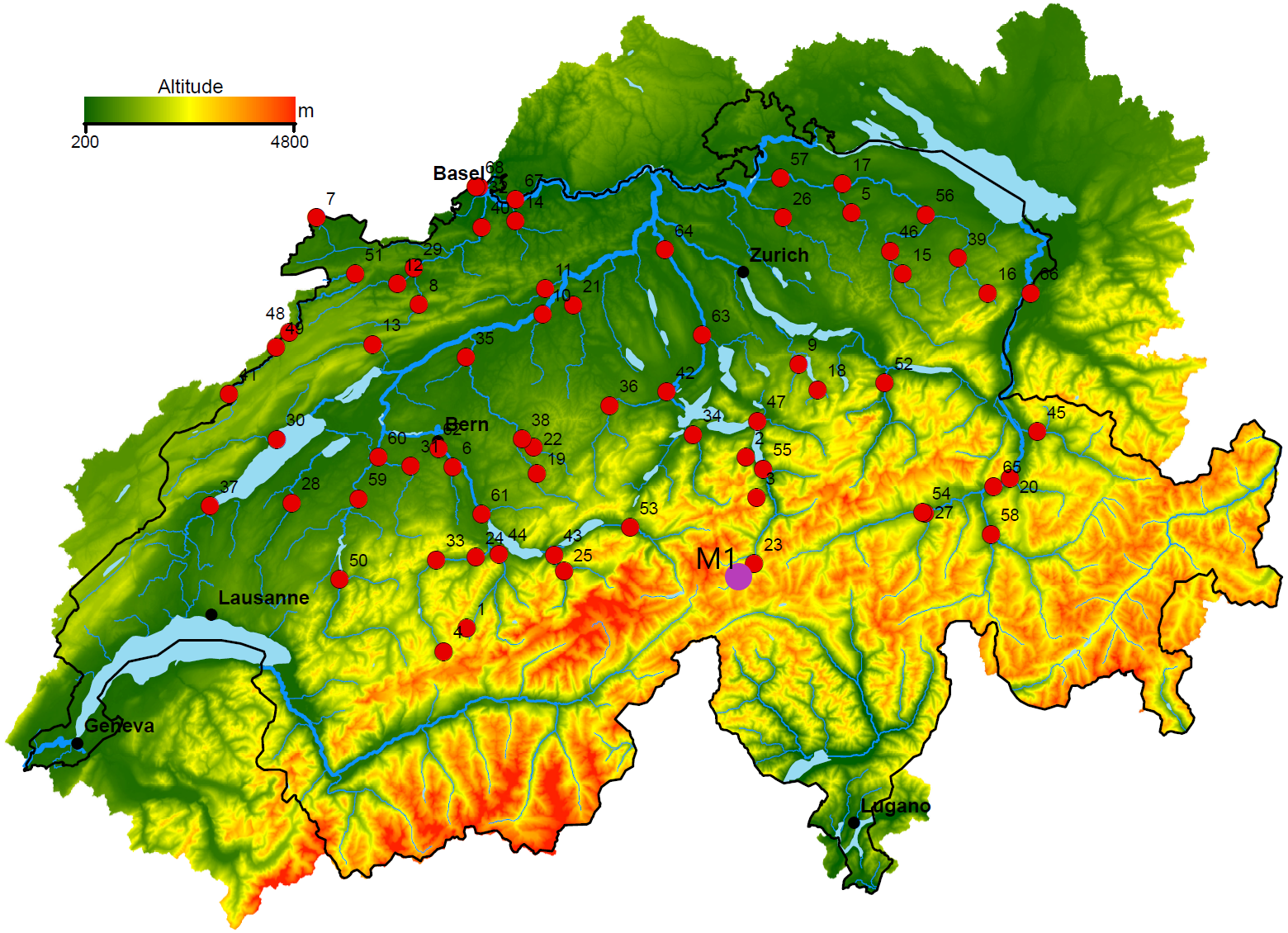}
\caption{Topographic map showing all 68 gauging stations in Switzerland \citep{Pasche}. The purple dot ``M1'' represents the meteorological station.}
\label{map}
\end{figure}

Figure~\ref{map} displays a map of all gauging stations with river discharge measurements, along with the meteorological station M1 located at the source of the Reuss river in Switzerland. Let $\textbf{X} = (X_t)_{t\in\mathbb{Z}}$ represent the daily total precipitation at M1 and $\textbf{Z} = (\textbf{Z}_t)_{t\in\mathbb{Z}}$ denote other meteorological measurements, in particular the daily maximum temperature and the relative air humidity 2m above the surface. Let $\textbf{Y}^k = (Y_t^k)_{t\in\mathbb{Z}}$ represent the daily average river discharge at station $k\in\{1, \dots, 68\}$. Most river stations have been monitored for over 50 years, providing extensive historical data. Following \cite{Pasche}, we only focus on the summer months. 

\subsubsection{Difference between the two types of causality in extremes}
A natural working assumption is that precipitation is the cause of river discharge and river discharge is not the cause of the precipitation. However, the causal relations vary across the river stations. For instance, consider station number 23 located close to M1. We posit that the ground truth is $\textbf{X}\toverset{ext}{\to}\textbf{Y}^{23}$ since extreme precipitation at M1 should always lead to large discharge values at station 23. This also applies to all stations along the Reuss river, as extreme discharge at station 23 propagates downstream to stations 3, 55, and so on.

Conversely, consider station 7 in the northwest of Switzerland. We posit that $\textbf{X}\toverset{ext}{\not\to}\textbf{Y}^{7}$, since extreme precipitation at M1 {does not always} lead to extreme discharge levels at station 7, but that $\textbf{X}\toverset{tail}{\to}\textbf{Y}^{7}$, as the cloud causing extreme precipitation at M1 may sometimes reach station 7, but not always. 
In summary, our hypothesis for the ground truth is the following: $\textbf{X}\toverset{tail}{\to}\textbf{Y}^{k}$ for all $k$, while $\textbf{X}\toverset{ext}{\to}\textbf{Y}^{k}$ only for stations located downstream of M1, on the Reuss river.

\subsubsection{Testing for causality in the tails}

We test whether $\textbf{X}\toverset{tail}{\to}\textbf{Y}^{k}$ and whether $\textbf{Y}^k\toverset{tail}{\to}\textbf{X}$ for all $k\in\{1, \dots, 68\}$ using the procedure outlined in Section~\ref{Section_testing}, with significance level $\alpha = 0.05$. This results in $2\cdot68 = 136$ tests. 
Choosing hyper-parameters as detailed in Section~\ref{Section_choice_of_hyperparameters}, and considering the temperature and humidity $\textbf{Z}$ as potential confounders, we obtain the following results.

Out of $136$ tests conducted, $134$ yielded outcomes supporting the assumed ground truth. There were two instances of disagreements: for station $k=65$ the null hypothesis $H_0: \textbf{Y}^{65}\toverset{tail}{\not\to}\textbf{X}$ was rejected, and for station $k=4$ the converse $H_0: \textbf{X}\toverset{tail}{\not \to}\textbf{Y}^{4}$ was not rejected. As some of the tests can have false positives with a significance level lower than $\alpha=0.05$ simply by randomness, the first case is expected over $68$ such tests. The second case suggests that extreme precipitation in M1 does not lead to an increased chance of extreme precipitation in station 4.  As the highest peaks of the Swiss Alps mountains are situated between these two stations, clouds at M1 may be prevented from moving to the catchment of station~4, which could explain this outcome. 
 
An intriguing observation emerges when examining the coefficients $\hat{\Gamma}_{\textbf{X}\to\textbf{Y}^k \mid \textbf{Z}}$: all stations situated to the east of meteorological station M1 demonstrate notably high values of $\hat{\Gamma}_{\textbf{X}\to\textbf{Y}^k \mid \textbf{Z}}$, whereas stations to the west exhibit comparatively lower values, often just reaching the threshold of significance. 
This phenomenon might be due to a prevailing movement of clouds from west to east, a phenomenon well-known in the meteorological community as the `westerlies'. 
However, a further dedicated analysis, including additional meteorological data from other locations, would be necessary to confirm this hypothesis.

\subsection{Causality in extreme events of cryptocurrency returns}
\label{section_crypto_application}
We analyse data sourced from the G-Research Crypto Forecasting competition\footnote{See \url{https://www.kaggle.com/c/g-research-crypto-forecasting}.}. 
The dataset comprises 14 high-frequency time series representing various cryptocurrencies' return performances. 
We focus solely on a subset of the data, examined in \cite{g-research-crypto-forecastingII}: that is, adopting minute-wise time intervals and considering the closing price at the end of each minute, transformed into negative log returns. 
However, we consider the last ten days of data rather than only the last day, resulting in $n=14400$ closing log-return observations for each of the 14 variables.
Our goal is to identify any causal relationships in extremes among these 14 time series, to determine which cryptocurrency might serve as the primary driver, causing extreme events in returns for the others.

\begin{figure}[!tb]
\centering
\includegraphics[scale=0.7]{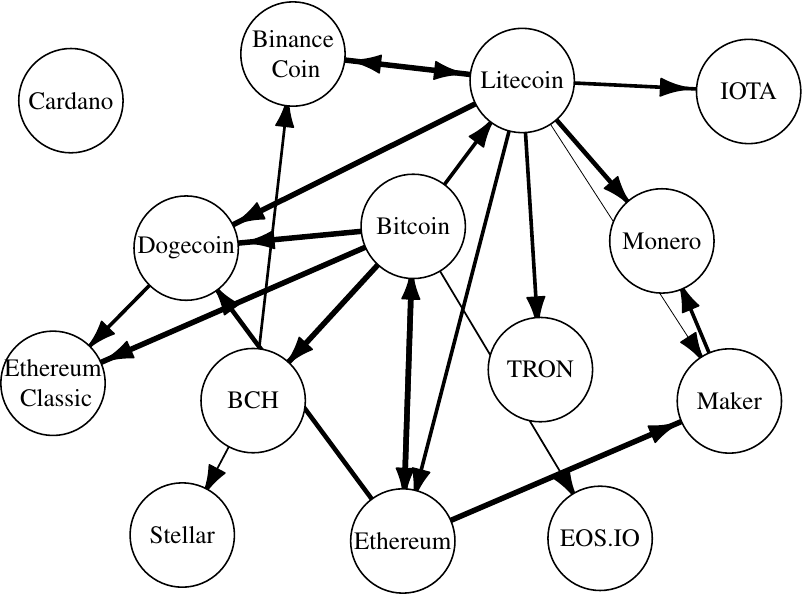}
\caption{Estimated summary causal graph indicating Granger causality in extremes among the negative log returns of cryptocurrencies. 
The graph is obtained using Algorithm~\ref{Algorithm2} incorporating the testing procedure outlined in Section~\ref{Section_testing}.  
The width of each edge represents the magnitude of the p-value; a value close to $0$ results in a wider edge.}
\label{Crypto_graphs}
\end{figure}

We apply Algorithm~\ref{Algorithm2} using the testing procedure described in Section~\ref{Section_testing}.
The findings are presented in Figure~\ref{Crypto_graphs}. 
Applying Algorithm~\ref{Algorithm1} instead of the testing procedure leads to similar conclusions; %
see Figure~\ref{Crypto_graphs_full} in Supplement~\ref{appendix_crypto}.
We choose hyper-parameters as detailed in Section~\ref{Section_choice_of_hyperparameters}, with a lag of 1 min, recognizing the high-speed nature of the market, where changes can propagate within seconds.

The findings highlight Bitcoin and Litecoin as the main drivers,
with Monero, Stellar, Ethereum, Classic, EOS.IO or BCH appearing to be influenced by the others.  
These outcomes align with expectations, as Bitcoin is commonly regarded as a leading indicator in the cryptocurrency market, and Litecoin was the second cryptocurrency launched after Bitcoin.
As Ethereum is a token standardisation leader, network hub, protocol innovator, and market influencer, one would expect it to be an important driver.
Our results align only in part with this expectation, as they show a causal feedback loop with Bitcoin, making it the only currency seemingly influencing the latter, but seem to show direct influence on fewer currencies compared to Bitcoin or Litecoin.

Additionally, we present results when using a lag of 30 min in Supplement~\ref{appendix_crypto}, for comparison.
Although having many similar patterns, like bitcoin being the main driver, more causal relationships seem to be detected.
Given the high-speed nature of the cryptocurrency market, unobserved confounders are more likely to influence the results, and indirect causal effects are more likely to be detected, for such a comparatively large lag.

As a practical takeaway from these findings, if an extreme drop or increase is observed in one of the causal variables, 
we can anticipate a corresponding extreme event in the affected variables. For instance, if there is a notable drop or increase in Bitcoin returns, it may be advisable to promptly consider selling or buying BCH, respectively. 
It is crucial to note that our analysis is based on a fixed period of 10 days, and thus only represents the behaviour of this marker during that period. 
Within this time frame, some causal relationships may not fully manifest, and some observed relationships could be spurious, potentially influenced by unobserved events influencing the market during those days. 
Therefore, for more robust and representative results, a more extensive analysis of the market, coupled with background knowledge, would be necessary.

\section{Conclusion}

We formalized the concept of causality in extremes of time series through two intuitive definitions. Under weak assumptions, we demonstrated that these definitions are equivalent and represent a special case of Granger, Sims and structural causality. We proposed several estimation, causal discovery and testing strategies, which we showed to yield correct results as the sample size grows to infinity. Additionally, our framework can manage hidden confounders under the Grey assumptions. Through simulations, we established the effectiveness and efficiency of our methods, that outperform current state-of-the-art approaches in both accuracy and speed. When applied to real-world cases, our methods successfully uncovered coherent causal relationships between precipitation and river discharge at different locations, as well as between cryptocurrency returns.

However, several open questions remain. Can our framework be useful for other causal inference tasks besides causal discovery? For instance, can we quantify the effect of $X_t$ on $Y_{t+p}$ in extremes? Is our framework robust against hidden confounders under the Kesten-Goldie assumptions? Can we replace the bootstrap testing procedure from Section~\ref{Section_testing} with a faster and more theoretically justifiable alternative? Alternatively to the presented approach, we have also considered a permutation test; we ultimately did not retain this approach due to its lack of a well-functioning generalization for longer causal lags.

Granger causality in mean and Granger causality in variance are prominent concepts within the causal literature, applied across various scientific disciplines in thousands of research articles. The formalization of Granger causality in extremes could significantly advance research by complementing the other two types. It could particularly prove useful in practical applications where understanding the drivers of extreme events is increasingly central, for example in meteorology, weather, finance and insurance.

\section*{Supplementary material}
\subsection*{Supplementary results}
The Supplement discusses generalizations of the results presented in the main paper to non-unit causal lags and to both tails, details about the simulation studies, and the mathematical proofs.
It is provided in appendix to this paper.

\subsection*{Code and data}
The implementation of the methods discussed in this manuscript is available as an open-source R package at \url{https://github.com/opasche/ExtremeGranger}. 
The code to reproduce the simulations, as well as the cryptocurrency data analyzed in Section~\ref{section_crypto_application} is available at \url{https://github.com/jurobodik/Granger-causality-in-extremes}. 
While the hydro-meteorological data analysed in Section~\ref{section_river_application} are not publicly available, they can be ordered through \url{hydrodaten.admin.ch} and \url{gate.meteoswiss.ch/idaweb} after registration or by requesting the formatted data from the authors of \cite{Pasche}.

\section*{Declarations}

\subsection*{Acknowledgements}
The authors would like to thank Val\'erie Chavez-Demoulin and Sebastian Engelke for their support and advice. 
Part of this research was completed while the first author was a visiting scholar at the Department of Statistics, UC Berkeley, and while the second was a visiting scholar at the Department of Industrial Engineering and Operations Research, Columbia University. Both authors thank the departments for their hospitality during this period. 

\subsection*{Funding}
The first author was supported by the Swiss National Science Foundation grant number 201126. 
The second author was supported by the Swiss National Science Foundation Eccellenza Grant 186858.

\begin{footnotesize}
\bibliography{bibliography}

@PREAMBLE{
 "\providecommand{\noopsort}[1]{}" 
 # "\providecommand{\singleletter}[1]{#1}%" 
}

@book{Rubin, 
title={Causal Inference for Statistics, Social, and Biomedical Sciences: An Introduction}, 
DOI={10.1017/CBO9781139025751}, 
publisher={Cambridge University Press}, 
address={Cambridge},
author={G. W. Imbens and D. B. Rubin}, 
year={2015}
}

@article{birkhoff1931proof,
  title={Proof of the ergodic theorem},
  author={G.D. Birkhoff},
  journal={Proc. Natl. Acad. Sci. USA},
  volume={17},
  number={12},
  pages={656--660},
  year={1931},
  doi={10.1073/pnas.17.12.656}
}

@book{pene2022stochastic,
  title={Stochastic Properties of Dynamical Systems},
  author={P. Fran{\c{c}}oise},
  year={2022},
  volume={30},
  series={Cours Sp{\'e}cialis{\'e}s de la SMF}, 
publisher={Société Mathématique de France},
url = {https://smf.emath.fr/publications/proprietes-stochastiques-des-systemes-dynamiques}
}

@article{engelke2021sparse,
  title={Sparse structures for multivariate extremes},
  author={S. Engelke and J. Ivanovs},
  journal={Annual Review of Statistics and its Application},
  volume={8},
  pages={241--270},
  year={2021},
  doi={10.1146/annurev-statistics-040620-041554},
}

@misc{g-research-crypto-forecastingII,
   author = {C. M. Ellis},
    title = {G-Research Crypto Forecasting},
    publisher = {Kaggle},
    year = {2022},
    note = {\url{https://www.kaggle.com/code/carlmcbrideellis/granger-causality-testing-for-1-day/notebook}}
}

@book{krengel1985ergodic,
  title={Ergodic Theorems},
  author={U. Krengel},
  year={1985},
  publisher={Walter de Gruyter \& Co.},
  address={Berlin}
}

@book{gujarati2009causality,
  title={Causality in Economics: The Granger Causality Test},
  author={D.N. Gujarati and D.C. Porter},
  edition={Fifth international},
  year={2009},
  publisher={McGraw-Hill},
  address={New York},
  pages={652--658},
  isbn={978-007-127625-2}
}

@article{kinsch1989jackknife,
  title={The jackknife and the bootstrap for general stationary observations},
  author={H. R. Kinsch},
  journal={Annals of Statistics},
  volume={17},
  year={1989},
  pages={1217}
}

@book{davison1997bootstrap,
  title={Bootstrap Methods and Their Application},
  author={Davison, A. C. and Hinkley, D. V.},
  year={1997},
  publisher={Cambridge University Press},
  address={Cambridge},
  series={Cambridge Series in Statistical and Probabilistic Mathematics},
  isbn={0-521-57391-2},
}

@book{Vaart_1998, place={Cambridge}, series={Cambridge Series in Statistical and Probabilistic Mathematics}, 
title={Bootstrap}, booktitle={Asymptotic Statistics}, publisher={Cambridge University Press},
author={A. W. van der Vaart}, year={1998}, pages={326–340}, collection={Cambridge Series in Statistical and Probabilistic Mathematics}}

@article{dehaan2024bootstrapping,
  title={Bootstrapping Extreme Value Estimators},
  author={L. D. Haan and C. Zhou},
  journal={JASA},
  volume={119},
  number={545},
  pages={382--393},
  year={2024},
  doi={10.1080/01621459.2022.2120400}
}

@article{GrangerSims2,
  author={G. Chamberlain},
  title={{The General Equivalence Of Granger And Sims Causality}},
  year={1981},
  journal={Econometrica},
  doi={10.2307/1912601}
}

@article{Pedersen,
  author  = {Pedersen, R. S. and Wintenberger, O.},
  title   = {On the Tail Behavior of a Class of Multivariate Conditionally Heteroskedastic Processes},
  journal = {Extremes},
  volume  = {21},
  number  = {2},
  pages   = {261--284},
  year    = {2018},
  doi     = {10.1007/s10687-017-0307-3}
}

@article{attanasio2013granger,
  title={Granger Causality Analyses for Climatic Attribution},
  author={A. Attanasio and A. Pasini and U. Triacca},
  journal={Atmospheric and Climate Sciences},
  volume={3},
  number={4},
  pages={515--522},
  year={2013},
  doi={10.4236/acs.2013.34054}
}

@article{Runge2019PCMCI,
  title={Detecting and quantifying causal associations in large nonlinear time series datasets},
  author={J. Runge and P. Nowack and M. Kretschmer and S. Flaxman and D. Sejdinovic},
  journal={Science Advances},
  volume={5},
  number={11},
  pages={eaau4996},
  year={2019},
  publisher={American Association for the Advancement of Science},
  doi={10.1126/sciadv.aau4996}
}

@article{runge2023overview_causal,
  author  = {Runge, J. and Gerhardus, A. and Varando, G. and others},
  title   = {Causal Inference for Time Series},
  journal = {Nature Reviews Earth and Environment},
  volume  = {4},
  pages   = {487--505},
  year    = {2023},
  doi     = {10.1038/s43017-023-00431-y}
}

@article{assaad2022survey,
  author = {C. K. Assaad and E. Devijver and E. Gaussier},
  title = {Survey and Evaluation of Causal Discovery Methods for Time Series},
  journal = {Journal of Artificial Intelligence Research},
  volume = {73},
  year = {2022},
  pages = {1--45},
  doi = {10.1613/jair.1.13428}
}

@inproceedings{Pamfil2020DYNOTEARSSL,
  title={DYNOTEARS: Structure Learning from Time-Series Data},
  author={R. Pamfil and N. Sriwattanaworachai and S. Desai and P. Pilgerstorfer and P. Beaumont and K. Georgatzis and B. Aragam},
  booktitle={International Conference on Artificial Intelligence and Statistics},
  year={2020},
  url={https://api.semanticscholar.org/CorpusID:211010514}
}

@inproceedings{TIMINOPeters,
  author = {J. Peters and D. Janzing and B. Schölkopf},
  title = {Causal inference on time series using restricted structural equation models},
  year = {2013},
  booktitle = {Proceedings of the 26th International Conference NIPS},
  pages = {154–162},
  numpages = {9}
}

@article{Pasche2024,
	author = {Olivier C. Pasche and Sebastian Engelke},
	title = {Neural networks for extreme quantile regression with an application to forecasting of flood risk},
	journal = {{Ann. Appl. Stat.}},
	fjournal = {{The Annals of Applied Statistics}},
	year = {2024},
	volume = {18},
	number = {4},
	pages = {2818--2839},
	doi = {10.1214/24-AOAS1907},
	_url = {https://doi.org/10.1214/24-AOAS1907},
	_note = {ArXiv:2208.07590},
}

@inproceedings{
bodik2026retrospective,
title={Retrospective Counterfactual Prediction by Conditioning on the Factual Outcome: A Cross-World Approach},
author={J. Bodik},
booktitle={Fifth Conference on Causal Learning and Reasoning},
year={2026},
url={https://openreview.net/forum?id=xVYy96X1yq}
}

@article{HARDNESS_OF_CONDITIONAL_TESTING,
author = {R. D. Shah and J. Peters},
title = {{The hardness of conditional independence testing and the generalised covariance measure}},
volume = {48},
journal = {The Annals of Statistics},
number = {3},
publisher = {Institute of Mathematical Statistics},
pages = {1514 -- 1538},
year = {2020},
doi = {10.1214/19-AOS1857}
}

@article{kesten1973random,
  title={Random difference equations and renewal theory for products of random matrices},
  author={Kesten, H.},
  journal={Acta Mathematica},
  volume={131},
  pages={207--248},
  year={1973}
}

@article{Ewa_Damek,
  author={E. Damek and M. Matsui},
  title={Tails of bivariate stochastic recurrence equation with triangular matrices},
  journal={Stochastic Processes and their Applications},
  year=2022,
  volume={150},
  number={C},
  pages={147-191},
  doi={10.1016/j.spa.2022.04.008}
}

@article{Candelon2013,
title = {Testing for Granger causality in distribution tails: An application to oil markets integration},
journal = {Economic Modelling},
volume = {31},
pages = {276-285},
year = {2013},
doi = {10.1016/j.econmod.2012.11.049},
author = {B. Candelon and M. Joëts and S. Tokpavi}
}

@book{Markov_processes,
  author    = {S. N. Ethier and T. G. Kurtz},
  title     = {Markov Processes: Characterization and Convergence},
  year      = {1986},
  publisher = {Wiley Series in Probability and Mathematical Statistics},
  pages     = {158}
}

@article{Extremal_quantile_treatment_effect_Deuber_and_Engelke,
author = {D. Deuber and J. Li and S. Engelke and M. Maathuis},
year = {2022},
month = {10},
pages = {},
title = {Estimation and Inference of Extremal Quantile Treatment Effects for Heavy-Tailed Distributions}, 
journal = {JASA}, 
doi = {10.1080/01621459.2023.2252141}
}

@book{berzuini2012causality,
  title={Causality: Statistical Perspectives and Applications},
  author={C. Berzuini and P. Dawid and L. Bernardinell},
  year={2012},
  isbn={978-0-470-66556-5},
  pages={416},
  month={July},
  publisher={John Wiley \& Sons},
}

@article{white2010granger,
  title={Granger causality and dynamic structural systems},
  author={H. White and X. Lu},
  journal={Journal of Financial Econometrics},
  volume={8},
  number={2},
  pages={193--243},
  year={2010},
  publisher={Oxford University Press}, 
 doi = {10.1093/jjfinec/nbq006}
}

@article{hong2009granger,
  title={Granger causality in risk and detection of extreme risk spillover between financial markets},
  author={Y. Hong and Y. Liu and S. Wang},
  journal={Journal of Econometrics},
  volume={150},
  number={2},
  pages={271--287},
  year={2009},
  publisher={Elsevier}, 
  doi = {doi:10.1016/j.jeconom.2008.12.013}
}

@incollection{GrangerSims,
author="G. M. Kuersteiner",
title="Granger-Sims causality",
booktitle="Macroeconometrics and Time Series Analysis",
year="2010",
editor="S. N. Durlauf and L. E. Blume",
publisher="Palgrave Macmillan",
address="London",
pages="119--134",
isbn="978-0-230-28083-0",
doi="10.1057/9780230280830_14"
}

@article{sims,
 URL = {http://www.jstor.org/stable/1806097},
 author = {Christopher A. Sims},
 journal = {American Economic Review},
 number = {4},
 pages = {540--552},
 publisher = {American Economic Association},
 title = {Money, Income, and Causality},
 volume = {62},
 year = {1972}
}

@Article{LagSelection,
  author={R. S. Hacker and A. Hatemi-J},
  title={{Optimal lag-length choice in stable and unstable VAR models under situations of homoscedasticity and ARCH}},
  journal={Journal of Applied Statistics},
  year={2008},
  volume={35},
  number={6},
  pages={601--615},
  month={},
  doi={10.1080/02664760801920473}
}

@Manual{R,
     title = {R: A Language and Environment for Statistical Computing},
     author = {{R Core Team}},
     organization = {R Foundation for Statistical Computing},
     address = {Vienna, Austria},
     year = {2022},
     url = {https://www.R-project.org/},
   }

@book{MVT_timeseries_SPRINGER,
author = {H. Lütkepohl},
year = {2005},
month = {01},
pages = {},
publisher= {Springer},
address = {Berlin},
title = {New Introduction to Multiple Time Series Analysis},
isbn = {978-3-540-40172-8},
doi = {10.1007/978-3-540-27752-1}
}

@ARTICLE{Granger_criticism,
title = {A review of the Granger-causality fallacy},
author = {M. Maziarz},
year = {2015},
journal = {The Journal of Philosophical Economics},
volume = {8},
number = {2},
pages = {6},
doi = {10.46298/jpe.10676}
}

@Article{bodik2024extreme_treatment_effect,
AUTHOR = {J. Bodik},
TITLE = {Extreme Treatment Effect: Extrapolating Dose-Response Function into Extreme Treatment Domain},
JOURNAL = {Mathematics},
VOLUME = {12},
YEAR = {2024},
NUMBER = {10},
DOI = {10.3390/math12101556}
}

@article{GRANGER1980329,
title = {Testing for causality: A personal viewpoint},
journal = {Journal of Economic Dynamics and Control},
volume = {2},
pages = {329-352},
year = {1980},
issn = {0165-1889},
doi = {10.1016/0165-1889(80)90069-X},
author = {C. W. J. Granger},
}

@book{Heavy_tailed_time_series,
author = {R. Kulik and P. Soulier},
year = {2020},
title = {Heavy-Tailed Time Series},
doi = {10.1007/978-1-0716-0737-4},
publisher = {Springer},
address = {New York}
}

@article{Naveau,
author = {P. Naveau and A. Hannart and A. Ribes},
title = {Statistical Methods for Extreme Event Attribution in Climate Science},
journal = {Annual Review of Statistics and Its Application},
year = {2020},
volume = {7},
number = {1},
pages = {89--110},
doi = {10.1146/annurev-statistics-031219-041314}
}

@article{Extremogram,
   title={The extremogram: A correlogram for extreme events},
   DOI={10.3150/09-bej213},
   journal={Bernoulli},
   publisher={Bernoulli Society for Mathematical Statistics and Probability},
   author={R. A. Davis and T. Mikosch},
   year={2009},
   volume={15},
   number={4},
   pages={977--1009}
}

@book{RegularVariationBook, 
title={Modelling Extremal Events for Insurance and Finance}, 
publisher={Springer}, 
address={Berlin},
author={P. Embrechts and C. Klüppelberg and T. Mikosch}, 
year={1997},
isbn={978-3-642-33483-2},
doi={10.1007/978-3-642-33483-2}
}

@book{Elements_of_Causal_Inference,
author = {J. Peters and D. Janzing and B. Schölkopf},
title = {Elements of Causal Inference: Foundations and Learning Algorithms},
year = {2017},
url= {http://library.oapen.org/handle/20.500.12657/26040}, 
publisher={MIT Press},
address={Cambridge}
}

@article{CausalityInVariance,
  author={C. M. Hafner and H. Herwartz},
  title={Testing for Causality in Variance using Multivariate {GARCH} Models},
  year={2008},
  volume={89},
  pages={215--241},
  journal={Annales d'Économie et de Statistique},
  doi={10.2307/27715168}
}

@incollection{Eicher,
author = {M. Eichler},
publisher = {John Wiley and Sons},
title = {Causal Inference in Time Series Analysis},
booktitle = {Causality: Statistical Perspectives and Applications},
chapter = {22},
pages = {327--354},
doi = {10.1002/9781119945710.ch22},
year = {2012}
}

@article{Mazzarisi_causality_in_tail,
title = {Tail Granger causalities and where to find them: Extreme risk spillovers vs spurious linkages},
journal = {Journal of Economic Dynamics and Control},
volume = {121},
pages = {104022},
year = {2020},
doi = {10.1016/j.jedc.2020.104022},
author = {P. Mazzarisi and S. Zaoli and C. Campajola and F. Lillo}
}

@article{Gnecco,
      title={Causal discovery in heavy-tailed models}, 
      author={N. Gnecco and N. Meinshausen and J. Peters and S. Engelke},
      year={2020},
 volume = {49},
journal = {The Annals of Statistics},
doi = {10.1214/20-AOS2021}
}

@article{bodik,
    doi = {10.1007/s10687-023-00479-5},
    author = {J. Bodik and M. Paluš and Z. Pawlas},
    title = {Causality in extremes of time series},
    journal = {Extremes},
    volume = {27}, 
    pages = {67-121},
    year = {2024}
}

@article{barbero2018temperature,
  title={Temperature-extreme precipitation scaling: a two-way causality?},
  author={R. Barbero and S. Westra and G. Lenderink and H. J. Fowler},
  journal={International Journal of Climatology},
  volume={38},
  pages={e1274--e1279},
  year={2018},
  publisher={Wiley Online Library},
  doi={10.1002/joc.5370}
}

@book{Pearl_book,
author = {J. Pearl},
title = {Causality: Models, Reasoning and Inference},
year = {2009},
publisher = {Cambridge University Press},
ISBN={978-0521895606}
}

@article{Pasche,
  author        = {Olivier Colin Pasche and {Val\'erie} Chavez-Demoulin and Anthony Christopher Davison},
  title         = {Causal modelling of heavy-tailed variables and confounders with application to river flow},
  year          = {2023},
  journal       = {{Extremes}},
  volume        = {26},
  number        = {3},
  pages         = {573--594},
  doi           = {10.1007/s10687-022-00456-4},
}

@misc{engelke2025extremesstructuralcausalmodels,
      title={Extremes of structural causal models}, 
      author={S. Engelke and N. Gnecco and F. Röttger},
      year={2025},
      eprint={2503.06536},
      archivePrefix={arXiv},
      primaryClass={stat.ME},
      url={https://arxiv.org/abs/2503.06536}, 
}

@misc{wiecksosa2026conditionalindependencetestingsingle,
      title={Conditional independence testing with a single realization of a multivariate nonstationary nonlinear time series}, 
      author={M. Wieck-Sosa and M. F. C. Haddad and A. Ramdas},
      year={2026},
      eprint={2504.21647},
      archivePrefix={arXiv},
      primaryClass={stat.ME},
      url={https://arxiv.org/abs/2504.21647}, 
}

@article{freedman1975tail,
  title={On Tail Probabilities for Martingales},
  author={Freedman, David A.},
  journal={The Annals of Probability},
  volume={3},
  number={1},
  pages={100--118},
  year={1975}
}

@article{azuma1967weighted,
  title={Weighted sums of certain dependent random variables},
  author={K. Azuma},
  journal={Tohoku Mathematical Journal},
  volume={19},
  number={3},
  pages={357--367},
  year={1967}
}

@article{hoeffding1963probability,
  title={Probability inequalities for sums of bounded random variables},
  author={W. Hoeffding},
  journal={Journal of the American Statistical Association},
  volume={58},
  number={301},
  pages={13--30},
  year={1963}
}

@book{SRE,
author = {D. Buraczewski and E. Damek and T. Mikosch},
title = {Stochastic Models with Power-Law Tails},
year = {2016},
publisher = {Springer},
doi={10.1007/978-3-319-29679-1}
}

@BOOK{PCalgorithm,
title = {Causation, Prediction, and Search, 2nd Edition},
author = {P. Spirtes and C. Glymour and R. Scheines},
year = {2001},
volume = {1},
edition = {1},
publisher = {The MIT Press},
url = {https://EconPapers.repec.org/RePEc:mtp:titles:0262194406}
}

@article{courgeau2021extreme,
      title={Extreme event propagation using counterfactual theory and vine copulas}, 
      author={V. Courgeau and A. E. D. Veraart},
      year={2021},
journal = {Arxiv preprint 2106.13564}
}

@ARTICLE{EngelkeGraphicalModels,
title = {Graphical models for extremes},
author = {S. Engelke and A.S. Hitz},
year = {2020},
journal = {Journal of the Royal Statistical Society Series B},
volume = {82},
number = {4},
pages = {871-932},
doi = {10.1111/rssb.12355}
}
\end{footnotesize}

\newpage

\appendix

\begin{center}
{\LARGE\bf 
SUPPLEMENT TO\\ 
\vspace{12pt} 
``Granger Causality in Extremes''}
\end{center}

\vspace{10mm}

\setcounter{section}{0}
\setcounter{subsection}{0}
\setcounter{equation}{0}
\setcounter{figure}{0}
\setcounter{table}{0}

\setcounter{lemma}{0}
\setcounter{proposition}{0}
\setcounter{definition}{0}
\setcounter{example}{0}
\setcounter{algocf}{0}

\renewcommand{\thesection}{S.\arabic{section}}
\renewcommand{\thesubsection}{S.\arabic{section}.\arabic{subsection}}
\renewcommand{\theequation}{S.\arabic{equation}}
\renewcommand{\thefigure}{S.\arabic{figure}}
\renewcommand{\thetable}{S.\arabic{table}}
\renewcommand{\thealgocf}{S.\arabic{algocf}}

\renewcommand{\theHsection}{S.\thesection}
\renewcommand{\theHsubsection}{S.\thesubsection}
\renewcommand{\theHequation}{S.\theequation}
\renewcommand{\theHfigure}{S.\thefigure}
\renewcommand{\theHtable}{S.\thetable}
\renewcommand{\theHalgocf}{S.\thealgocf}

\renewcommand{\thelemma}{S.\arabic{lemma}}
\renewcommand{\theproposition}{S.\arabic{proposition}}
\renewcommand{\thedefinition}{S.\arabic{definition}}
\renewcommand{\theexample}{S.\arabic{example}}

\renewcommand{\theHlemma}{S.\thelemma}
\renewcommand{\theHproposition}{S.\theproposition}
\renewcommand{\theHdefinition}{S.\thedefinition}
\renewcommand{\theHexample}{S.\theexample}

This supplement is organized as follows.
\begin{itemize}
    \item \ref{Appendix_A} formalises the extensions sketched in the main text: non-unit causal lags, instantaneous effects, causality in both tails, and time series with bounded support.
    \item \ref{s:HardnessTesting} shows a no-free-lunch theorem: without restricting the model class, no valid-level test exists for Granger non-causality or for non-causality in extremes. This is established as a time-series counterpart of the i.i.d.\ conditional-independence hardness result of \cite{HARDNESS_OF_CONDITIONAL_TESTING}, obtained by embedding the i.i.d.\ problem into the one-step transition of a stationary process.
    \item \ref{appendix_all_simulations} gathers the additional numerical material: the block-bootstrap test for causality in tail, the hyper-parameter analysis, further comparative-performance results, and supplementary figures for the cryptocurrency application.
    \item \ref{Appendix_proofs} collects all the proofs.
\end{itemize}

\section{Extensions}
\label{Appendix_A}

In Section~\ref{section_extensions}, we briefly introduced several extensions of
Definition~\ref{definiton_of_tail_causality}. This section provides the formal details.

\begin{itemize}
    \item \textbf{Supplement~\ref{Appendix_A_lag_greater_than_1} Non-unit causal lags:} we allow an extreme event in $X_t$ to affect
    $Y$ at any time $t+1,\ldots,t+p$.

    \item \textbf{Supplement~\ref{Appendix_A_Instantaneous} Instantaneous effects:} we include $Y_t$ in the response window, which
    requires stronger structural assumptions for a causal interpretation.

    \item \textbf{Supplement~\ref{Appendix_A_both_tails} Both tails:} we extend the definitions from upper-tail events to
    two-sided extremes by considering $|X|$ and $|Y|$.

    \item \textbf{Supplement~\ref{Appendix_A_support} Bounded support:} we replace limits at infinity by limits toward the
    corresponding endpoint of the support.
\end{itemize}

\subsection{Causality in extremes with non-unit causal lags}
\label{Appendix_A_lag_greater_than_1}

The assumption that the effect of an extreme event in \(X_t\) must appear already at
time \(t+1\) can be restrictive. In many time series systems, causal effects may propagate
with a delay. We therefore extend Definition~\ref{definiton_of_tail_causality} by allowing
an extreme event in \(X_t\) to affect \(Y\) at any time \(t+1,\ldots,t+p\), where
\(p\in\mathbb N\) is a fixed max-lag.

This extension is motivated by structural time series of the form
\begin{equation}\label{structural_generation_lagged}
\begin{split}
    X_t
    &=h_{X}\big(
        X_{t-1},\dots,X_{t-q_X},
        Y_{t-1},\dots,Y_{t-q_X},
        \mathbf Z_{t-1},\dots,\mathbf Z_{t-q_X},
        \varepsilon^X_t
    \big),\\
    Y_t
    &=h_{Y}\big(
        X_{t-1},\dots,X_{t-q_Y},
        Y_{t-1},\dots,Y_{t-q_Y},
        \mathbf Z_{t-1},\dots,\mathbf Z_{t-q_Y},
        \varepsilon^Y_t
    \big),
\end{split}
\end{equation}
for all \(t\in\mathbb Z\), where \(q_X,q_Y\in\mathbb N\) are structural lags. The
parameter \(p\) below should be understood as the maximal future lag at which we assess
an extremal effect, and does not need to be identical to \(q_X\) or \(q_Y\).

\begin{definition}[Causality in extremes with max-lag \(p\)]
\label{definition_CTC_with_lag}
Let \(\mathbf W=(\mathbf X,\mathbf Y,\mathbf Z)=((X_t,Y_t,\mathbf Z_t)^\top,t\in\mathbb Z)\)
be a finite-dimensional stochastic process. Let \(F\) be a distribution function satisfying
\(F(x)<1\) for all \(x\in\mathbb R\). Let
\[
    \mathcal C_t
    :=
    \sigma(X_s,Y_s,\mathbf Z_s:s\le t),
    \qquad
    \mathcal C_t^{-X_t}
    :=
    \sigma(X_s:s<t,\;Y_s,\mathbf Z_s:s\le t).
\]
For \(p\in\mathbb N\), define the causal tail coefficient adjusted for \(\mathbf Z\), at time
\(t\) and max-lag \(p\), by
\[
    \Gamma^t_{\mathbf X\to\mathbf Y\mid\mathcal C}(p)
    :=
    \lim_{v\to\infty}
    \mathbb E\!\left[
        \max\{F(Y_{t+1}),\ldots,F(Y_{t+p})\}
        \mid X_t>v,\mathcal C_t^{-X_t}
    \right],
\]
provided that the limit exists almost surely. The corresponding baseline coefficient is
\[
    \Gamma^{t,\baseline}_{\mathbf X\to\mathbf Y\mid\mathcal C}(p)
    :=
    \mathbb E\!\left[
        \max\{F(Y_{t+1}),\ldots,F(Y_{t+p})\}
        \mid \mathcal C_t^{-X_t}
    \right].
\]
If these quantities do not depend on \(t\), for instance under stationarity, we omit the
superscript \(t\).

We say that the upper tail of \(\mathbf X\) causes \(\mathbf Y\) at time \(t\), adjusted for
\(\mathbf Z\), with max-lag \(p\), if
\[
    \Gamma^t_{\mathbf X\to\mathbf Y\mid\mathcal C}(p)
    \not\equiv
    \Gamma^{t,\baseline}_{\mathbf X\to\mathbf Y\mid\mathcal C}(p)
    \quad \text{a.s.}
\]
We write
\(
    \mathbf X\overset{{\rm tail}(p)}{\longrightarrow}\mathbf Y\mid\mathbf Z
\)
if this holds for some \(t\in\mathbb Z\).

We say that an upper extreme in \(\mathbf X\) causes an extreme in \(\mathbf Y\) at time
\(t\), adjusted for \(\mathbf Z\), with max-lag \(p\), if
\[
    \Gamma^t_{\mathbf X\to\mathbf Y\mid\mathcal C}(p)=1
    \quad \text{a.s.}
\]
We write
\(
    \mathbf X\overset{{\rm ext}(p)}{\longrightarrow}\mathbf Y\mid\mathbf Z
\)
if this holds for some \(t\in\mathbb Z\).

As in Definition~\ref{definiton_of_tail_causality}, we omit ``\(\mid\mathbf Z\)'' from the
notation when the conditioning set \(\mathcal C_t^{-X_t}\) is causally sufficient.
\end{definition}

For \(p=1\), Definition~\ref{definition_CTC_with_lag} reduces to
Definition~\ref{definiton_of_tail_causality} under 1-Markov property. The results in the main text can be adapted to this max-lag setting by replacing \(F(Y_{t+1})\) with \(     \max\{F(Y_{t+1}),\ldots,F(Y_{t+p})\}\) throughout.

\subsubsection{Results from Section 2 adjusted to extremal lag}
\label{appendix_lag_section2}
Before we delve into the connections between the concept of causality in extremes adjusted to extremal lag $p$ and classical causality, we introduce a new type of causal notion, called Sims causality \citep{sims,GrangerSims2, GrangerSims}. In contrast to Granger causality, it takes in account not only direct but also indirect causal effects.

\begin{definition}[Sims causality]\label{Definition_Sims}
Following the notation from Definition~\ref{DEF1}, we say that the process \(\mathbf X\)
Sims-causes the process \(\mathbf Y\), adjusted for \(\mathbf Z\), if
\[
    \mathbf Y_{\future(t)}
    :=
    \{Y_{t+s}:s\ge 1\}
    \notindep X_t
    \mid \mathcal C_t^{-X_t}
    \quad \text{for some } t\in\mathbb Z.
\]
We write \(\mathbf X\toverset{\mathrm{Sims}}{\to}\mathbf Y\mid\mathbf Z\). If the conditioning
set \(\mathcal C_t^{-X_t}\) is causally sufficient, we simply write
\(\mathbf X\toverset{\mathrm{Sims}}{\to}\mathbf Y\).
\end{definition}

Granger causality and Sims causality are related, but not equivalent \citep{GrangerSims}. Notable difference is that if $\textbf{X}$ Granger-causes $\textbf{Y}$ only via a mediator ( $\textbf{X}\toverset{G}{\to}\textbf{Z}\toverset{G}{\to}\textbf{Y}$, but $\textbf{X}\toverset{G}{\not\to}\textbf{Y}$), Sims causality typically captures this relation ($\textbf{X}\toverset{Sims}{\to}\textbf{Y}$). Note that for $1$-Markov time series,
\begin{equation*}
    \begin{split}
Y_{t+1}\notindep X_t \mid \mathcal{C}^{-\textbf{X}}_t \iff Y_{t+1} \notindep X_t \mid   \mathcal{C}^{-X_t}_t \implies 
\textbf{Y}_{\future(t)} \notindep X_t \mid    \mathcal{C}^{-X_t}_t,
    \end{split}
\end{equation*}
 and hence, Granger causality implies Sims causality \citep{GrangerSims}. 

 An analogous result to Proposition~\ref{Proposition_Granger_equivalent_tail} can be stated. The proof is presented in Supplement \ref{proof_of_proposition_with_lag}. 

\begin{proposition}\label{proposition_with_lag_1}
Consider the data-generating process (\ref{structural_generation_lagged}). Then for every $p\in\mathbb{N}$, 
\begin{equation*}
    \begin{split}
   \textbf{X}\overset{{\rm ext}(p)}{\longrightarrow}\textbf{Y}    &\implies \textbf{X}\overset{{\rm tail}(p)}{\longrightarrow}\textbf{Y}    
    \implies \textbf{X}\toverset{Sims}{\to}\textbf{Y}.  
    \end{split}
\end{equation*}
\end{proposition}

\begin{proposition}[$   \textbf{X}\overset{{\rm ext}(p)}{\longrightarrow}\textbf{Y}  \impliedby \textbf{X}\toverset{Sims}{\to}\textbf{Y}  $]\label{proposition_with_lag_2}
Consider the structural time series
\[
\begin{split}
   \mathbf Z_t
   &= h_Z(X_{t-1},\ldots,X_{t-q_Z},
          Y_{t-1},\ldots,Y_{t-q_Z},
          \mathbf Z_{t-1},\ldots,\mathbf Z_{t-q_Z},
          \varepsilon^Z_t),\\
   X_t
   &= h_X(X_{t-1},\ldots,X_{t-q_X},
          Y_{t-1},\ldots,Y_{t-q_X},
          \mathbf Z_{t-1},\ldots,\mathbf Z_{t-q_X},
          \varepsilon^X_t),\\
   Y_t
   &= h_Y(X_{t-1},\ldots,X_{t-q_Y},
          Y_{t-1},\ldots,Y_{t-q_Y},
          \mathbf Z_{t-1},\ldots,\mathbf Z_{t-q_Y},
          \varepsilon^Y_t).
\end{split}
\]
Assume that \(h_X,h_Y,h_Z\), are upper-tail preserving (Definition~\ref{upper_tail_preserving_definition}). Assume further that,
for every \(t\) and every \(m\ge 1\),
\[
   (\varepsilon^X_{t+1},\ldots,\varepsilon^X_{t+m},
    \varepsilon^Y_{t+1},\ldots,\varepsilon^Y_{t+m},
    \varepsilon^Z_{t+1},\ldots,\varepsilon^Z_{t+m})
   \indep X_t \mid \mathcal C_t^{-X_t}.
\]
Then, if $
   \ell_t:=\min\{s\ge 1:
   Y_{t+s}\not\indep X_t\mid \mathcal C_t^{-X_t}\}<\infty$, then
\[
   \Gamma_{\mathbf X\to \mathbf Y\mid\mathcal C}^t(p)=1
   \qquad\text{for every }p\ge \ell_t.
\]
\end{proposition}

\begin{definition}[Upper-tail preserving functions]\label{upper_tail_preserving_definition}
Let \(h:\mathbb R^m\to\mathbb R\) be continuous. We say that \(h\) is upper-tail preserving if, for every non-empty
\(I\subset\{1,\ldots,m\}\), either \(h\) is constant in the coordinates
\((x_i)_{i\in I}\), or, for every compact set
\(K\subset \mathbb R^{m-|I|}\),
\[
    \lim_{r\to\infty}
    \inf_{\substack{x_i\ge r,\ i\in I\\ x_{-I}\in K}}
    h(x_1,\ldots,x_m)=\infty .
\]
\end{definition}

\subsubsection{Results from Section 3 adjusted to extremal lag}
\label{supplement_section3_lagged}
We assert that a lagged version of Theorem~\ref{Theorem1}, specifically:
\begin{equation}\label{OhYeah_lagged}
\Gamma_{\textbf{X}\to\textbf{Y}\mid\mathcal{C}}(p) = 1 \iff  \Gamma_{\textbf{X}\to\textbf{Y}\mid \emptyset}(p) = 1,
\end{equation}
can be established. 

\begin{assumptions*}
Consider the SRE with a lag of $p$ \citep[Chapter 5]{SRE}:
\begin{equation*}\label{generalSRE_lagged}
    \textbf{W}_t = \sum_{i=1}^p\textbf{A}_{t}^{(i)}\textbf{W}_{t-i} +  \textbf{B}_t,\qquad t\in\mathbb{Z},
\end{equation*}
where $(\textbf{A}_t^{(1)},\dots, \textbf{A}_t^{(p)}, \textbf{B}_t)$ is an iid sequence, $\textbf{A}_t^{(1)}$ are $d\times d$ matrices and $\textbf{B}_t$ are $d$ dimensional vectors. We will work with the following assumptions.
\begin{itemize}
    \item[(A1.2)]  $\mathbb{E}\log||\textbf{A}_t^{(i)}|| < 0$ and $\mathbb{E}\log_+|\textbf{B}_t| < \infty$ (ensuring stationarity of our time series),
    \item[(A2.2)]  $\varepsilon_t^{z}, \varepsilon_t^{x}, \varepsilon_t^{y}$ are independent for all $t\in\mathbb{Z}$ (i.e. no instantaneous causality).
    \item[(A3.2)] $\textbf{B}_t\indep \textbf{A}_t^{(i)}$ \footnote{We follow the convention that a deterministic variable is independent with any other variable} for all $t\in\mathbb{Z}$.
    \item[(A5.2)] If $P(A_{j,t}^{(k),i}=0)\neq1$ then $A_{j,t}^{(k),i}\toverset{a.s.}{>}0$  for all  $t\in\mathbb{Z}$ and $j=1,2,3$ and  $i=z,x,y$ and $k = 1, \dots, p$ (positivity assumption)
\end{itemize}
\end{assumptions*} 

We conjecture that Equation (\ref{OhYeah_lagged}) holds true under the aforementioned assumptions and with the condition of appropriate regular variation in our time series. However, proving this claim falls beyond the scope of this work. The tails of a lagged SRE remain relatively understudied in the literature.

\subsubsection{Results from Section 4 adjusted to extremal lag}
\label{supplement_section4_lagged}
In the following, we present an estimator of the coefficient $\Gamma_{\textbf{X}\to\textbf{Y}\mid\mathcal{C}}(p)$ based on a random sample, $p\in\mathbb{N}$. We denote by  $\textbf{Z} = (\textbf{Z}_t, t\in\mathbb{Z})$ a vector of other relevant time series (possible confounders).  We assume that we observe $(x_1, y_1, \textbf{z}_1)^\top, \dots, (x_T, y_T, \textbf{z}_T)^\top$, with the maximum observed time $T\in\mathbb{N}$. 

\begin{definition}
  We propose an estimator
  \begin{equation*}
\hat{\Gamma}_{\textbf{X}\to\textbf{Y}\mid\mathcal{C}}(p_x, p_y):= \frac{1}{|S_{p_x}|}\sum_{t\in S_{p_x}} \max\{F(y_{t+1}),\dots, F(y_{t+p_y})\},
  \end{equation*}
 where $S_{p_x}\subseteq\{1, \dots, T\}$ is a set described below. If $p_x=p_y=:p$, we simply write $\hat{\Gamma}_{\textbf{X}\to\textbf{Y}\mid\mathcal{C}}(p)$.
\end{definition}

 \begin{definition}
We propose the following definition:  \begin{equation*}
      S_{p_x} := \{i\in\{1, \dots, T\}: X_i\geq \tau_X, 
      \begin{pmatrix}
    Y_i \\
    \textbf{Z}_i 
\end{pmatrix}   \leq  \boldsymbol{\tau},
 \begin{pmatrix}
    Y_{i-1} \\
    \textbf{Z}_{i-1} 
\end{pmatrix}   \leq  \boldsymbol{\tau}, 
\dots, 
 \begin{pmatrix}
    Y_{i-p_x+1} \\
    \textbf{Z}_{i-p_x+1} 
\end{pmatrix}   \leq  \boldsymbol{\tau} \},
  \end{equation*}
where $\boldsymbol{\tau}$ is a hyperparameter and  $\tau_X$ is the $k$-th largest value in the set $\{X_t: t\in\tilde{S}_{p_x} \}$, where
$
      \tilde{S}_{p_x} := \{i\in\{1, \dots, T\}: 
      \begin{pmatrix}
    Y_i \\
    \textbf{Z}_i 
\end{pmatrix}   \leq  \boldsymbol{\tau},
 \begin{pmatrix}
    Y_{i-1} \\
    \textbf{Z}_{i-1} 
\end{pmatrix}   \leq  \boldsymbol{\tau}, 
\dots, 
 \begin{pmatrix}
    Y_{i-p_x+1} \\
    \textbf{Z}_{i-p_x+1} 
\end{pmatrix}   \leq  \boldsymbol{\tau} \},
$ and where $k$ satisfies (\ref{k_deleno_n}). 
\end{definition}
In other words, we condition on $X_i$ being extreme, while we require all variables in the past $p_x$ steps to be not extreme. 

Algorithms~\ref{Algorithm1} and \ref{Algorithm2}, along with the testing procedure outlined in Section~\ref{Section_testing}, can be straightforwardly adapted to incorporate the notion of the extremal lag.

\subsubsection{Alternative approach for defining causality in extremes with non-unit lag}

An alternative approach to extending $\Gamma_{\textbf{X}\to\textbf{Y}\mid\mathcal{C}}$, distinct from Definition \ref{definition_CTC_with_lag}, involves conditioning on the lagged values of $\textbf{X}$. 

\begin{definition}[Alternative definition of the causality in extremes---lagged version]\label{definition_CTC_with_lag_alternative}
    $$  \tilde{\Gamma}_{\textbf{X}\to\textbf{Y}\mid\mathcal{C}}(p) :=  \lim_{v\to\infty} \mathbb{E}[F(Y_{t+1})\mid X_{t-k}>v, \textbf{X}_{[t, t-q_y]\setminus\{p\}}, \mathcal{C}^{-\textbf{X}}_t],$$
      $$  \tilde{\Gamma}^{\baseline}_{\textbf{X}\to\textbf{Y}\mid\mathcal{C}}(p) :=  \lim_{v\to\infty} \mathbb{E}[F(Y_{t+1})\mid  \textbf{X}_{[t, t-q_y]\setminus\{p\}}, \mathcal{C}^{-\textbf{X}}_t],$$
where $0\leq p\leq q_y$ and where we used the notation $[t, t-q_y]\setminus\{p\} := (t-q_y, t-q_y+1, \dots, p-1, p+1, \dots, t)$. 

We define the causality in extremes and in tail up to lag $p$ analogously to the Definition~\ref{definition_CTC_with_lag}. 
\end{definition}
This approach offers a more intricate characterization of the causal structure, revealing which lagged value $X_{t-k}$ causes $Y_{t+1}$. However, that this option is not be well-suited for inference as it typically assumes that an extreme value in $X_{t-k}$ is observed while $X_{t-k+1},X_{t-k-1} $ are not extreme. This may be very impractical in real scenarios.

\subsection{Instantaneous Causality}\label{Appendix_A_Instantaneous}

We now allow the response window to include the contemporaneous value \(Y_t\). This is useful when the sampling frequency is too coarse to separate very short-lag effects from instantaneous effects, or when the model contains a meaningful contemporaneous causal ordering.

\begin{definition}[Instantaneous causality in extremes]
\label{definition_instantaneous_causality}
Let \(\mathbf W=(\mathbf X,\mathbf Y,\mathbf Z)=((X_t,Y_t,\mathbf Z_t)^\top,t\in\mathbb Z)\) be a finite-dimensional stochastic process. Let \(F\) be a distribution function satisfying \(F(x)<1\) for all \(x\in\mathbb R\), and let \(p\in\mathbb N\cup\{0\}\). Define
\[
\Gamma^{t}_{\mathbf X\to \mathbf Y\mid\mathcal C}([0,p])
:=
\lim_{v\to\infty}
\mathbb E\!\left[
\max_{0\le j\le p} F(Y_{t+j})
\,\middle|\,
X_t>v,\mathcal C_t^{-\{X_t,Y_t\}}
\right],
\]
provided that the limit exists almost surely, where
\(\mathcal C_t^{-\{X_t,Y_t\}}:=\sigma(X_s,Y_s:s<t,\;\mathbf Z_s:s\le t)\). The corresponding baseline coefficient is
\(\Gamma^{t,\baseline}_{\mathbf X\to \mathbf Y\mid\mathcal C}([0,p])
:=\mathbb E[\max_{0\le j\le p}F(Y_{t+j})\mid \mathcal C_t^{-\{X_t,Y_t\}}]\).

We say that \(\mathbf X\) tail-causes \(\mathbf Y\) on the window \([0,p]\), adjusted for \(\mathbf Z\), if
\(\Gamma^{t}_{\mathbf X\to \mathbf Y\mid\mathcal C}([0,p])
\not\equiv
\Gamma^{t,\baseline}_{\mathbf X\to \mathbf Y\mid\mathcal C}([0,p])\) a.s. for some \(t\). We write
\(\mathbf X\overset{{\rm tail}([0,p])}{\longrightarrow}\mathbf Y\mid\mathbf Z\).
We say that an extreme in \(\mathbf X\) causes an extreme in \(\mathbf Y\) on \([0,p]\), adjusted for \(\mathbf Z\), if
\(\Gamma^{t}_{\mathbf X\to \mathbf Y\mid\mathcal C}([0,p])=1\) a.s. for some \(t\), and write
\(\mathbf X\overset{{\rm ext}([0,p])}{\longrightarrow}\mathbf Y\mid\mathbf Z\).
As before, we omit ``\(\mid\mathbf Z\)'' when the conditioning set is causally sufficient.
\end{definition}

The case \(p=0\) corresponds to purely instantaneous extremal dependence. For \(p\ge1\), the coefficient detects an effect somewhere in the window \(Y_t,\ldots,Y_{t+p}\); it does not identify the exact lag without further restrictions.

\subsubsection{Results from Section 2 adjusted to instantaneous causality}

Definition~\ref{definition_instantaneous_causality} is \textbf{not} equivalent to Granger causality without strong additional assumption. The following example illustrates this point: Suppose that, at a one-second resolution, extremes propagate through one-step lagged effects \(X_t\to Y_{t+1}\) and \(Y_{t+1}\to X_{t+2}\). If only every other second is observed, or if observations are recorded as two-second block maxima, these one-step effects are collapsed into the same observed time point. The resulting instantaneous coefficients may therefore be large in both directions, even though the underlying effects are lagged rather than contemporaneous.    

\begin{proposition}[Identifiability with instantaneous effects]
\label{theorem_instantaneous}
Fix \(t\in\mathbb Z\). Let \(
\tilde{\mathcal C}_t:=\mathcal C_t^{-\{X_t,Y_t\}}
\) denote the admissible information available at time \(t\), excluding \(X_t\) and \(Y_t\).
Suppose that, conditionally on \(\tilde{\mathcal C}_t\), the contemporaneous structural equations are
\[
X_t=\mu_X(\tilde{\mathcal C}_t)+\varepsilon_t^X,
\qquad
Y_t=\mu_Y(\tilde{\mathcal C}_t)+\beta X_t+\varepsilon_t^Y,
\]
where \(\mu_X\) and \(\mu_Y\) are \(\tilde{\mathcal C}_t\)-measurable and finite almost surely, and $\beta>0$. 

Assume that \(\varepsilon_t^X\), \(\varepsilon_t^Y\), and \(\tilde{\mathcal C}_t\) are mutually independent,  \(\varepsilon_t^X, \varepsilon_t^Y\in RV(\alpha)\) are compatible and non-negative and define the $p=0$ coefficient 
\[
\Gamma^{t,\mathrm{inst}}_{\mathbf X\to\mathbf Y\mid\tilde{\mathcal C}}
:=
\lim_{v\to\infty}
\mathbb E\!\left[
F_Y(Y_t)\mid X_t>v,\tilde{\mathcal C}_t
\right],
\]
assuming that the limit exists a.s. Then,
\[
\Gamma^{t,\mathrm{inst}}_{\mathbf X\to\mathbf Y\mid\tilde{\mathcal C}}=1, \quad  \text{ and } \quad \Gamma^{t,\mathrm{inst}}_{\mathbf Y\to\mathbf X\mid\tilde{\mathcal C}}<1,
\quad
\quad\text{a.s.}
\]
\end{proposition}
Proof is located in Supplement~\ref{section_proof_instantaneous}.

\subsubsection{Results from Section 4 adjusted to instantaneous causality}

The estimation procedure from Section~\ref{Section_estimation} is modified by replacing the score \(F(y_{t+1})\) with \(M_t^{[0,p]}:=\max_{0\le j\le p}F(y_{t+j})\), using only indices \(t\le n-p\). Thus,
\[
\widehat\Gamma_{\mathbf X\to\mathbf Y\mid\mathbf Z}([0,p])
:=
\frac{1}{|S^{[0,p]}|}
\sum_{t\in S^{[0,p]}} M_t^{[0,p]} .
\]
The set \(S^{[0,p]}\) is chosen as in Section~\ref{Section_estimation}, with one important change: since \(Y_t\) is part of the response window, it must not be used as a conditioning variable. For example, the analogue of \(S_1\) conditions on \(x_t\) being extreme and on the variables in \(\mathcal C_t^{-\{X_t,Y_t\}}\), such as \(\mathbf z_t\), being non-extreme, but it does not impose a restriction on \(y_t\). The baseline estimator is obtained by using the same conditioning restrictions, but without the event \(x_t\ge \tau_k^X\).
Algorithm~\ref{Algorithm1} and the bootstrap test from Section~\ref{Section_testing} can then be applied without further changes, after replacing \(F(y_{t+1})\) by \(M_t^{[0,p]}\). 

\subsection{Causality in both tails}\label{Appendix_A_both_tails}

We discuss the modification of our framework for causality-in-both-tails. Recall that (both) tails of $\textbf{X}$ cause $\textbf{Y}$ if 
\begin{equation*}\label{CTC_both_tails}
\begin{split} &\Gamma_{|\textbf{X}|\to|\textbf{Y}|\mid\mathcal{C}}:=  \lim_{v\to\infty} \mathbb{E}[ F^{\pm}(|Y_{t+1}|)\,\,\,\bigr\vert\,\,\, |X_t|>v,\mathcal{C}^{-\textbf{X}}_t]\\ 
&\neq\Gamma_{|\textbf{X}|\to|\textbf{Y}|\mid\mathcal{C}}^{\baseline}:=\mathbb{E}[F^{\pm}(|Y_{t+1}|)\,\,\,\bigr\vert\,\,\,  \mathcal{C}^{-\textbf{X}}_t],    
\end{split}
\end{equation*}
where $F^{\pm}$ is a distribution function satisfying $F^{\pm}(x)<1$ for all $x\in\mathbb{R}$.

\subsubsection{Results from Section 2 adjusted to both tails}
For completeness, we reformulate the results from the main part of the manuscript for causality in both tails. Proposition~\ref{Proposition_both_tails} shows the modification of the results presented in Section~\ref{Section_Reformulating}.  We discuss the modification of Theorem~\ref{Theorem1} in Section~\ref{appendix_both_tails_section3}. Finally, we modify the inference procedure to be able to handle both-tails in Section~\ref{appendix_both_tails_section4}.

\begin{proposition}\label{Proposition_both_tails}The following statements are true:
\begin{itemize}
    \item If 
\begin{equation}\label{eq3522}
\Gamma_{|\textbf{X}|\to|\textbf{Y}|\mid\mathcal{C}} = 1,
\end{equation}
then $\textbf{X}\toverset{tail^{\pm}}{\longrightarrow}\textbf{Y}$. Under Assumption \ref{AssumptionA2},  $\textbf{X}\toverset{tail^{\pm}}{\longrightarrow}\textbf{Y}$ implies (\ref{eq3522}).
    \item  $\textbf{X}\toverset{tail^{\pm}}{\longrightarrow}\textbf{Y}$ implies $\textbf{X}\toverset{G}{\longrightarrow}\textbf{Y}$. Under Assumption \ref{AssumptionA2},  $\textbf{X}\toverset{G}{\longrightarrow}\textbf{Y}$ implies $\textbf{X}\toverset{tail^{\pm}}{\longrightarrow}\textbf{Y}$. 
    \item  Under Assumption \ref{AssumptionA2}, the definition of  $\textbf{X}\toverset{tail^{\pm}}{\longrightarrow}\textbf{Y}$ is invariant with the choice of $F^{\pm}$. 
\end{itemize}
\end{proposition}
The proof is presented in Supplement \ref{proof_of_proposition_both_tails}.

\subsubsection{Results from Section 3 adjusted to both tails}
\label{appendix_both_tails_section3}
Theorem~\ref{Theorem1} can be restated to account for causality in both tails. We require the following two-sided analogue of (B3):
\[
\tag{B3$^\pm$}
\lim_{v\to\infty}
P\bigl(|Z_t|\leq a|X_t| \mid |X_t|>v,Y_{past(t)}\bigr)=1
\qquad\text{for every }a>0.
\]

\begin{lemma}\label{Theorem1_both_tails_lemma}
Consider time series following the SRE model defined in~(\ref{sre2}) satisfying
(B1), (B2), and (B5). 

\begin{itemize}
    \item Under $ (B3^{\pm})$ holds \(
\Gamma_{|\mathbf X|\to|\mathbf Y|\mid\mathcal C}=1
\implies
\Gamma_{|\mathbf X|\to|\mathbf Y|\mid\emptyset}=1.
\)

 \item Assume a two-sided analogue of the Grey-type tail condition: There exist
\(\alpha_x>0\) and \(\nu>0\) such that \(|B_t^x|\in RV(\alpha_x)\) and
\(
E|A^x_{j,t}|^{\alpha_x+\nu}<\infty,
 j=1,2,3,
\)
and assume the conditional tail-dominance condition
\(
\limsup_{u\to\infty}
\frac{P(|X_t|>u \mid Y_{past(t)})}{P(|B_t^x|>u)}
<\infty.
\) Then \(
\Gamma_{|\mathbf X|\to|\mathbf Y|\mid\emptyset}=1
\implies
\Gamma_{|\mathbf X|\to|\mathbf Y|\mid\mathcal C}=1.
\)
\end{itemize}
\end{lemma}
The proof is presented in Supplement \ref{proof_of_Theorem1_both_tails_lemma}. 

\subsubsection{Results from Section 4 adjusted to both tails}
\label{appendix_both_tails_section4}

In the following, we present an estimator of the coefficient \(\Gamma_{|\textbf{X}|\to|\textbf{Y}|\mid \textbf{Z}}\) based on a random sample. Specifically, one can directly work with the estimator (\ref{Gamma_hat_original}), substituting \(|X_t|\) and \(|Y_t|\) for \(X_t\) and \(Y_t\), respectively. However in various real-world scenarios, asymmetric tail importance holds significant relevance. This is particularly important in contexts involving investment behavior or policy decision-making, where the concept of 'loss aversion' plays an important role. Loss aversion denotes a cognitive bias wherein individuals assign higher importance to evading losses as opposed to attaining equivalent gains. In simpler terms, the emotional impact of losing $100$ dollars is psychologically more pronounced than the satisfaction derived from gaining the same amount. Therefore, we generalize the coefficient by employing asymmetric thresholds. This results in capturing the asymmetric emphasis on positive and negative values.

We denote by  $\textbf{Z} = (\textbf{Z}_t, t\in\mathbb{Z})$ a vector of other relevant time series (possible confounders).  We assume that we observe $(x_1, y_1, \textbf{z}_1)^\top, \dots, (x_T, y_T, \textbf{z}_T)^\top$, with the maximum observed time $T\in\mathbb{N}$. We propose the following estimator:
 \begin{equation*}
\hat{\Gamma}_{|\textbf{X}|\to|\textbf{Y}|\mid \textbf{Z} }:=  \frac{1}{|S^{\pm}|}\sum_{t\in S^{\pm}}  
F^{\pm}(|Y_{t+1}|),
  \end{equation*}
 where
   \begin{equation*}
      S^{\pm} := \{t\in\{1, \dots, T\}: X_i\not\in [\tau_{X}^{-}, \tau_{X}^{+}], 
\begin{pmatrix}
    Y_i \\
    \textbf{Z}_i 
\end{pmatrix}  \in \begin{pmatrix}
   [\tau_{Y}^{-}, \tau_{Y}^{+}] \\
   [\tau_{Z}^{-}, \tau_{Z}^{+}]
\end{pmatrix}
\},
  \end{equation*}
 where $\boldsymbol{\tau} = ( \begin{pmatrix}
   \tau_{X}^{-} \\
   \tau_{X}^{+}
\end{pmatrix},
\begin{pmatrix}
   \tau_{Y}^{-} \\
   \tau_{Y}^{+}
\end{pmatrix},
\begin{pmatrix}
   \tau_{Z}^{-} \\
   \tau_{Z}^{+}
\end{pmatrix})$ are some hyperparameters.

\begin{example*}[Symmetric thresholds]
    For the symmetric choice $ \tau_{X}^{-} =- \tau_{X}^{+}, \tau_{Y}^{-} =- \tau_{Y}^{+}, \tau_{Z}^{-} =- \tau_{Z}^{+}$, we obtain 
   \begin{equation*}
      S^{\pm} = \{t\in\{1, \dots, T\}: |X_t|> \tau_{X}^{+},  |Y_t|\leq  \tau_{Y}^{+},|Z_t|\leq  \tau_{Z}^{+} \}. 
  \end{equation*}
In the case of symmetric thresholds, this estimator matches the one from Definition~\ref{Gamma_hat_equation}, where \(|X_t|\) and \(|Y_t|\) are used instead of \(X_t\) and \(Y_t\), respectively.
\end{example*}

\subsection{Time series with finite upper endpoints}
\label{Appendix_A_support}

The assumption that \(X_t\) and \(Y_t\) are supported on a neighborhood of
\(+\infty\) is mainly a normalization. If \(X\) has a finite upper endpoint,
upper-tail events should instead be interpreted as events in which \(X_t\)
approaches this endpoint. Let
\(
    r_X:=\operatorname{ess\,sup} X_0,
    r_Y:=\operatorname{ess\,sup} Y_0 .
\)
Throughout this subsection we assume, for simplicity, that
\(X_t<r_X\) and \(Y_t<r_Y\) almost surely. Endpoint atoms can be handled
separately. 

Let \(F_Y:\mathbb R\to[0,1]\) be a distribution function satisfying
\[
    F_Y(y)<1 \quad\text{for all } y<r_Y,
    \qquad
    \lim_{y\uparrow r_Y}F_Y(y)=1.
\]
We define the bounded-version of the causal tail coefficient by
\[
\Gamma^{\rm bd}_{\mathbf X\to\mathbf Y\mid\mathcal C}
:=
\lim_{u\uparrow r_X}
\mathbb E\!\left[
    F_Y(Y_{t+1})
    \,\middle|\,
    X_t>u,\mathcal C_t^{-\mathbf X}
\right],
\]
provided that the limit exists almost surely, and define the corresponding baseline
coefficient by
\[
\Gamma^{{\rm bd},baseline}_{\mathbf X\to\mathbf Y\mid\mathcal C}
:=
\mathbb E\!\left[
    F_Y(Y_{t+1})
    \,\middle|\,
    \mathcal C_t^{-\mathbf X}
\right].
\]
We say that \(\mathbf X\) causes \(\mathbf Y\) in the upper tail, in the
bounded-endpoint sense, if
\[
    \Gamma^{\rm bd}_{\mathbf X\to\mathbf Y\mid\mathcal C}
    \not\equiv
    \Gamma^{{\rm bd},baseline}_{\mathbf X\to\mathbf Y\mid\mathcal C}
    \quad\text{a.s.}
\]
We say that an upper extreme in \(\mathbf X\) causes an upper extreme in
\(\mathbf Y\), in the bounded-endpoint sense, if
\[
    \Gamma^{\rm bd}_{\mathbf X\to\mathbf Y\mid\mathcal C}=1
    \quad\text{a.s.}
\]

This formulation is equivalent to applying an increasing endpoint transformation.
Indeed, let \(T_X\) and \(T_Y\) be increasing one-to-one Borel maps such that
\[
    T_X(x)\to\infty \quad\text{as } x\uparrow r_X,
    \qquad
    T_Y(y)\to\infty \quad\text{as } y\uparrow r_Y.
\]
Set \(\widetilde X_t=T_X(X_t)\), \(\widetilde Y_t=T_Y(Y_t)\), and
\[
    \widetilde F_Y(s):=F_Y(T_Y^{-1}(s)).
\]
Then \(\widetilde F_Y(s)<1\) for all finite \(s\), and
\[
\Gamma^{\rm bd}_{\mathbf X\to\mathbf Y\mid\mathcal C}
=
\lim_{w\to\infty}
\mathbb E\!\left[
    \widetilde F_Y(\widetilde Y_{t+1})
    \,\middle|\,
    \widetilde X_t>w,\mathcal C_t^{-\mathbf X}
\right],
\]
after the reparametrization \(w=T_X(u)\). Since \(T_X\) and \(T_Y\) are
one-to-one, they preserve the relevant sigma-fields and conditional-independence
relations.

\begin{proposition}[Bounded analogue of Propositions~\ref{Proposition_CTC_iff_1} and
\ref{Proposition_Granger_equivalent_tail} ]
Consider the setup of Definition~\ref{Definition_str} and the bounded-endpoint
coefficient defined above. Then, with the notions of tail and extreme causality
understood in the bounded-endpoint sense,
\[
    X \stackrel{\mathrm{ext}}{\longrightarrow} Y
    \quad\Longrightarrow\quad
    X \stackrel{\mathrm{tail}}{\longrightarrow} Y
    \quad\Longrightarrow\quad
    X \stackrel{G}{\longrightarrow} Y .
\]
If, in addition, we assume that for every admissible \(y,z,e\), either
\[
(A1_{\rm bd})\qquad\qquad    \lim_{x\uparrow r_X} h_Y(x,y,z,e)=r_Y,
\]
or \(h_Y(\cdot,y,z,e)\) is constant in \(x\) holds, then all three notions are equivalent.
\end{proposition}

\begin{proof}
Let \(\widetilde X_t=T_X(X_t)\) and \(\widetilde Y_t=T_Y(Y_t)\), where
\(T_X,T_Y\) are increasing one-to-one endpoint transformations as above.
Because these transformations are one-to-one, they preserve the relevant
sigma-fields; hence Granger causality is unchanged.

The transformed structural equation for \(Y\) is
\[
    \widetilde Y_{t+1}
    =
    \widetilde h_Y(\widetilde X_t,\widetilde Y_t,Z_t,\varepsilon^Y_{t+1}),
\]
where
\[
    \widetilde h_Y(\widetilde x,\widetilde y,z,e)
    :=
    T_Y\!\left(
        h_Y(T_X^{-1}(\widetilde x),T_Y^{-1}(\widetilde y),z,e)
    \right).
\]
Under \((A1_{\rm bd})\), this transformed structural equation satisfies
Assumption~A1. Moreover, the bounded-endpoint coefficient above is exactly the
coefficient from Definition~\ref{CTC_definition} applied to
\((\widetilde X,\widetilde Y)\) with score
\(\widetilde F_Y=F_Y\circ T_Y^{-1}\). The conclusion follows by applying
Propositions~\ref{Proposition_CTC_iff_1} and
\ref{Proposition_Granger_equivalent_tail} to the transformed process and then
transforming back.
\end{proof}

\newpage

\section{Hardness of testing for causality in time series}\label{s:HardnessTesting}
In this section, we argue that testing for causality in extremes is fundamentally more difficult than testing for causality in tail, as discussed in Section~\ref{Section_testing}. In particular, we show that testing for causality in extremes is not possible without imposing stronger assumptions; specifically, a more restrictive statistical model than the one in Definition~\ref{Definition_str}.

First, we show that it is impossible to find a test for the null hypothesis
$H_0: \textbf{X}\toverset{G}{\not\to}\textbf{Y}$, 
with a valid level, without restricting the structure of the time series. Then, we show that the same holds for $H_0: \textbf{X}\toverset{ext}{\not\to}\textbf{Y}$.  This is a time series counterpart of an i.i.d. concept presented in \cite{HARDNESS_OF_CONDITIONAL_TESTING}. Although an i.i.d. sequence is a special case of
a time series, the results from \cite{HARDNESS_OF_CONDITIONAL_TESTING} are not directly applicable since an i.i.d. sequence \((X_t,Y_t,Z_t)_{t\in\mathbb Z}\) does not by itself create a useful Granger-causality alternative ($Y_{t+1}$ is automatically independent on the past). Instead, we embed an i.i.d. conditional-independence problem into the one-step transition of a stationary time series. \cite{wiecksosa2026conditionalindependencetestingsingle} consider more detailed discussion about conditional independence testing in time series. 

Denote  \(\Xi_0\) the class of stationary time series generated according to Definition~\ref{Definition_str}, with all variables absolutely continuous with respect to Lebesgue measure, and define  
\[
    \mathcal P_0^G
    :=
    \left\{
        P\in\Xi_0:
        \mathbf X\toverset{G}{\not\to}\mathbf Y\mid\mathbf Z
    \right\}. 
\]
Consider the following subclass \(\mathcal A_0\) of \(\Xi_0\) containing  simple one-lag non-auto-correlated processes. Its elements are generated, for some mutually independent i.i.d. continuous innovation sequences \((\varepsilon^X_t)_{t\in\mathbb Z}\), \((\varepsilon^Y_t)_{t\in\mathbb Z}\), \((\varepsilon^Z_t)_{t\in\mathbb Z}\) and some measurable function \(f:\mathbb R^3\to\mathbb R\), as
\[
Z_t=\varepsilon^Z_t,\qquad    X_t=\varepsilon^X_t,\qquad  Y_t=f(X_{t-1}, Z_{t-1}, \varepsilon^Y_t),
\]
and satisfy \(\mathbf X\toverset{G}{\to}\mathbf Y\mid\mathbf Z\). 

Let \(
    \psi_n:\mathbb R^{3n}\to\{0,1\}
\) be a test, where \(\psi_n=1\) denotes rejection of the null hypothesis of
Granger non-causality based on a sample $((x_1, y_1, z_1), \dots, (x_n, y_n, z_n))$. Throughout, \(P(\psi_n=1)\) denotes the rejection probability under the \(n\)-sample marginal of \(P\).

\begin{proposition}[No-free-lunch reduction]
\label{no_free_lunch}
Let \(n\in\mathbb N\), let \(\alpha\in(0,1)\), and let
\(\psi_n:\mathbb R^{3n}\to\{0,1\}\) be any test of
\[
    H_0^G:\mathbf X\toverset{G}{\not\to}\mathbf Y\mid\mathbf Z .
\]
If \(\psi_n\) has level \(\alpha\) uniformly over the unrestricted Granger non-causality null, then it cannot have nontrivial finite-sample power even against the embedded alternatives \(\mathcal A_0\); that is, 
\[
    \sup_{Q\in\mathcal P_0^G} Q(\psi_n=1)\le \alpha  \quad \implies \quad       \sup_{P\in\mathcal A_0}   P(\psi_n=1)\le \alpha. 
\]
\end{proposition}

The proof is presented in Supplement~\ref{proof_of_no_free_lunch}. We now state the analogous corollary for causality in extremes. Let
\[
    \mathcal P_0^{ext}
    :=
    \left\{
        P\in\Xi_0:
        \Gamma_{\mathbf X\to\mathbf Y\mid\mathbf Z}\overset{a.s.}{<}1
    \right\}.
\]
We implicitly assume that \(\Gamma_{\mathbf X\to\mathbf Y\mid\mathbf Z}\) exists and all variables are supported on some neighborhood of infinity. 

\begin{corollary}
\label{corollary_no_free_lunch}
Let \(\psi_n:\mathbb R^{3n}\to\{0,1\}\) be any test of
\[
    H_0^{ext}:
    \Gamma_{\mathbf X\to\mathbf Y\mid\mathbf Z}<1, \qquad \text{ against} \qquad 
    H_1^{ext}:
    \Gamma_{\mathbf X\to\mathbf Y\mid\mathbf Z}=1.
\]
For every embedded alternative \(P\in\mathcal A_0\) holds 
\[
\sup_{Q\in\mathcal P_0^{ext}}Q(\psi_n=1)\le \alpha \qquad 
\text{ then } \qquad 
    P(\psi_n=1)\le \alpha. 
\]
\end{corollary}

The conclusion should not be interpreted as implying that the tests developed in the
structured parts of the paper are impossible. Rather, it shows that some
additional model assumptions are unavoidable. The structural
assumptions used in the paper are precisely the type of additional structure under which the
proposed estimators and testing procedures become meaningful.

\newpage
\section{Additional details and numerical results}
\label{appendix_all_simulations}

\subsection{Block bootstrap test for causality in tail}
\label{SupplSect_blockBootstrapAlgo}

Algorithm~\ref{Algorithm_testing_appendix} details the block-bootstrap procedure used to obtain the confidence intervals for the tail causality test describes in Section~\ref{Section_testing}.

\begin{algorithm}[tbh]
    \caption{Block Bootstrap test for causality in tail}\label{Algorithm_testing_appendix}
    \KwIn{Time series data $(x_1, y_1, \textbf{z}_1)^\top, \dots, (x_n, y_n, \textbf{z}_n)^\top$, block size $b$ (default $b=\sqrt{n})$, number of bootstrap samples $B$, significance level  $\alpha\in(0,1)$.}
    \KwOut{Test of the hypothesis $H_0: \Gamma_{\textbf{X}\to\textbf{Y} \mid \textbf{Z}} - \Gamma^{baseline}_{\textbf{X}\to\textbf{Y} \mid \textbf{Z}} = 0$.}
    \vspace{4pt}
Denote $\textbf{w}_i:=(x_i, y_i, \textbf{z}_i)^\top$ for $i=1, \dots, n$\;

    \For{$k = 1,\ldots,B$}{
        $\tilde{\textbf{w}}^{(k)} \gets \emptyset$\;
        \For{$i=1,\ldots,\left\lceil{n/b}\right\rceil$}{
            Randomly select a starting point $s$ from $\{1, 2, \dots, n-b+1\}$\;
            Extract block $B_s = \{\textbf{w}_s, \textbf{w}_{s+1}, \dots, \textbf{w}_{s+b-1}\}$\;
            Append block $B_s$ to $\tilde{\textbf{w}}^{(k)}$\;
        }
        Truncate $\tilde{\textbf{w}}^{(k)}$ to length $n$ if necessary\;

        Compute $\hat{\Delta}^{(k)}:=\hat{\Gamma}^{(k)}_{\textbf{X}\to\textbf{Y} \mid \textbf{Z}} - \hat{\Gamma}^{baseline {(k)}}_{\textbf{X}\to\textbf{Y} \mid \textbf{Z}}$ on the bootstrapped sample $\tilde{\textbf{w}}^{(k)}$\;
    }
    \eIf{\text{the $\alpha$-quantile of $\hat{\Delta}^{(1)},\ldots,\hat{\Delta}^{(B)}$ is strictly positive}}
    {
    \Return ``$H_0$ is rejected''\;
    }{
    \Return ``$H_0$ is not rejected''\;
    }
\end{algorithm}

\subsection{Hyper-parameters analysis}
\label{section_simulations_hyperparameters}

In this section, we outline our simulation study aimed at determining the optimal hyper-parameters discussed in Section~\ref{Section_choice_of_hyperparameters}. We employ two of the most prominent time series models: VAR and GARCH, to generate data, and always consider both $\textbf{X}\toverset{ext}{\to}\textbf{Y}$ and $\textbf{Y}\toverset{ext}{\not\to}\textbf{X}$. 
To assess the comparative efficacy of different hyper-parameters, we utilize Algorithm~\ref{Algorithm1} on the aforementioned models, computing their respective performance. Here, performance is measured by the percentage of correct outputs, when both $\textbf{X}\toverset{ext}{\to}\textbf{Y}$ and $\textbf{Y}\toverset{ext}{\not\to}\textbf{X}$ are inferred correctly, over $100$ repetitions for each of the four models. 
We, here, only focus on the classification algorithm's performance, since the results obtained from testing $H_0^{\rm tail}$ using a p-value, as discussed in Section~\ref{Section_testing}, yielded similar outcomes.

\begin{model}[VAR]
\label{VAR_time_series_definition}
Let $(\mathbf{X,Y,Z})^\top$ follow data-generating model
\begin{align*}
Z_t&=0.5Z_{t-1}+\varepsilon_t^Z,\\
X_t&=0.5X_{t-1}+ \alpha_Z Z_{t-1}+\varepsilon_t^X,\\
Y_t&=\alpha_YY_{t-1}+ \alpha_ZZ_{t-1}+ \alpha_X X_{t-1}+\varepsilon_t^Y,
\end{align*}
with independent noise variables $\varepsilon_t^X,  \varepsilon_t^Y, \varepsilon_t^Z$ and some hyper-parameters $\boldsymbol{\alpha}:=(\alpha_X, \alpha_Y, \alpha_Z)\in\mathbb{R}^3$. We refer to `heavy-tailed model \ref{VAR_time_series_definition}' when we generate  $\varepsilon_t^X,  \varepsilon_t^Y, \varepsilon_t^Z\sim {\rm Pareto}(1)$. We refer to `non-heavy-tailed model \ref{VAR_time_series_definition}' when we generate  $\varepsilon_t^X,  \varepsilon_t^Y, \varepsilon_t^Z\sim \mathcal{N}(0,1)$. 
\end{model}

\begin{model}[GARCH]
\label{GARCH_time_series_definition}
Let $(\mathbf{X,Y,Z})^\top$ follow data-generating model
\begin{align*}
Z_t&= \bigg(\frac{1}{10}+\frac{1}{10}Z_{t-1}^2 \bigg)^{1/2} \varepsilon^Z_t \\
X_t&= \bigg(\frac{1}{10}+\frac{1}{10}X_{t-1}^2 +  \alpha_Z Z_{t-1}^2  \bigg)^{1/2}\varepsilon^X_t \\
Y_t&= \bigg(\frac{1}{10} + \frac{\alpha_Y}{5}Y_{t-1}^2 +  \alpha_Z Z_{t-1}^2 +  \alpha_{X}X_{t-1}^2   \bigg)^{1/2}\varepsilon^Y_t,
\end{align*}
with independent noise variables $\varepsilon_t^X,  \varepsilon_t^Y, \varepsilon_t^Z$ and some hyper-parameters  $\boldsymbol{\alpha}:=(\alpha_X, \alpha_Y, \alpha_Z)\in\mathbb{R}^3$. We refer to `heavy-tailed model \ref{GARCH_time_series_definition}' when we generate  $\varepsilon_t^X,  \varepsilon_t^Y, \varepsilon_t^Z\sim {\rm Cauchy}(0,1)$. We refer to `non-heavy-tailed model \ref{GARCH_time_series_definition}' when we generate  $\varepsilon_t^X,  \varepsilon_t^Y, \varepsilon_t^Z\sim \mathcal{N}(0,1)$. 
\end{model}
Although not explicitly addressed, similar outcomes were observed with non-unit causal lags and when $\varepsilon_t^X, \varepsilon_t^Y, \varepsilon_t^Z$ exhibit distinct tail behaviors. In this scenario, $\alpha_X$ represents the causal effect of $\textbf{X}$ on $\textbf{Y}$, $\alpha_Y$ describes the auto-correlation of $\textbf{Y}$ and $\alpha_Z$ the confounding effect of $\textbf{Z}$ on $\textbf{X}$ and $ \textbf{Y}$.

\subsubsection[Choice of F]{Choice of $F$}
\label{Section_choice_F_{}}
In this section, we discuss the selection of the distribution function $F$ used in the estimators. We not only compare the choices $F = \hat{F}_Y(t)$ and $F = \hat{F}^{truc}_Y(t)$ as delineated in Section~\ref{Section_choice_of_hyperparameters}, but also more broadly examine
    \[  \hat{F}_Y^{truc(q_F)}(t):= \begin{cases} \hat{F}_Y(t) & \text{if } t \geq q_F \text{  quantile of }Y \\ 0 & \text{if } t < q_F \text{  quantile of }Y \end{cases} \]
across a range of $q_F\in [0,1]$. Note that $q_F=0$ corresponds to the choice $F = \hat{F}_Y(t)$, while $q_F=0.5$ corresponds to $F = \hat{F}^{truc}_Y(t)$.

We generate diverse datasets with a sample size of $n=500$ according to both heavy-tailed and non-heavy-tailed Models \ref{VAR_time_series_definition} and \ref{GARCH_time_series_definition}, with parameters $\alpha_Y=\alpha_Z=0.5$ and $\alpha_Y=\alpha_Z=0.1$, respectively. Employing Algorithm~\ref{Algorithm1} and selecting $F = \hat{F}_Y^{truc(q_F)}(t)$, we repeat the process $100$ times to assess the algorithm's performance as a function of $\alpha_X$. 

\begin{figure}[!t]
\centering
\includegraphics[width=0.77\textwidth]{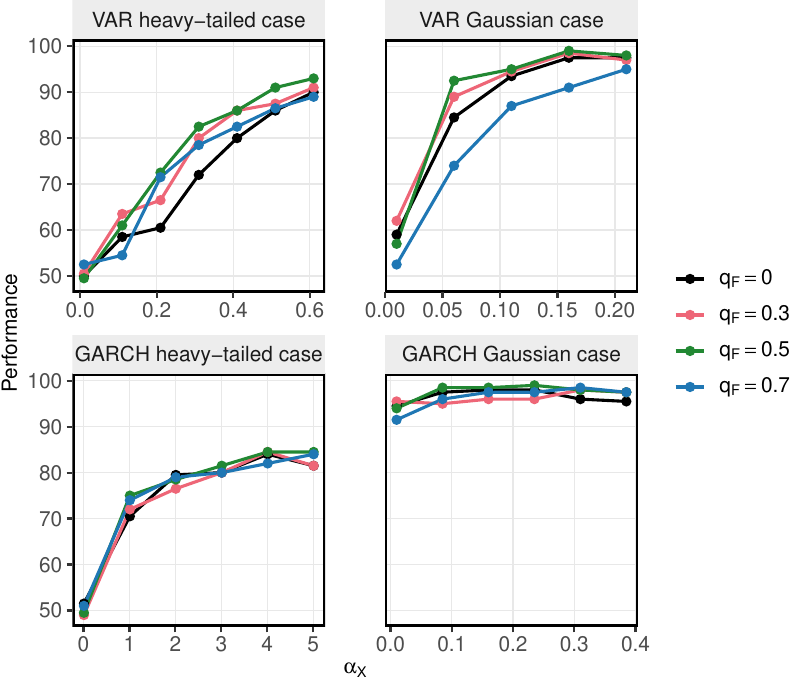}
\caption{Performance of Algorithm~\ref{Algorithm1} for a range of causal strengths $\alpha_X$, for different choices of $q_F$ in $\hat{F}_{}^{truc(q_F)}$, and for all four considered data models (Models~\ref{VAR_time_series_definition}  and \ref{GARCH_time_series_definition} with heavy- and non-heavy-tailed noise distributions).} %
\label{q_F}
\end{figure}
The results are depicted in Figure~\ref{q_F}. They
indicate that choosing $q_F  = 0.5$ results in optimal performance. Specifically, among all simulations conducted, the choice $q_F=0$ yielded correct outputs in $77\%$ of cases, $q_F=0.3$ yielded correct outputs in $80\%$ of cases, the $q_F=0.5$ choice led to correct outputs in $81\%$ of cases, and the $q_F=0.7$ choice resulted in correct outputs in $73\%$ of cases. Consequently, we opt for $F = \hat{F}^{truc}_Y(t)$. Nevertheless, we note that the differences between the different choices were small.

\subsubsection[Choice of tau-X (k)]{Choice of $\tau_X$ ($k_n$)}
\label{Section_choice_k_n}
A natural construction for $k_n$ is $k_n = \left \lfloor{n^\nu}\right \rfloor$ for some $\nu\in(0,1)$, since $k_n$ must satisfy (\ref{k_deleno_n}). \cite{Gnecco}, who considered i.i.d. random variables following a SCM, found that, in certain simulation setups, the value $\nu = 0.4$ is optimal. Conversely, \cite{bodik} used $\nu=\frac{1}{2}$ and argued that a lower $\nu$ may result in choosing all extreme values in the same cluster in a time series setting.

For the heavy-tailed and non-heavy-tailed Model \ref{VAR_time_series_definition} we use $\boldsymbol{\alpha}=(0.1, \frac{1}{2}, \frac{1}{2})$ and $\boldsymbol{\alpha}=(\frac{1}{2}, \frac{1}{2}, \frac{1}{2})$, respectively. For the heavy-tailed and non-heavy-tailed Model \ref{GARCH_time_series_definition} we use $\boldsymbol{\alpha}=(\frac{1}{2}, \frac{1}{2}, \frac{1}{2})$ and  $\boldsymbol{\alpha}=(1, \frac{1}{2}, \frac{1}{2})$, respectively.

For each of the four models, we generate time series with a sample size $n\in\{200,400,600\}$. %
Then, we apply Algorithm~\ref{Algorithm1} with $k_n = \left \lfloor{n^\nu}\right \rfloor$ with and without adjusting for $\textbf{Z}$, to assess robustness against hidden confounders. %
Figure~\ref{k_n} shows the algorithm's performance as a function of $\nu$. As all four considered models exhibited similar performance trends with respect to $\nu$, we present their aggregated performance. We observe that the optimal value seems to be around $\nu\approx \frac{1}{3}$ when the confounder is accounted for, and around $\nu\approx \frac{1}{2}$ when there is hidden confounding.

\begin{figure}[!t]
\centering
\includegraphics[width=0.7\textwidth]{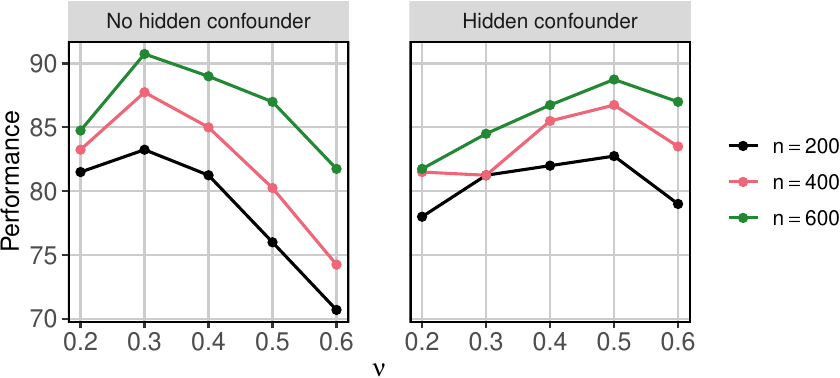}
\caption{Aggregated performance of Algorithm~\ref{Algorithm1} with $k_n = \left \lfloor{n^\nu}\right \rfloor $ as a function of $\nu$ over all data models, when the confounder is accounted for in the estimation (left) or ignored to simulate hidden confounding (right).} %
\label{k_n}
\end{figure}

We conclude that in scenarios where several potential confounders are modeled and strong hidden confounding is not expected, choosing $\nu\approx \frac{1}{3}$ seems preferable. Conversely, if a strong unmeasured confounder is anticipated, opting for a larger value around $\nu\approx \frac{1}{2}$ might be a better choice. Furthermore, from additional exploratory experiments, it seems that smaller values of $\nu$ might be advantageous when dealing with large sample sizes ($n\geq 10000$).

\subsubsection[Choice of tau-Y]{Choice of $\tau_Y$}
\label{Section_choice_tau_Y}

Recall that we define $\tau_Y$ as a $q_Y$-quantile of $Y$. 
For the heavy-tailed and non-heavy-tailed Model \ref{VAR_time_series_definition} we use $(\alpha_X, \alpha_Z)=(0.1, \frac{1}{2})$ and $(\alpha_X, \alpha_Z)=(\frac{1}{2}, \frac{1}{2})$, respectively. For the heavy-tailed and non-heavy-tailed Model \ref{GARCH_time_series_definition} we use $(\alpha_X, \alpha_Z)=(\frac{1}{2}, \frac{1}{2})$ and $(\alpha_X, \alpha_Z)=(1, \frac{1}{2})$, respectively. 

\begin{figure}[!t]
\centering
\includegraphics[width=0.7\textwidth]{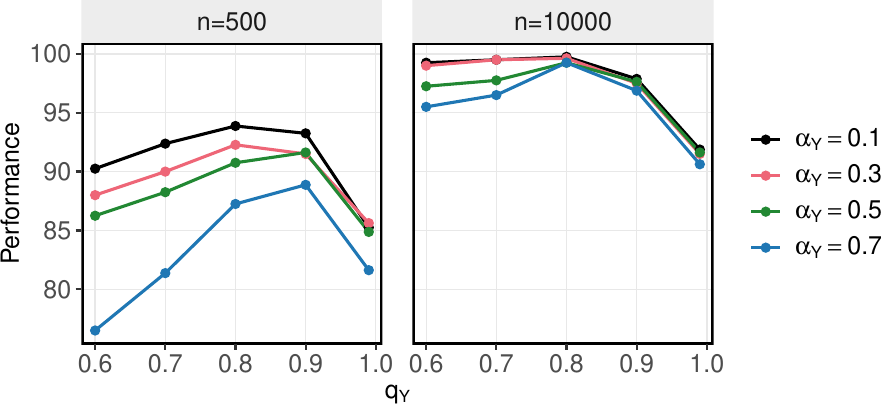}
\caption{Aggregated performance of Algorithm~\ref{Algorithm1} over all data models for different threshold choices $\tau_Y$ in set $S_1$, defined as $q_Y$-quantiles of $\textbf{Y}$, for various auto-correlation values~$\alpha_Y$ and sample sizes~$n$.}
\label{q_Y}
\end{figure}
Figure~\ref{q_Y} illustrates the obtained performance as a function of $q_Y$ for $n\in\{500, 10000\}$. 
The performances across all four considered models again displayed similar trends, hence we show their aggregated performance. 
We observe a seemingly optimal value around $q_Y\approx 0.8$, in most cases, with larger $q_Y$ values preferred under a significant autocorrelation structure of $Y$ and smaller sample size. 
Using different tail indexes and not accounting for the confounder in the model resulted in similar conclusions. %

\subsubsection[Choice of tau-Z]{Choice of $\tau_Z$}
\label{Section_choice_tau_Z}

Recall that for a $d$-dimensional confounder $\textbf{Z}\in\mathbb{R}^d$, we define $\tau_Z^i$ as the $q_Z^i\in (0,1)$ quantile of $Z_i, i=1, \dots, d$. We consider the case $d=1$, for simplicity.

For the heavy-tailed and non-heavy-tailed Model \ref{VAR_time_series_definition} we use $(\alpha_X, \alpha_Y)=(0.1, \frac{1}{2})$ and $(\alpha_X, \alpha_Y)=(\frac{1}{2}, \frac{1}{2})$, respectively. For the heavy-tailed and non-heavy-tailed Model \ref{GARCH_time_series_definition} we use $(\alpha_X, \alpha_Y)=(\frac{1}{2}, \frac{1}{2})$ and $(\alpha_X, \alpha_Y)=(1, \frac{1}{2})$, respectively. The sample size is $n=1000$. %
Figure~\ref{q_Z} illustrates the algorithm's performance as a function of $q_Z$, for various underlying confounding strengths.

\begin{figure}[!t]
\centering
\includegraphics[width=0.46\textwidth]{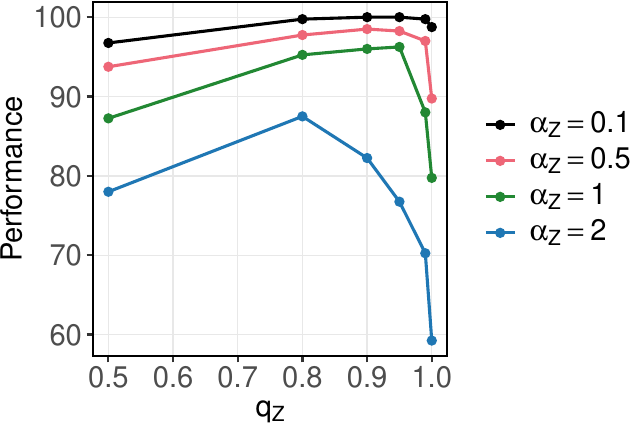}
\caption{Aggregated performance of Algorithm~\ref{Algorithm1} over all data models for different threshold choices $\tau_Z$ in set $S_1$, defined as $q_Z$-quantiles of $\textbf{Z}$ and for various confounding strengths $\alpha_Z$.}
\label{q_Z}
\end{figure}

Again, all four considered models exhibited similar performance trends, thus we present their aggregated performance. The results were consistent across different values of $\alpha_X, \alpha_Y$, different lags, and different tail indexes of the time series. We observe that the optimal value of $q_Z$ strongly depends on the strength of the confounding effect; as $\alpha_Z$ increases, the optimal $q_Z$ decreases. Particularly, under a very small confounding effect $\alpha_Z=0.1$, values around $q_Z\approx 0.99$ seem optimal, whereas under $\alpha_Z=2$ (where the effect of $Z$ is several times stronger than the effect of $X$), values around $q_Z\approx 0.8$ seem optimal.

We choose $q_Z = 0.9$ as a default, which seems to be a reasonable trade-off when the true strength is unknown. 
Furthermore, from additional informal experiments, larger values of $q_Z^i$ seem advantageous when dealing with dimensions $d>1$. Thus, we select $q_Z^i = 1- \frac{0.2}{d}$ for $i=1, \dots, d$ when $d>1$. However, it is worth noting that lower values of $q_Z^i$ should be chosen when a strong confounder is expected, in particular if its effect is stronger than that of $\textbf{X}$.

\clearpage

\subsection{Comparative performance study: additional results}
\label{appendix_simulations_Pa2}
Figure~\ref{Sim_6_results_Pa2} shows the results of the comparative simulation study on the VAR data process, with $\text{Pareto}(2)$ noise distribution.
This noise distribution choice is in between the `heavy-tailed' and the Gaussian cases discussed in Section~\ref{Section_Sim_perform} and shown in Figure~\ref{Sim_6_results}, in terms of heavy-tailness.

\begin{figure}[!tb]
\centering
\includegraphics[width=0.65\textwidth]{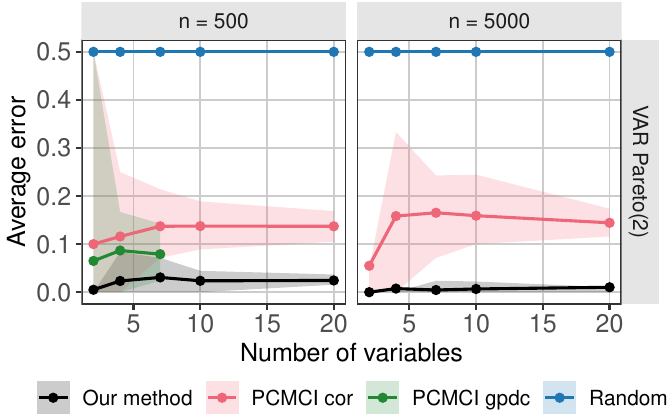}
\caption{Comparison of the average model errors between our approach and the competitors on the VAR data process, with $\text{Pareto}(2)$ noise distribution, for different numbers of variables (x-axis) and sample sizes (columns).
The variability bands show the 10--90\% inter-quantile spread across repetitions.
The ``random'' algorithm generates a random graph with each edge present with probability $\frac{1}{2}$. Due to time complexity constraints, PCMCI with GPDC independence test is estimated only for $n=500, m\leq 7$. 
}
\label{Sim_6_results_Pa2}
\end{figure}

\newpage

\subsection{Additional figures for Section~\ref{section_crypto_application}}
\label{appendix_crypto}

Figure~\ref{Crypto_graphs_full} shows additional results for the application to cryptocurrencies in Section~\ref{section_crypto_application}, when using Algorithm~\ref{Algorithm1} instead of the testing procedure, and when using a lag $p = 30$ instead of $p = 1$.

\begin{figure}[!tb]
\centering
\includegraphics[scale=0.62]{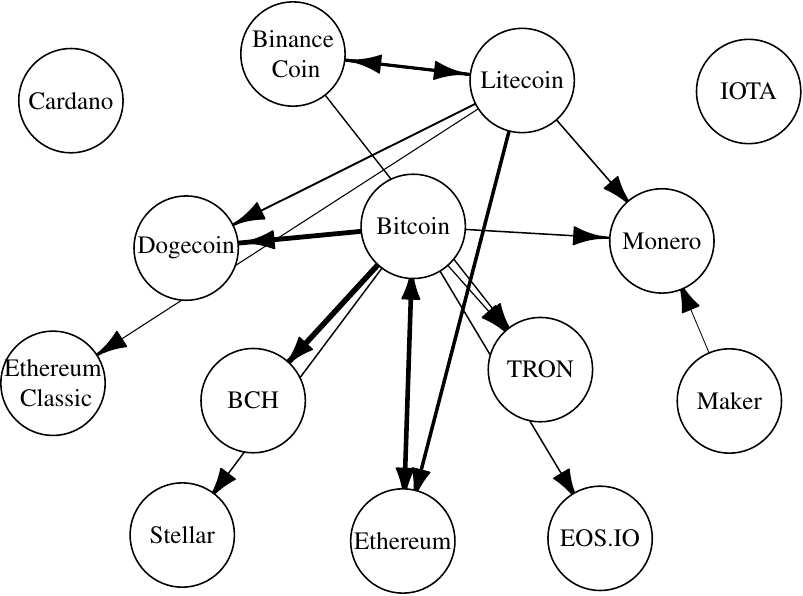}
\includegraphics[scale=0.62]{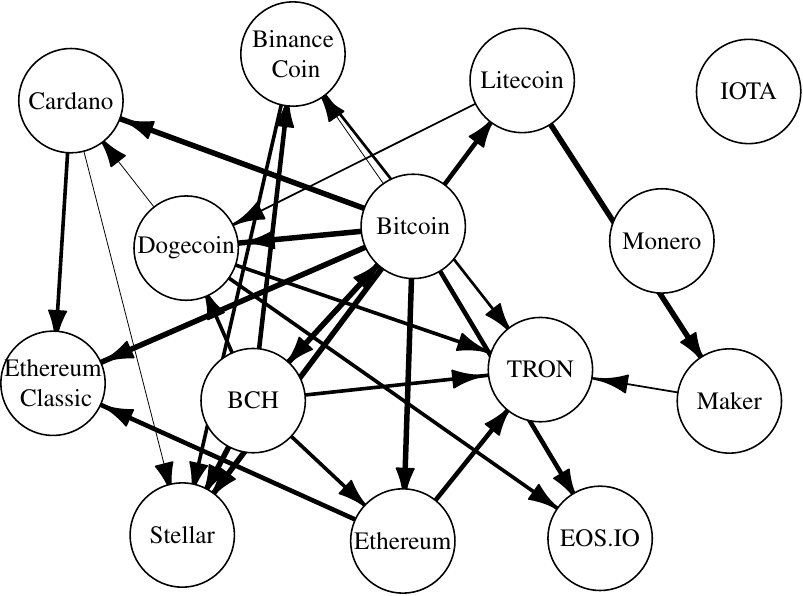}
\caption{Estimated causal graphs indicating Granger causality in extremes among the log returns of cryptocurrencies. {Top:}  Graph generated using Algorithm~\ref{Algorithm2} employing Algorithm~\ref{Algorithm1} with a lag of 1 min. The width of each edge represents the magnitude of $\hat{\Gamma}_{\textbf{X}^i\to\textbf{X}^j \mid \textbf{Z}}$; a value close to $1$ results in a wider edge, while a value close to $(1 + \hat{\Gamma}_{\textbf{X}^i\to\textbf{X}^j \mid \textbf{Z}}^{\baseline})/2$ is depicted with a narrower width. 
{Bottom:} Graph generated similarly to the Figure~\ref{Crypto_graphs}, but using a lag of 30 min.
}
\label{Crypto_graphs_full}
\end{figure}

\clearpage

\newpage
\section{Proofs} \label{Appendix_proofs}

\subsection{Auxiliary results}
\label{Appendix_auxiliary_results}
In this section, we provide auxiliary results, namely Lemmas~\ref{Observation}--\ref{LemmaAboutInfinity2}, that are used in subsequent proofs in the following sections.

\begin{lemma}\label{Observation}
Let $F$ be a distribution function satisfying $F(x)<1$ for all $x\in\mathbb{R}$.  Let $Z_1, Z_2$ be random variables supported on some neighborhood of infinity. Then, the following is equivalent:
\begin{itemize}
    \item    $
        \lim_{v\to\infty} \mathbb{E}[ F(Z_1)\mid Z_2>v] = 1,
   $
    \item for any $c\in\mathbb{R}$, $ \lim_{v\to\infty}P(Z_1> c\mid Z_2>v) = 1.$
\end{itemize}
\end{lemma}
\begin{proof}
``$\implies$''  Fix \(c\in\mathbb R\). Since \(F(c)<1\),
\[
1-\mathbb E[F(Z_1)\mid Z_2>v]
\ge
(1-F(c))P(Z_1\le c\mid Z_2>v).
\]
If the Left Side converges to 0 as $v\to\infty$, then necessary
\(P(Z_1\le c\mid Z_2>v)\to0\). Therefore, 
\(P(Z_1>c\mid Z_2>v)\to1\), which is what we wanted to prove. 

``$\Longleftarrow$'' Let  $\varepsilon>0$ and find $c$  such that $F(c)>1-\varepsilon$.  Then, 
\begin{align*}
   & \lim_{v\to\infty}P(Z_1> c\mid Z_2>v) = 1 \implies 
   \lim_{v\to\infty}P(F(Z_1)> 1-\varepsilon\mid Z_2>v) = 1\\&\implies \lim_{v\to\infty}\mathbb{E}[F(Z_1)\mid Z_2>v] >1-\varepsilon. 
\end{align*}
Sending $\varepsilon\to 0 $ finishes the proof. 
\end{proof}

\begin{lemma}\label{LemmaAboutInfinity}
Consider two independent real random variables $Z_1, Z_2$ and a measurable real function $h:\mathbb{R}^2\to\mathbb{R}$ such that $\lim_{v\to\infty}h(v, z_2) = \infty$ for any $z_2\in\mathbb{R}$. Let $Z_1$ be supported on some neighborhood of infinity. Then, for any $c\in\mathbb{R}$, 
$$
\lim_{v\to\infty} P(h(Z_1, Z_2)> c\mid Z_1>v)=1.
$$
\end{lemma}

\begin{proof}
Let $c\in\mathbb R$. It holds that

\begin{align*}
 P(h(Z_1,Z_2)\leq c\vert Z_1>v)&=\frac{ P(h(Z_1,Z_2)\leq c,Z_1>v)}{ P(Z_1>v)}\\
&=\int_{\mathbb{R}}\frac{ P(h(Z_1,z_2)\leq c,Z_1>v)}{ P(Z_1>v)}\,F_{2}(dz_2),
\end{align*}
where $F_2$ is the distribution of $Z_2$.

For a given $z_2\in\mathbb{R}$, the integrand is $0$ for $v$ large enough. We deduce that the integrand converges point-wise to $0$ as $v\to \infty$. As it is bounded by $1$, the dominated convergence theorem yields that the integral converges to $0$ as $v\to\infty$. Hence $\lim_{v\to\infty} P(h(Z_1,Z_2)\leq c\vert Z_1>v) = 0$, which concludes the proof. 
\end{proof}

\begin{lemma}\label{Observation2}
Let  $F^{\pm}$ be a continuous distribution function with $F^{\pm}(x)<1$ for all $x\in\mathbb{R}$. 
Let $Z_1, Z_2$ be random variables, where $Z_2$ is supported on some neighborhood of infinity. Then, the following are equivalent:
\begin{itemize}
    \item    $
        \lim_{|v|\to\infty} \mathbb{E}[ F^{\pm}(|Z_1|)\mid |Z_2|>v] = 1,
   $
    \item for any $c\in\mathbb{R}$, $ \lim_{v\to\infty}P(|Z_1|> c\mid |Z_2|>v) = 1.$
\end{itemize}
\end{lemma}
\begin{proof}
``$\implies$'' Fix $c\in\mathbb{R}$. Find $\varepsilon>0$ such that $F^{\pm}(c)<1-\varepsilon$ and find $v$ such that $\mathbb{E}[ F^{\pm}(|Z_1|)\mid |Z_2|>v] >1-\varepsilon.$ Then,  
\begin{align*}
    &P(F^{\pm}(|Z_1|)>1-\varepsilon\mid |Z_2|>v)>1-\varepsilon\\
    &P(|Z_1|>c\mid |Z_2|>v)>1-\varepsilon.
\end{align*}
Sending $\varepsilon\to 0 $ gives us the first implication. 

``$\Longleftarrow$'' Let  $\varepsilon>0$ and find $c$  such that  $F^{\pm}(|c|)<1-\varepsilon$.  We have 
\begin{align*}
   & \lim_{v\to\infty}P(|Z_1|> c\mid |Z_2|>v) = 1 \\&
   \lim_{v\to\infty}P(F^{\pm}(|Z_1|)> 1-\varepsilon\mid |Z_2|>v) = 1.
\end{align*}
Hence, we get $\lim_{v\to\infty}\mathbb{E}[F^{\pm}(|Z_1|)\mid |Z_2|>v] >1-\varepsilon$. Sending $\varepsilon\to 0 $ finishes the proof. 
\end{proof}

\begin{lemma}\label{LemmaAboutInfinity2}
Consider two independent real random variables $Z_1, Z_2$ and a measurable real function $h:\mathbb{R}^2\to\mathbb{R}$ such that $\lim_{|v|\to\infty}|h(v, z_2)| = \infty$ for any $z_2\in\mathbb{R}$. Let $Z_1$ be supported on some neighborhood of $\pm\infty$. Then, for any $c\in\mathbb{R}$, 
$$
\lim_{v\to\infty} P(|h(Z_1, Z_2)|> c\mid |Z_1|>v)=1.
$$
\end{lemma}

\begin{proof}
Let $c\in\mathbb R$. It holds that

\begin{align*}
 P(|h(Z_1,Z_2)|\leq c\vert |Z_1|>v)&=\frac{ P(|h(Z_1,Z_2)|\leq c,|Z_1|>v)}{ P(|Z_1|>v)}\\
&=\int_{\mathbb{R}}\frac{ P(|h(Z_1,z_2)|\leq c,|Z_1|>v)}{ P(|Z_1|>v)}\,F_{2}(dz_2),
\end{align*}
where $F_2$ is the distribution of $Z_2$.

For a given $z_2\in\mathbb{R}$, the integrand is $0$ for $v$ large enough. We deduce that the integrand converges point-wise to $0$ as $v\to \infty$. As it is bounded by $1$, the dominated convergence theorem yields that the integral converges to $0$ as $v\to\infty$. Hence $\lim_{v\to\infty} P(|h(Z_1,Z_2)|\leq c\vert |Z_1|>v) = 0$, what we wanted to show. 
\end{proof}

\subsection{Proofs of Propositions \ref{Proposition_CTC_iff_1},  \ref{Proposition_Granger_equivalent_tail} and \ref{Proposition_both_tails}}
\label{Proof_of_propositions_1_and_2}

\begin{customprop}{\ref{Proposition_CTC_iff_1} and  \ref{Proposition_Granger_equivalent_tail}}
$X\toverset{ext}{\to}Y \implies X\toverset{tail}{\to}Y \implies X\toverset{G}{\to}Y$. \textit{Under Assumption} \ref{AssumptionA1},   $X\toverset{G}{\to}Y \implies X\toverset{tail}{\to}Y \implies X\toverset{ext}{\to}Y$. 
\end{customprop}

\begin{proof}
In this proof, we use Lemma \ref{Observation} and Lemma \ref{LemmaAboutInfinity}. 
We prove the following three implications  
\begin{equation*}
    X\toverset{tail}{\to}Y\implies X\toverset{G}{\to}Y \overset{A1}{\implies}X\toverset{ext}{\to}Y\implies X\toverset{tail}{\to}Y.
\end{equation*}

\textbf{FIRST IMPLICATION $X\toverset{tail}{\to}Y$ implies $X\toverset{G}{\to}Y$:} 
We show the negation; that is, we show $X\toverset{G}{\not\to}Y$ implies $X\toverset{tail}{\not\to}Y$.  

If $X\toverset{G}{\not\to}Y$ then $Y_{t+1} \indep X_{\past(t)}\mid \mathcal{C}^{-\textbf{X}}_{t}$, which directly implies
\[
\lim_{v\to\infty} \mathbb{E}[ F(Y_{t+1})\mid X_t>v,\mathcal{C}^{-\textbf{X}}_t]=\mathbb{E}[ F(Y_{t+1})\mid  \mathcal{C}^{-\textbf{X}}_t].
\]
Hence $X\toverset{G}{\not\to}Y$ implies $X\toverset{tail}{\not\to}Y$. 

\textbf{THIRD IMPLICATION $X\toverset{ext}{\to}Y$ implies $X\toverset{tail}{\to}Y$:} 
Generally, it always holds that 
\[
\mathbb{E}[ F(Y_{t+1})\mid  \mathcal{C}^{-\textbf{X}}_t] < 1,
\]
since $F(y) < 1$ for all $y \in \mathbb{R}$. In more detail, if the distribution of $Y_{t+1} \mid \mathcal{C}^{-\textbf{X}}_t$ is well-defined and almost surely less than infinity, we also have that $F(Y_{t+1}) \mid \mathcal{C}^{-\textbf{X}}_t$ is almost surely less than 1. 
If $X\toverset{ext}{\to}Y$,  then
\[
\lim_{v\to\infty} \mathbb{E}[ F(Y_{t+1})\mid X_t>v,\mathcal{C}^{-\textbf{X}}_t] = 1 \neq \mathbb{E}[ F(Y_{t+1})\mid  \mathcal{C}^{-\textbf{X}}_t],
\]
which is what we wanted to prove. 

\textbf{SECOND IMPLICATION $X\toverset{G}{\to}Y$ implies $X\toverset{ext}{\to}Y$:}
We know that Granger causality implies structural causality. Due to  Lemma \ref{Observation}, if we show that for any $c \in \mathbb{R}$:
\begin{equation*}
 \lim_{v\to\infty}P(Y_{t+1} > c \mid X_t > v, \mathcal{C}^{-\textbf{X}}_t) = 1,   
\end{equation*}
then $X\toverset{ext}{\to}Y$, see Lemma \ref{Observation}.

Using the structural equation for $Y_{t+1}$, we rewrite
\[
\lim_{v\to\infty} P(Y_{t+1} > c \mid X_t > v, \mathcal{C}^{-\textbf{X}}_t) = \lim_{v\to\infty} P(h_{Y,t+1}(X_t, Y_t, \textbf{Z}_t, \varepsilon^Y_{t+1}) > c \mid X_t > v, Y_{\past(t)}, \textbf{Z}_{\past(t)}).
\]

Fix $y, \textbf{z}$ and define a function $\tilde{h}(x,e) := h_{Y,t+1}(x, y, \textbf{z}, e)$. Since $\varepsilon^Y_{t+1} \indep X_t \mid Y_{\past(t)}, \textbf{Z}_{\past(t)}$ and $\lim_{x\to\infty}\tilde{h}(x,e) = \infty$ for any $e$, we can directly use Lemma \ref{LemmaAboutInfinity}, which gives us $\lim_{v\to\infty}P(Y_{t+1} > c \mid X_t > v, \mathcal{C}^{-\textbf{X}}_t) = 1$. 
\end{proof}

\begin{customprop}{\ref{Proposition_both_tails}}
    \item If 
$
\Gamma_{|\textbf{X}|\to|\textbf{Y}|\mid\mathcal{C}} = 1,
$
then $\textbf{X}\toverset{tail^{\pm}}{\longrightarrow}\textbf{Y}$. Under Assumption \ref{AssumptionA2},  $\textbf{X}\toverset{tail^{\pm}}{\longrightarrow}\textbf{Y}$ implies $
\Gamma_{|\textbf{X}|\to|\textbf{Y}|\mid\mathcal{C}} = 1,
$.
    \item  $\textbf{X}\toverset{tail^{\pm}}{\longrightarrow}\textbf{Y}$ implies $\textbf{X}\toverset{G}{\longrightarrow}\textbf{Y}$. Under Assumption \ref{AssumptionA2},  $\textbf{X}\toverset{G}{\longrightarrow}\textbf{Y}$ implies $\textbf{X}\toverset{tail^{\pm}}{\longrightarrow}\textbf{Y}$. 
    \item  Under Assumption \ref{AssumptionA2}, the definition of  $\textbf{X}\toverset{tail^{\pm}}{\longrightarrow}\textbf{Y}$ is invariant with the choice of $F^{\pm}$. 
\end{customprop}
\begin{proof}\label{proof_of_proposition_both_tails}
The proof is fully analogous with the proof concerning the upper tail. Both-tail counterparts of Lemma \ref{Observation} and Lemma \ref{LemmaAboutInfinity} are restated in Lemma \ref{Observation2} and Lemma \ref{LemmaAboutInfinity2}. Apart of that, proving  $X\toverset{tail^{\pm}}{\to}Y\implies X\toverset{G}{\to}Y \overset{A2}{\implies}$ $\Gamma_{|\textbf{X}|\to|\textbf{Y}|\mid\mathcal{C}} = 1$ $\implies X\toverset{tail^{\pm}}{\to}Y$ is fully analogous to the proof of the upper tail, by substituting $|\textbf{X}|$ for $\textbf{X}$ and $|\textbf{Y}|$ for $\textbf{Y}$.  
\end{proof}

\subsection{Proof of Lemma \ref{lemma_o_invariance_about_F}  }

\begin{customlem}{\ref{lemma_o_invariance_about_F}}
Under Assumption \ref{AssumptionA1}, the definition of  $\textbf{X}\toverset{tail}{\longrightarrow}\textbf{Y}$ is invariant with the choice of $F$. That is, for any distribution functions $F_1, F_2$ satisfying $F_i(x)<1$ for all $x\in\mathbb{R}$, $i=1,2$, 
\begin{equation*}
    \begin{split}
        \lim_{v\to\infty} \mathbb{E}[ F_1(Y_{t+1})\mid X_t>v,&  \, \mathcal{C}^{-\textbf{X}}_t]\neq\mathbb{E}[ F_1(Y_{t+1})\mid  \mathcal{C}^{-\textbf{X}}_t]\\
        & \iff \\
        \lim_{v\to\infty} \mathbb{E}[ F_2(Y_{t+1})\mid X_t>v,&\,  \mathcal{C}^{-\textbf{X}}_t]\neq\mathbb{E}[ F_2(Y_{t+1})\mid  \mathcal{C}^{-\textbf{X}}_t]. 
    \end{split}
\end{equation*}
\end{customlem}

\begin{proof}\label{proof_of_lemma_o_invariance_about_F}
This is a direct consequence of Propositions \ref{Proposition_CTC_iff_1} and \ref{Proposition_Granger_equivalent_tail}, since  
\begin{equation*}
    \begin{split}
      &  \lim_{v\to\infty} \mathbb{E}[ F_1(Y_{t+1})\mid X_t>v,\mathcal{C}^{-\textbf{X}}_t]\neq\mathbb{E}[ F_1(Y_{t+1})\mid  \mathcal{C}^{-\textbf{X}}_t]\\&\iff \lim_{v\to\infty} \mathbb{E}[ F_1(Y_{t+1})\mid X_t>v,\mathcal{C}^{-\textbf{X}}_t]=1\\&
      \iff X\toverset{G}{\to}Y\\&
      \iff \lim_{v\to\infty} \mathbb{E}[ F_2(Y_{t+1})\mid X_t>v,\mathcal{C}^{-\textbf{X}}_t]=1
      \\&\iff  \lim_{v\to\infty} \mathbb{E}[ F_2(Y_{t+1})\mid X_t>v,\mathcal{C}^{-\textbf{X}}_t]\neq\mathbb{E}[ F_2(Y_{t+1})\mid  \mathcal{C}^{-\textbf{X}}_t].
    \end{split}
\end{equation*}
\end{proof}

\newpage
\subsection{Proof of Theorem~\ref{Theorem1} and Lemma~\ref{Theorem1_both_tails_lemma}}
\label{proof_of_thm1}
Before we prove Theorem \ref{Theorem1}, we first introduce some auxiliary lemmas. 

\begin{lemma}\label{PropositionLemma}
 \begin{itemize}
     \item Let a pair $(A,B)$ satisfy the Grey assumption with index $\alpha>0$, where $A$ is positive. Let $X$ be a random variable independent of $(A,B)$ which is either regularly varying or $P(X>t) = o(P(B>t))$. Then,$$P(AX+B>t)\sim P(AX>t)+P(B>t), \quad \text{ as }t\to\infty.  $$
     \item Let $A_1, A_2, A_3$ be positive random variables such that all pairs $(A_1, B)$, $(A_2, B)$, $(A_3,B)$ satisfy the Grey assumption with index $\alpha>0$, and $(X,Y,Z)$ are random variables independent of $(A_1, A_2, A_3,B)$. Then,
 $$P(A_1X+A_2Y+A_3Z+ B>t)\sim P(A_1X+A_2Y+A_3Z>t)+P(B>t),$$
as $t\to\infty. $\end{itemize}
\end{lemma}

\begin{proof}
Essentially, this is a non-trivial consequence of  \citep[Lemma B.6]{SRE}. Before we prove the Lemma, we first show three true statements. 

\textbf{Claim 1: }  $\lim_{t\to\infty}P(cA>t\mid |B|>t)= 0$ for any $c\in\mathbb{R}$.

This claim is true, since $$
P(cA>t\mid |B|>t) = P(|B|>t\mid cA>t)\frac{P(cA>t)}{P(|B|>t)}\leq \frac{P(cA>t)}{P(|B|>t)}\overset{t\to\infty}{\to} 0,
$$
as $cA$ has finite $\alpha$th expectation and $|B|$ is regularly varying. 

\textbf{Claim 2: } $\lim_{t\to\infty}P(A|X|>t\mid |B|>t)= 0.$ 

Let $\varepsilon>0$, and let $c>0$ such that $P(|X|>c)<\varepsilon$. Then, $$P(A|X|>t\mid |B|>t)\leq P(cA>t\mid |B|>t)+\varepsilon.$$
From Claim 1, the right side converges to $0+\varepsilon$. This claim is proven by taking $\varepsilon\to 0$. 

\textbf{Claim 3:}  If $
    P(|AX|>t, |B|>t) = o\big(P(AX>t)+P(B>t)\big),
    $    then $P(AX+B>t)\sim P(AX>t)+P(B>t)$. 

    This result is known as the max-sum equivalence theorem for dependent variables  \citep[Lemma B.6]{SRE}. 

\textbf{Proof of the lemma:     } Using Bayes theorem, we rewrite $$ \frac{P(|AX|>t, |B|>t)}{P(AX>t)+P(B>t)} = \frac{P(A|X|>t\mid |B|>t)}{\frac{P(B>t)}{P(|B|>t)}+\frac{P(AX>t)}{P(|B|>t)}}\overset{t\to\infty}{\to}0,$$
since the denominator if non-zero and the nominator converges to $0$ from Claim 2. Finally, Claim 3 concludes the proof.  

The second part of the lemma is a direct generalization of the first part. 

\end{proof}

\begin{customthm}{\ref{Theorem1}}
Consider a time series following the SRE model (\ref{sre2}), satisfying (B1), (B2), (B4). 
\begin{itemize}
\item Under (B3), $\Gamma_{\textbf{X}\to\textbf{Y}\mid\mathcal{C}} = 1 \implies \Gamma_{\textbf{X}\to\textbf{Y}\mid \emptyset} = 1$. That is, 
\begin{equation*}
\lim_{v\to\infty} \mathbb{E}[ F(Y_{t+1})\mid X_t>v,\mathcal{C}^{-\textbf{X}}_t]=1 \implies \lim_{v\to\infty} \mathbb{E}[ F(Y_{t+1})\mid X_t>v, \textbf{Y}_{\past(t)}]=1. 
\end{equation*}
\item Let the pairs  $(A^x_{1,t},B^x_t)^\top,(A^x_{2,t},B^x_t)^\top,(A^x_{3,t},B^x_t)^\top$ satisfy the Grey assumption with index $\alpha_x$, and $\operatorname{lim sup}_{u\to\infty}\frac{P(X_t>u\mid Y_{past(t)})}{P(B_t^x>u)}\overset{a.s.}{<}\infty$. Then, $\Gamma_{\textbf{X}\to\textbf{Y}\mid\mathcal{C}} = 1 \;\Longleftarrow\; \Gamma_{\textbf{X}\to\textbf{Y}\mid \emptyset} = 1$. That is, 
\begin{equation*}
\lim_{v\to\infty} \mathbb{E}[ F(Y_{t+1})\mid X_t>v,\mathcal{C}^{-\textbf{X}}_t]=1 \;\Longleftarrow\; \lim_{v\to\infty} \mathbb{E}[ F(Y_{t+1})\mid X_t>v, \textbf{Y}_{\past(t)}]=1. 
\end{equation*}
\end{itemize}
\end{customthm}
\begin{proof}
We consider $t$ to be fixed in the entire proof, and all limits below are understood almost surely. We will extensively use the notation 
\begin{equation*}
    \begin{split}
Y_{t+1}&=A_{1,t+1}^yZ_t + A_{2,t+1}^yX_t + A_{3,t+1}^yY_t+B_{t+1}^y 
    \end{split}
\end{equation*}
\textbf{Claim 1.} Under (B2) and (B4),
\[
\lim_{v\to\infty}
\mathbb E\!\left[
F(Y_{t+1})\mid X_t>v,\mathcal C_t^{-\mathbf X}
\right]=1
\quad\Longleftrightarrow\quad
P(A^y_{2,t+1}=0)\neq 1 .
\]

\textbf{Proof of Claim 1: } ``$\implies$'' For a contradiction, assume  $A_{2,t+1}^y\toverset{a.s.}{=}0$ and write 
\begin{equation}\label{eq97}
  Y_{t+1}=A_{1,t+1}^yZ_t + A_{2,t+1}^yX_t + A_{3,t+1}^yY_t+B_{t+1}^y\toverset{a.s.}{=}A_{1,t+1}^yZ_t + A_{3,t+1}^yY_t+B_{t+1}^y.   
\end{equation}
Notice that $A^y_{j,t+1}\indep X_t$ for $j=1,2,3$.  Together, we have that $\mathbb{E}[ F(Y_{t+1})\mid X_t>v,\mathcal{C}^{-\textbf{X}}_t]= \mathbb{E}[ F(Y_{t+1})\mid \mathcal{C}^{-\textbf{X}}_t]$, since $Y_{t+1}$ is only a function of $Y_t, Z_t, A^y_{j,t+1}, B^y_{t+1}$ which are independent of $X_t$ given $Y_t, Z_t$. However it is always $ \mathbb{E}[ F(Y_{t+1})\mid \mathcal{C}^{-\textbf{X}}_t]\neq 1$ (since $F(y)<1$ for all $y\in\mathbb{R}$, see the same argument in the proof of Proposition \ref{Proposition_CTC_iff_1}).  That is a contradiction.

``$\impliedby$'' Suppose \(P(A^y_{2,t+1}=0)\neq1\). By (B4), \(A^y_{2,t+1}>0\) a.s. Fix \(c\in\mathbb R\), condition on
\(\mathcal C_t^{-\mathbf X}\), and then fix a realization \(\varepsilon^y_{t+1}=(a_1,a_2,a_3,b)\) with \(a_2>0\). Since \(Y_t\) and \(Z_t\) are then fixed,
\[
a_1Z_t+a_2X_t+a_3Y_t+b \to \infty
\qquad \text{as } X_t\to\infty .
\]
More formally, the following holds:
\[
P(Y_{t+1}>c
\mid X_t>v,\mathcal C_t^{-\mathbf X},\varepsilon^y_{t+1})
\to 1 .
\]
Dominated convergence over \(\varepsilon^y_{t+1}\), followed by Lemma~\ref{Observation}, proves the claim.

\textbf{Proof of the first bullet-point. }
Assume \(\Gamma_{\mathbf X\to\mathbf Y\mid\mathcal C}=1\). By the equivalence above and \((B4)\), we have
\[
  \Gamma_{\mathbf X\to\mathbf Y\mid\mathcal C}=1 \iff    A^y_{2,t+1}>0
    \qquad\text{a.s.}
\]

Due to Lemma~\ref{Observation}, in order to prove \(
    \Gamma_{X\to Y\mid \emptyset}=1
\), it is enough to show that, for every \(c\in\mathbb R\),
\begin{equation}
    \label{eq432}
    P(Y_{t+1}>c\mid X_t>v,\mathbf Y_{\past(t)})\to 1 .
\end{equation}

Fix \(c\in\mathbb{R}\). Since
\(\varepsilon^y_{t+1}\) is independent of \((X_t,Z_t,\mathbf Y_{\past(t)})\), condition
on a realization \(\varepsilon^y_{t+1}=(a_1,a_2,a_3,b)\) with \(a_2>0\).
We can assume \(a_1>0\), because otherwise \(Y_{t+1}=a_2X_t+a_3Y_t+b\), so the conditional probability \eqref{eq432} converges trivially to one.

Now let \(a_1>0\), and put \(a:=\frac{a_2}{2a_1}\). For all sufficiently large \(v\), depending only on \(c,Y_t,a_1,a_2,a_3,b\),
\[
\{X_t>v,\ Z_t>-av\}
\subseteq
\{Y_{t+1}>c\}.
\]
Therefore
\[
P(Y_{t+1}>c
\mid X_t>v,\mathbf Y_{\past(t)},\varepsilon^y_{t+1})
\ge
P(Z_t>-av\mid X_t>v,\mathbf Y_{\past(t)}).
\]
The right-hand side converges to one by the assumption (B3). Dominated
convergence over \(\varepsilon^y_{t+1}\) gives
\[
P(Y_{t+1}>c\mid X_t>v,\mathbf Y_{\past(t)})\to1 .
\]
Hence \(\Gamma_{\mathbf X\to\mathbf Y\mid\emptyset}=1\).

\textbf{Proof of the second bullet-point. }
We show the negated claim. Assume \(\Gamma_{\mathbf X\to\mathbf Y\mid\mathcal C}\neq1\). By the Claim 1, \(A^y_{2,t+1}=0\) a.s., and hence
\[
Y_{t+1}
=
A^y_{1,t+1}Z_t
+
A^y_{3,t+1}Y_t
+
B^y_{t+1}.
\]
Choose a threshold \(c_t\) such that \(
p_t:=P(Y_{t+1}\le c_t\mid\mathbf Y_{\past(t)})>0. 
\) By Bayes' formula,
\begin{equation*}
    \begin{split}
        P(Y_{t+1}\le c_t\mid X_t>v,\mathbf Y_{\past(t)})
&=
p_t\,
\frac{
P(X_t>v\mid Y_{t+1}\le c_t,\mathbf Y_{\past(t)})
}{
P(X_t>v\mid\mathbf Y_{\past(t)})
} \\& = p_t\,
\frac{
P(X_t>v\mid Y_{t+1}\le c_t,\mathbf Y_{\past(t)})
}{P(B^x_t>v)} \frac{P(B^x_t>v)}{
P(X_t>v\mid\mathbf Y_{\past(t)})
}.
    \end{split}
\end{equation*}
Write
\[
X_t
=
\underbrace{
A^x_{1,t}Z_{t-1}
+
A^x_{2,t}X_{t-1}
+
A^x_{3,t}Y_{t-1}
}_{=:Q_t}
+
B^x_t .
\]
Notice that, since \(A^y_{2,t+1}=0\), the event
\(\{Y_{t+1}\le c_t\}\) does not involve \(\varepsilon^x_t\). Let
\(\widetilde W_{t-1}\) have the conditional law of \(W_{t-1}\) given
\(\{Y_{t+1}\le c_t\}\) and \(\mathbf Y_{\past(t)}\); however, let it be chosen conditionally independent of
\(\varepsilon^x_t\). Define
\[
\widetilde Q_t
=
A^x_{1,t}\widetilde Z_{t-1}
+
A^x_{2,t}\widetilde X_{t-1}
+
A^x_{3,t}\widetilde Y_{t-1}.
\]
Then
\[
P(X_t>v\mid Y_{t+1}\le c_t,\mathbf Y_{\past(t)})
=
P(B^x_t+\widetilde Q_t>v\mid\mathbf Y_{\past(t)}).
\]
Now apply Lemma~\ref{PropositionLemma} conditionally on \(\mathbf Y_{\past(t)}\):
\[
P(B^x_t+\widetilde Q_t>v\mid\mathbf Y_{\past(t)})
\sim
P(\widetilde Q_t>v\mid\mathbf Y_{\past(t)})
+
P(B^x_t>v),
\]
and in particular
\[
\liminf_{v\to\infty}
\frac{
P(X_t>v\mid Y_{t+1}\le c_t,\mathbf Y_{\past(t)})
}{
P(B^x_t>v)
}
\ge 1 .
\]
By the assumption $\operatorname{lim sup}_{u\to\infty}\frac{P(X_t>u\mid Y_{past(t)})}{P(B_t^x>u)}\overset{a.s.}{<}\infty$, there is a finite \(M_t>0\) such that, for all sufficiently large \(v\),
\[
P(X_t>v\mid\mathbf Y_{\past(t)})
\le
M_t P(B^x_t>v).
\]
Consequently, merging the two results, we obtain
\[
\liminf_{v\to\infty}
P(Y_{t+1}\le c_t\mid X_t>v,\mathbf Y_{\past(t)})
\ge
p_t \cdot 1 \cdot \frac{1}{M_t}
>0 .
\]
Using the argument from Lemma~\ref{Observation}, since \(F(c_t)<1\) a.s., we have
\[
1-
\mathbb E[F(Y_{t+1})\mid X_t>v,\mathbf Y_{\past(t)}]
\ge
(1-F(c_t))
P(Y_{t+1}\le c_t\mid X_t>v,\mathbf Y_{\past(t)}).
\]
Taking the limit inferior on the right gives
\[
\limsup_{v\to\infty}
\mathbb E[F(Y_{t+1})\mid X_t>v,\mathbf Y_{\past(t)}]
<1 .
\]
Thus \(\Gamma_{\mathbf X\to\mathbf Y\mid\emptyset}\neq1\). This proves the theorem. 
\end{proof}

\begin{customlem}{\ref{Theorem1_both_tails_lemma}}
Consider a time series following an SRE model as in~(\ref{sre2}) satisfying
(B1), (B2), and (B5). Assume that all limits below exist almost surely.

\begin{itemize}
    \item Define the following two-sided analogue of (B3):
\[
\tag{B3$^\pm$}
\lim_{v\to\infty}
P\bigl(|Z_t|\leq a|X_t| \mid |X_t|>v,Y_{past(t)}\bigr)=1
\qquad\text{for every }a>0.
\]
Then \(
\Gamma_{|\mathbf X|\to|\mathbf Y|\mid\mathcal C}=1
\implies
\Gamma_{|\mathbf X|\to|\mathbf Y|\mid\emptyset}=1.
\)

 \item Assume a two-sided analogue of the Grey-type tail condition: There exist
\(\alpha_x>0\) and \(\nu>0\) such that \(|B_t^x|\in RV(\alpha_x)\) and
\(
E|A^x_{j,t}|^{\alpha_x+\nu}<\infty,
 j=1,2,3,
\)
and assume the conditional tail-dominance condition
\(
\limsup_{u\to\infty}
\frac{P(|X_t|>u\mid Y_{past(t)})}{P(|B_t^x|>u)}
<\infty.
\) Then \(
\Gamma_{|\mathbf X|\to|\mathbf Y|\mid\emptyset}=1
\implies
\Gamma_{|\mathbf X|\to|\mathbf Y|\mid\mathcal C}=1.
\)
\end{itemize}
\end{customlem}
\begin{proof}\label{proof_of_Theorem1_both_tails_lemma}
We write
\[
Y_{t+1}
=
A^y_{1,t+1}Z_t+A^y_{2,t+1}X_t+A^y_{3,t+1}Y_t+B^y_{t+1}.
\]
\textbf{Claim~$1^\pm$} : under (B2) and (B5),
\[
\lim_{v\to\infty}
\mathbb E\!\left[
F^\pm(|Y_{t+1}|)\mid |X_t|>v,\mathcal C_t^{-\mathbf X}
\right]=1
\iff
P(A^y_{2,t+1}=0)\neq 1 .
\]
The proof is identical to the proof of Claim~1 in Theorem~\ref{Theorem1}, using
Lemma~\ref{Observation2} instead of Lemma~\ref{Observation}. Moreover, by (B5),
\(P(A^y_{2,t+1}=0)\neq 1\) implies \(P(A^y_{2,t+1}=0)=0\).

\textbf{First bullet point.}
Assume \(\Gamma_{|\mathbf X|\to|\mathbf Y|\mid\mathcal C}=1\). By the above claim,
\(P(A^y_{2,t+1}=0)=0\). By Lemma~\ref{Observation2}, applied conditionally on
\(\mathbf Y_{\past(t)}\), it is enough to show that, for every \(c\geq0\),
\[
P\bigl(|Y_{t+1}|\leq c\mid |X_t|>v,\mathbf Y_{\past(t)}\bigr)\to0 .
\]
Condition on \(
\varepsilon^y_{t+1}:=
(A^y_{1,t+1},A^y_{2,t+1},A^y_{3,t+1},B^y_{t+1})
=(a_1,a_2,a_3,b),
\)
where \(a_2\neq0\). By the innovation independence assumptions and (B2),
\(\varepsilon^y_{t+1}\) is independent of \((X_t,Z_t,\mathbf Y_{\past(t)})\). Hence it
suffices to prove that
\[
P\bigl(|a_1Z_t+a_2X_t+a_3Y_t+b|\leq c\mid |X_t|>v,\mathbf Y_{\past(t)}\bigr)\to0 .
\]
Set \(r:=a_3Y_t+b\). If \(a_1=0\), then on \(\{|X_t|>v\}\),
\[
|a_2X_t+r|\geq |a_2|v-|r|,
\]
which is larger than \(c\) for all sufficiently large \(v\). Thus the last
probability is eventually zero.

Now suppose \(a_1\neq0\), and set \(
\eta:=\frac{|a_2|}{2|a_1|}.
\) On \(\{|Z_t|\leq \eta |X_t|\}\), we have
\[
|a_1Z_t+a_2X_t+r|
\geq
|a_2||X_t|-|a_1||Z_t|-|r|
\geq
\frac{|a_2|}{2}|X_t|-|r|.
\]
Therefore, on \(\{|X_t|>v,\ |Z_t|\leq \eta |X_t|\}\), the last display is
larger than \(c\) for all sufficiently large \(v\). Hence
\[
P\bigl(|a_1Z_t+a_2X_t+r|\leq c\mid |X_t|>v,\mathbf Y_{\past(t)}\bigr)
\leq
P\bigl(|Z_t|>\eta |X_t|\mid |X_t|>v,\mathbf Y_{\past(t)}\bigr)\to0
\]
by (B3\(^{\pm}\)). Dominated convergence over
\(\varepsilon^y_{t+1}\) gives
\[
P\bigl(|Y_{t+1}|\leq c\mid |X_t|>v,\mathbf Y_{\past(t)}\bigr)\to0 .
\]
Thus \(
\Gamma_{|\mathbf X|\to|\mathbf Y|\mid\emptyset}=1 .
\)

\textbf{Second bullet point.}
We prove the contrapositive. Assume
\(
\Gamma_{|\mathbf X|\to|\mathbf Y|\mid\mathcal C}<1 .
\) 
By the Claim $1^{\pm}$,
\(P(A^y_{2,t+1}=0)=1\), and therefore
\[
Y_{t+1}=A^y_{1,t+1}Z_t+A^y_{3,t+1}Y_t+B^y_{t+1}.
\]
In particular, conditionally on \(\mathbf Y_{\past(t)}\), the variable \(Y_{t+1}\) does
not involve \(
\varepsilon^x_t=(A^x_{1,t},A^x_{2,t},A^x_{3,t},B^x_t).
\)

Choose a threshold \(c_t\) such that
\[
p_t:=P(|Y_{t+1}|\leq c_t\mid\mathbf Y_{\past(t)})>0,
\]
where the threshold can be any finite \(\mathbf Y_{\past(t)}\)-measurable random; for instance, take the first integer \(m\geq1\) such that \(P(|Y_{t+1}|\leq m\mid\mathbf Y_{\past(t)})>0\). Such an integer exists a.s. because \(Y_{t+1}\) is finite a.s.

By Bayes' formula,
\[
P(|Y_{t+1}|\leq c_t\mid |X_t|>v,\mathbf Y_{\past(t)})
=
p_t
\frac{
P(|X_t|>v\mid |Y_{t+1}|\leq c_t,\mathbf Y_{\past(t)})
}{
P(|X_t|>v\mid\mathbf Y_{\past(t)})
}.
\]
It remains to lower-bound the numerator. Write
\[
X_t=Q_t+B^x_t,
\qquad
Q_t:=A^x_{1,t}Z_{t-1}+A^x_{2,t}X_{t-1}+A^x_{3,t}Y_{t-1}.
\]
Let
\((\widetilde Z_{t-1},\widetilde X_{t-1},\widetilde Y_{t-1})\) have the
conditional distribution of
\((Z_{t-1},X_{t-1},Y_{t-1})\) given
\(\{|Y_{t+1}|\leq c_t\}\) and \(\mathbf Y_{\past(t)}\), chosen conditionally
independently of \(\varepsilon^x_t\). Define
\[
\widetilde Q_t
:=
A^x_{1,t}\widetilde Z_{t-1}
+
A^x_{2,t}\widetilde X_{t-1}
+
A^x_{3,t}\widetilde Y_{t-1}.
\]
Since \(A^y_{2,t+1}=0\), the conditioning event
\(\{|Y_{t+1}|\leq c_t\}\) does not involve \(\varepsilon^x_t\). Hence
\[
P(|X_t|>v\mid |Y_{t+1}|\leq c_t,\mathbf Y_{\past(t)})
=
P(|B^x_t+\widetilde Q_t|>v\mid\mathbf Y_{\past(t)}).
\]
Fix \(\delta\in(0,1)\). Then
\[
\begin{aligned}
P(|B^x_t+\widetilde Q_t|>v\mid\mathbf Y_{\past(t)})
&\geq
P(|B^x_t|>(1+\delta)v) \\
&\quad -
P(|B^x_t|>(1+\delta)v,\ |\widetilde Q_t|>\delta v
\mid\mathbf Y_{\past(t)}).
\end{aligned}
\]
We prove that the second term is \(o(P(|B^x_t|>v))\). By the union bound, it is
bounded by
\[
\sum_{j=1}^3
P\left(
|B^x_t|>(1+\delta)v,\
|A^x_{j,t}\widetilde U_j|>\frac{\delta v}{3}
\mid\mathbf Y_{\past(t)}
\right),
\]
where
\(
(\widetilde U_1,\widetilde U_2,\widetilde U_3)
=
(\widetilde Z_{t-1},\widetilde X_{t-1},\widetilde Y_{t-1}).
\)
For each \(j\), \(\widetilde U_j\) is conditionally independent of
\((A^x_{j,t},B^x_t)\) given \(\mathbf Y_{\past(t)}\). Moreover, for every fixed \(M>0\),
\[
\begin{aligned}
&
P\left(
|A^x_{j,t}\widetilde U_j|>\frac{\delta v}{3}
\mid
|B^x_t|>(1+\delta)v,\mathbf Y_{\past(t)}
\right) \\
&\qquad\leq
P(|\widetilde U_j|>M\mid\mathbf Y_{\past(t)})
+
\frac{
P\left(|A^x_{j,t}|>\delta v/(3M)\right)
}{
P\left(|B^x_t|>(1+\delta)v\right)
}.
\end{aligned}
\]
The second term tends to zero because
\(E|A^x_{j,t}|^{\alpha_x+\nu}<\infty\) and
\(|B^x_t|\in RV(\alpha_x)\). Then taking \(M\to\infty\) proves the claim.

Since \(|B^x_t|\in RV(\alpha_x)\), we obtain
\[
\liminf_{v\to\infty}
\frac{
P(|X_t|>v\mid |Y_{t+1}|\leq c_t,\mathbf Y_{\past(t)})
}{
P(|B^x_t|>v)
}
\geq
(1+\delta)^{-\alpha_x}>0 .
\]
By the assumption \(
\limsup_{u\to\infty}
\frac{P(|X_t|>u\mid Y_{past(t)})}{P(|B_t^x|>u)}
<\infty
\), there exists a finite
(possibly \(\mathbf Y_{\past(t)}\)-measurable random variable) \(M_t>0\) such that, for all
sufficiently large \(v\),
\[
P(|X_t|>v\mid\mathbf Y_{\past(t)})
\leq
M_t P(|B^x_t|>v).
\]
Therefore,
\[
\liminf_{v\to\infty}
P(|Y_{t+1}|\leq c_t\mid |X_t|>v,\mathbf Y_{\past(t)})
\geq
\frac{p_t(1+\delta)^{-\alpha_x}}{M_t}
>0 .
\]
Finally, since \(c_t<\infty\) a.s. and \(F^\pm(c_t)<1\) a.s.,
\[
\begin{aligned}
1-
\mathbb E[
F^\pm(|Y_{t+1}|)\mid |X_t|>v,\mathbf Y_{\past(t)}]
&\geq
(1-F^\pm(c_t))
P(|Y_{t+1}|\leq c_t\mid |X_t|>v,\mathbf Y_{\past(t)}).
\end{aligned}
\]
Taking $v\to\infty$ gives \(
\Gamma_{|\mathbf X|\to|\mathbf Y|\mid\emptyset}<1
\), what we wanted to prove. 

\end{proof}

\newpage
\subsection{Proof of Theorem \ref{Theorem2v2}}
\label{Proof of Theorem 2v2}

\begin{lemma}[Freedman's inequality, \cite{freedman1975tail}]
\label{lemma:freedman}
Let \((M_k,\mathcal F_k)_{k\ge0}\) be a martingale with \(M_0=0\). Write
\[
    M_k=\sum_{j=1}^k \xi_j,
    \qquad
    \mathbb E[\xi_j\mid\mathcal F_{j-1}]=0,
\]
and assume that $    |\xi_j|\le b\in\mathbb{R}$ almost surely for all $j$. Let
\[
    V_k
    :=
    \sum_{j=1}^k
    \mathbb E[\xi_j^2\mid\mathcal F_{j-1}]
\]
be the predictable quadratic variation. Then, for every \(x>0\) and \(v>0\),
\[
    \mathbb P\left(
        M_k\ge x
        \ \text{and}\
        V_k\le v
    \right)
    \le
    \exp\left\{
        -\frac{x^2}{2(v+bx/3)}
    \right\}.
\]
Consequently,
\[
    \mathbb P\left(
        |M_k|\ge x
        \ \text{and}\
        V_k\le v
    \right)
    \le
    2\exp\left\{
        -\frac{x^2}{2(v+bx/3)}
    \right\}.
\]
\end{lemma}

\begin{lemma}
\label{Law_of_large_numbers_lemma}
Let $(X_i, Y_i)_{i=1}^{\infty}$ be iid continuous random vectors with support $\mathcal{X}\times [0, 1]$ and continuous joint density. Let $B_n\subset \mathcal{X}$ be decreasing balls such that $\cap_{n=1}^\infty B_n= x_0\in\overline{\mathcal{X}}$ and  $n\mathbb{P}(X_i\in B_n)\to\infty$. 

Then, 
$$
\frac{1}{|\tilde{S}_n|}\sum_{i\in \tilde{S}_n}Y_i \overset{P}{\to} \mathbb{E}[Y_1\mid X_1=x_0], \,\,\,\,as \,\,n\to\infty, 
$$
where $\tilde{S}_n = \{i \in \{1, ..., n\}: X_i\in B_n\}$. 
\end{lemma}

\begin{proof}
\label{proof_of_thm2v2}

Let $(X, Y) := (X_1, Y_1)$, $p_n:=P(X\in B_n)$ and $m_n:=\mathbb{E}[Y\mid X\in B_n]$. Note that $\lim_{n\to\infty}np_n=\infty$,  $m_n \in [0,1]$ and that $\lim_{n\to\infty}m_n = m:=\mathbb{E}[Y\mid X=x_0]$ from the assumption of continuous joint density. 

Define $Z_{i,n} = 1_{\{X_i\in B_n\}}$. Then $\{Z_{i,n}\}_{i=1}^n$ are i.i.d. and $|\tilde{S}_n|=\sum_{i=1}^n Z_{i,n}$. Let

\[
M_n =
\begin{cases}
\frac{1}{|\tilde{S}_n|}\sum_{i\in \tilde{S}_n}Y_i & \text{if } |\tilde{S}_n| > 0 \\
0 & \text{otherwise}
\end{cases}
\]

\textbf{Claim:} $\lim_{n\to\infty}  P[|\tilde{S}_n|\leq \frac{1}{2}np_n]=0$

\textbf{Proof of the claim: } 

We have
\begin{align*}
P[|\tilde{S}_n|\leq \frac{1}{2}np_n] & = P[np_n - |\tilde{S}_n|\geq \frac{1}{2}np_n] \\&
\leq P[\big||\tilde{S}_n|-np_n\big| \geq \frac{1}{2}np_n]\\&
=P\left[\big|\sum_{i=1}^n(Z_{i,n}-p_n)\big|\geq \frac{1}{2}np_n\right]\\
&\overset{}{\leq} \frac{np_n(1-p_n)}{\frac{1}{4}(np_n)^2}\to 0 ,\,\,\,\,\,as\,\,n\to\infty,
\end{align*}
where we used the Chebyshev inequality in the last step. 

\textbf{Final proof: }
Fix \(\epsilon>0\) and find \(n_0\in\mathbb N\) such that
\(|m_n-m|\leq \epsilon/2\) for all \(n\geq n_0\). Then, for \(n\geq n_0\),
\begin{align*}
P(|M_n-m|\geq \epsilon)
&\leq P(|M_n-m_n|\geq \epsilon/2)\\
&\leq \frac{4}{\epsilon^2}\mathbb{E}[(M_n-m_n)^2] \qquad \,\text{(Markov inequality)}\\
&=\frac{4}{\epsilon^2}\sum_{k=0}^n
\mathbb{E}[(M_n-m_n)^2\mid |\tilde{S}_n|=k]P(|\tilde{S}_n|=k)\\
&=\frac{4}{\epsilon^2}\sum_{k\leq np_n/2}
\mathbb{E}[(M_n-m_n)^2\mid |\tilde{S}_n|=k]P(|\tilde{S}_n|=k)\\
&\quad+
\frac{4}{\epsilon^2}\sum_{np_n/2<k\leq n}
\mathbb{E}[(M_n-m_n)^2\mid |\tilde{S}_n|=k]P(|\tilde{S}_n|=k)\\
&\overset{(a)}{\leq}
\frac{4}{\epsilon^2}P(|\tilde{S}_n|\leq np_n/2)
+
\frac{4}{\epsilon^2}\sum_{np_n/2<k\leq n}
\mathbb{E}[(M_n-m_n)^2\mid |\tilde{S}_n|=k]P(|\tilde{S}_n|=k)\\
&\overset{(b)}{\leq}
\frac{4}{\epsilon^2}P(|\tilde{S}_n|\leq np_n/2)
+
\frac{4}{\epsilon^2}\sum_{np_n/2<k\leq n}
\frac{\operatorname{Var}(Y\mid X\in B_n)}{k}P(|\tilde{S}_n|=k)\\
&\overset{(c)}{\leq}
\frac{4}{\epsilon^2}P(|\tilde{S}_n|\leq np_n/2)
+
\frac{4}{\epsilon^2}\sum_{np_n/2<k\leq n}
\frac{1}{k}P(|\tilde{S}_n|=k)\\
&\overset{(d)}{\leq}
\frac{4}{\epsilon^2}P(|\tilde{S}_n|\leq np_n/2)
+
\frac{4}{\epsilon^2}\frac{2}{np_n}
P(|\tilde{S}_n|>np_n/2)\\
&\leq
\frac{4}{\epsilon^2}P(|\tilde{S}_n|\leq np_n/2)
+
\frac{8}{\epsilon^2np_n}
\to 0,
\end{align*}
where
\begin{itemize}
    \item (a) follows from the fact that  \(M_n,m_n\in[0,1]\) and \((M_n-m_n)^2\leq 1\). Therefore $\sum_{k\leq np_n/2}
\mathbb{E}[(M_n-m_n)^2\mid |\tilde{S}_n|=k]P(|\tilde{S}_n|=k)\leq \sum_{k\leq np_n/2}
1[|\tilde{S}_n|=k ]P(|\tilde{S}_n|=k) = P(|\tilde{S}_n|\leq np_n/2)$.
    \item (b) follows from the fact that \(M_n\) is an average of
     \(k\) iid random variables with distribution \(Y\mid X\in B_n\) and $m_n$ is its expectation.
    \item (c) follows from the fact that $Y\in[0,1]$
    \item \((d)\) follows from \(k>np_n/2\), which implies \(1/k\leq 2/(np_n)\). 
\end{itemize}
Therefore, $M_n$ converges to $m$ in probability, concluding the proof.
\end{proof}

\begin{lemma}
    \label{ergodic_result}
Let $(\textbf{X}_i, Y_i)_{i=1}^{\infty}$ be a stationary ergodic stochastic process with $\mathbb{E}|Y|<\infty$. Let $D\subseteq\mathbb{R}^d$  be a measurable set with $P(\textbf{X}_1\in D)>0$. Let $\tilde{S}_n=\{t\in\{1, \dots, n\}: \textbf{X}_t\in D\}$. Then 
\begin{equation*}\label{eq897651}
    \frac{1}{|\tilde{S}_n|} \sum_{i\leq n: i\in \tilde{S}_n}Y_i \toverset{a.s.}{\to} \mathbb{E}[Y_1\mid \textbf{X}_1\in D], \,\,\,\,as \,\,n\to\infty. 
\end{equation*}
\end{lemma}
\begin{proof}
    The proof is a direct consequence of the weak law of large numbers for ergodic processes (Birkhoff’s Ergodic Theorem, see Proposition 4.3 in \cite{krengel1985ergodic} or  \citep{birkhoff1931proof} or chapter 4 in \cite{pene2022stochastic}). Let $\tilde{X}_i := 1[\textbf{X}_i\in D]$ and let $W_i:=Y_i\tilde{X}_i$. Note that $W_i$ is stationary and ergodic (Proposition 4.3 in \cite{krengel1985ergodic}). Now, (\ref{eq897651}) reads as 
    $$
    \frac{1}{|\tilde{S}_n|} \sum_{i\leq n: i\in \tilde{S}_n}Y_i =    \bigg( \frac{n}{\sum_{i\leq n}\tilde{X}_i}\bigg)   \bigg( \frac{1}{n} \sum_{i\leq n}W_i    \bigg) . 
    $$
The first part converges a.s. to $1/P(X_1\in D)$ from ergodicity of $\tilde{X}_i$ and the second part converges a.s. to $\mathbb{E}[Y_1 1[X_1\in D]]$     from ergodicity of $W_i$. Combining these results with Slutsky theorem, we get that   $\bigg(   \frac{1}{|\tilde{S}_n|} \sum_{i\leq n: i\in \tilde{S}_n}Y_i    \bigg)$ converges almost surely to $\frac{1}{P(X_1\in D)}\mathbb{E}[Y_1 1[X_1\in D]] = \mathbb{E}[Y_1 \mid X_1\in D]$, what we wanted to show. 
\end{proof}

\begin{customthm}{\ref{Theorem2v2}}
Consider a data-generating process as in Definition~\ref{Definition_str}. 
Assume that the process \((\mathbf X,\mathbf Y,\mathbf Z)\) is stationary 
and ergodic. Assume that the relevant finite-dimensional distributions are 
absolutely continuous with respect to Lebesgue measure and have continuous 
densities.  Assume that the structural function \(h_Y\) satisfies 
Assumption~\ref{AssumptionA1} and is continuously differentiable in 
\((y,\mathbf z)\) on a neighbourhood of \((y_0,\mathbf z_0)\), with gradient 
uniformly bounded on that neighbourhood in a sense that on a neighbourhood \(U\) of \((y_0, \textbf{z}_0)\), for \(\varepsilon^Y\)-almost every \(e\), there exist constants \(C(e)<\infty\) and \(x_0(e)<\infty\) such that
\[
    \sup_{\substack{x\ge x_0(e)\\ w\in U}}
    \|\nabla_w h_Y(x,w,e)\|_1
    \le C(e).
\]
Then, the estimator \(\hat{\Gamma}_{\textbf{X}\to\textbf{Y}\mid \textbf{Z}}\) defined in equation \eqref{Gamma_hat_equation} with \(S \equiv S_2\),  is consistent in the sense that   \begin{equation*}
   \hat{\Gamma}_{\textbf{X}\to\textbf{Y}\mid \textbf{Z}} \overset{P}{\to} \Gamma_{\textbf{X}\to\textbf{Y}\mid\mathcal{C}_0}, \quad \text{ as }n\to\infty,
  \end{equation*}
where $\Gamma_{\textbf{X}\to\textbf{Y}\mid\mathcal{C}_0} = \lim_{v\to\infty} \mathbb{E}[ F(Y_{t+1})\mid X_t>v,Y_t = y_0, \textbf{Z}_t = \textbf{z}_0]$, provided that the limit exists.

Define second-order assumptions: for \(
    m(x,y,z):=
    \mathbb E\!\left[
        F\{h_Y(x,y, z,\varepsilon^Y_{t+1})\}
    \right],
\) we assume
\[
    V_n:=\sum_{t\in S_2}    \operatorname{Var}\!\left(
        F\{h_Y(X_t,Y_t,\mathbf Z_t,\varepsilon^Y_{t+1})\}
    \right)\overset{P}{\longrightarrow}\infty, \qquad \text{ and }
\]
\[
    \frac{\sum_{t\in S_2}\{m(X_t,Y_t,\mathbf Z_t)-\Gamma_{\mathbf X\to\mathbf Y\mid\mathcal C_0}\}}
         {\sqrt{V_n}}
    \overset{P}{\longrightarrow}0,
    \qquad
    \frac{\sum_{t\in S_2}\{m(X_t,Y_t,\mathbf Z_t)-\Gamma_{\mathbf X\to\mathbf Y\mid\mathcal C_0}\}^2}
         {V_n}
    \overset{P}{\longrightarrow}0.
\]
Then
\[
    \frac{\sqrt{|S_2|}\,
    \left(
        \widehat\Gamma_{\mathbf X\to\mathbf Y\mid\mathbf Z}
        -
        \Gamma_{\mathbf X\to\mathbf Y\mid\mathcal C_0}
    \right)}
    {\widehat\sigma_n}
    \overset{d}{\to} N(0,1), \quad \text{ where }\quad
        \widehat\sigma_n^2
    :=
    \frac1{|S_2|}
    \sum_{t\in S_2}
    \left\{
        F(Y_{t+1})
        -
        \widehat\Gamma_{\mathbf X\to\mathbf Y\mid\mathbf Z}
    \right\}^2.
\]
\end{customthm}

\begin{proof}
\textit{\textbf{Idea of the proof: }If $\tau_X$ an $r$ were fixed near-limit constants, then for set $D =(\tau_x, \infty)\times B_{(y_0, \textbf{z}_0)(r)}$,}
\begin{equation*}
\begin{split}
      \hat{\Gamma}_{\textbf{X}\to\textbf{Y}\mid\mathcal{C}} &=\frac{1}{|S_2|}\sum_{\substack{t\in\{1, \dots, n\}: \\ (x_t, y_t, \textbf{z}_t)\in D }}  F(Y_{t+1}) \\&
    \overset{P}{\to} \mathbb{E}[ F(Y_{t+1})\mid (X_t, Y_t, \textbf{Z}_t)\in D] \qquad \qquad \,as\,\,n\to\infty\\&=  
     \mathbb{E}[ F(Y_{t+1})\mid X_t>\tau_X,(Y_t ,\textbf{Z}_t) \in B_{(y_0, \textbf{z}_0)}(r)]
    \\&
    \approx \lim_{v\to\infty} \mathbb{E}[ F(Y_{t+1})\mid X_t>v,Y_t = y_0, \textbf{Z}_t = \textbf{z}_0] = \Gamma_{\textbf{X}\to\textbf{Y}\mid\mathcal{C}_0}, 
\end{split}
\end{equation*}
\textit{where the convergence follows directly from the weak law of large numbers for ergodic processes (Lemma~\ref{ergodic_result} with notation $Y_t = F(Y_t)$ and $\textbf{X}_t = (X_{t-1}, Y_{t-1}, Z_{t-1})$). What remains is to show that the same argument holds also if $\tau_X$ and $r$ are non-fixed. This is made precise by Freedman's inequality and mean-value theorem. Asymptotic normality is a consequence of the martingale central limit theorem.}

\textbf{Proof of consistency. } Let \(\tilde W_t=(Y_t,\mathbf Z_t)\), \(\tilde w_0=(y_0,\mathbf z_0)\),
\(I_{n,t}:=\mathbf 1\{X_t\ge \tau_n^X,\; \tilde W_t\in B_{\tilde w_0}(r_n)\}\), and
\(N_n:=\sum_{t=1}^{n-1}I_{n,t}=|S_2|\). By the definition of \(S_2\),
\(N_n\to\infty\) in probability. We work on the event \(\{N_n>0\}\), whose
probability tends to one.

Let \(\mathcal F_t:=\sigma(X_s,Y_s,\mathbf Z_s:s\le t)\). By the structural equation,
\(Y_{t+1}=h_Y(X_t,\tilde W_t,\varepsilon^Y_{t+1})\), with
\(\varepsilon^Y_{t+1}\perp\!\!\!\perp\mathcal F_t\). Define
\[
    m(x,\tilde w):=
    \mathbb E\!\left[
        F\{h_Y(x,\tilde w,\varepsilon^Y_{t+1})\}
    \right].
\]
Then \(\mathbb E[F(Y_{t+1})\mid\mathcal F_t]=m(X_t,\tilde W_t)\). By definition, we rewrite

\[
\hat{\Gamma}_{\mathbf X\to\mathbf Y\mid\mathbf Z}
=
\underbrace{
\frac{1}{N_n}\sum_{t=1}^{n-1} I_{n,t}m(X_t,\tilde W_t)
}_{=:A_n}
+
\underbrace{
\frac{1}{N_n}\sum_{t=1}^{n-1} I_{n,t}
\bigl[F(Y_{t+1})-m(X_t,\tilde W_t)\bigr]
}_{=:C_n}.
\]

\textbf{Second term $C_n$.} We show that $C_n$ is \(o_P(1)\).  Put
\(\xi_{t+1}:=F(Y_{t+1})-m(X_t,\tilde W_t)\). Then
\(\mathbb E[\xi_{t+1}\mid\mathcal F_t]=0\) and \(|\xi_{t+1}|\le 1\). Since
\(I_{n,t}\) is \(\mathcal F_t\)-measurable,
\(M_n:=\sum_{t=1}^{n-1}I_{n,t}\xi_{t+1}\) is a martingale-difference sum with bounded
increments and predictable quadratic variation bounded by \(N_n\). Therefore, by Freedman's inequality, for every \(\delta>0\) and every
\(R\ge1\),
\[
    \mathbb P\left(\left|\frac{M_n}{N_n}\right|>\delta\right)
    \le
    \mathbb P(N_n<R)
    +
    \sum_{k=R}^{\infty}2\exp\{-c_\delta k\},
\]
for some constant \(c_\delta>0\). Indeed, on the event \(\{N_n=k\}\), the
predictable quadratic variation of \(M_n\) is at most \(k\), since the
increments are bounded and only \(k\) summands are selected. Freedman's
inequality therefore bounds the probability of \(|M_n|>\delta k\) on this
event by \(2\exp\{-c_\delta k\}\), and summing over all \(k\ge R\) gives the
displayed inequality. Letting first \(n\to\infty\) and then \(R\to\infty\)
gives \(M_n/N_n=o_P(1)\).

\textbf{First term $A_n$.} On \(\{N_n>0\}\), all selected observations satisfy \(X_t\ge\tau_n^X\) and
\(\tilde W_t\in B_{\tilde w_0}(r_n)\). Therefore
\[
    |A_n-\Gamma_{\mathbf X\to\mathbf Y\mid\mathcal C_0}|
    \le
    \sup_{\substack{x\ge\tau_n^X\\ \tilde w\in B_{\tilde w_0}(r_n)}}
    \left|m(x,\tilde w)-\Gamma_{\mathbf X\to\mathbf Y\mid\mathcal C_0}\right|.
\]
We now show that this supremum converges to zero. 

\textbf{First term $A_n$, case when \(h_Y\) is constant in \(x\).} If \(h_Y\) is constant in \(x\), then
\(m(x,\tilde w)=m_0(\tilde w)\), where
\(m_0(\tilde w):=\mathbb E[F\{h_Y(\tilde w,\varepsilon^Y_{t+1})\}]\). Since \(h_Y\) is
continuous in \(\tilde w\) at \(\tilde w_0\), and \(F\) is continuous and bounded, dominated
convergence gives \(m_0(\tilde w)\to m_0(\tilde w_0)\) as \(\tilde w\to \tilde w_0\). Hence
\[
    \sup_{\tilde w\in B_{\tilde w_0}(r_n)}
    |m_0(\tilde w)-m_0(\tilde w_0)|
    \to0.
\]
Moreover, in this case
\(\Gamma_{\mathbf X\to\mathbf Y\mid\mathcal C_0}=m_0(\tilde w_0)\), so \(A_n\to
\Gamma_{\mathbf X\to\mathbf Y\mid\mathcal C_0}\) in probability.

\textbf{First term $A_n$, case when \(h_Y\) is not constant in \(x\).}   By
Assumption~\ref{AssumptionA1}, \(h_Y(x,\tilde w_0,e)\to\infty\) as \(x\to\infty\), for
almost every \(e\). By the bounded-gradient assumption, for almost every \(e\), there exist
a neighbourhood \(U\) of \(\tilde w_0\), a constant \(C(e)<\infty\), and \(x_0(e)<\infty\),
such that
\[
    \sup_{\substack{x\ge x_0(e)\\ \tilde w\in U}}
    \|\nabla_{\tilde w} h_Y(x,\tilde w,e)\|_1\le C(e).
\]
Since \(r_n\to0\), eventually \(B_{\tilde w_0}(r_n)\subset U\). Thus, by the mean-value
theorem, for \(x\ge\tau_n^X\) and \(\tilde w\in B_{\tilde w_0}(r_n)\),
\[
    h_Y(x,\tilde w,e)
    \ge
    h_Y(x,\tilde w_0,e)-C(e)r_n.
\]
Because \(\tau_n^X\to\infty\) and \(r_n\to0\), it follows that
\[
    \inf_{\substack{x\ge\tau_n^X\\ \tilde w\in B_{\tilde w_0}(r_n)}}
    h_Y(x,\tilde w,e)
    \longrightarrow \infty
\]
for almost every \(e\). Since \(F(u)\to1\) as \(u\to\infty\), dominated convergence gives
\[
    \sup_{\substack{x\ge\tau_n^X\\ \tilde w\in B_{\tilde w_0}(r_n)}}
    |m(x,\tilde w)-1|
    \le
    \mathbb E\!\left[
        1-
        F\!\left(
            \inf_{\substack{x\ge\tau_n^X\\ \tilde w\in B_{\tilde w_0}(r_n)}}
            h_Y(x,\tilde w,\varepsilon^Y_{t+1})
        \right)
    \right]
    \longrightarrow 0.
\]
The same argument with \(\tilde w=\tilde w_0\) and \(v\to\infty\) shows that
\(\Gamma_{\mathbf X\to\mathbf Y\mid\mathcal C_0}=1\). Hence \(A_n\to
\Gamma_{\mathbf X\to\mathbf Y\mid\mathcal C_0}\) in probability also in the non-constant
case.

\textbf{Finish.} Combining \(A_n\to\Gamma_{\mathbf X\to\mathbf Y\mid\mathcal C_0}\) with \(M_n/N_n\to 0\) yields \(
    \hat{\Gamma}_{\mathbf X\to\mathbf Y\mid\mathbf Z}
    \overset{P}{\longrightarrow}
    \Gamma_{\mathbf X\to\mathbf Y\mid\mathcal C_0},
\) what we wanted to prove.

\textbf{Proof of CLT. }Put 
\[
v_t:= \operatorname{Var}\!\left( F\{h_Y(x,y,z,\varepsilon^Y_{t+1})\} \right),
    \qquad
    V_n:=\sum_{t\in S_2}v_t .
\]
\[
    b_t:=m(X_t,Y_t,\mathbf Z_t)-\Gamma_0,     \qquad
   \xi_{t+1}:=F(Y_{t+1})-m(X_t,Y_t,\mathbf Z_t).
\]The martingale differences \(I_{n,t}\xi_{t+1}\) are bounded and have predictable quadratic
variation \(V_n\). Since \(V_n\overset{P}{\to}\infty\), the martingale central limit theorem
gives
\[
    \frac{\sum_{t\in S_2}\xi_{t+1}}{\sqrt{V_n}}
    \overset{d}{\to}
    N(0,1).
\]
By the first assumed bias condition,
\[
    \frac{\sum_{t\in S_2}\{F(Y_{t+1})-\Gamma_0\}}{\sqrt{V_n}}
    \overset{d}{\to}
    N(0,1).
\]
It remains to replace \(V_n\) by the empirical variance. Since
\(\xi_{t+1}^2-v_t\) is again a bounded martingale difference, we use again the Freedman's inequality, which gives us
\[
    \frac{\sum_{t\in S_2}(\xi_{t+1}^2-v_t)}{V_n}
    \overset{P}{\to}0.
\]
Together with the second bias condition and Cauchy's inequality, this implies
\[
    \frac{\sum_{t\in S_2}\{F(Y_{t+1})-\Gamma_0\}^2}{V_n}
    \overset{P}{\to}1.
\]
Moreover,
\[
    N_n\widehat\sigma_n^2
    =
    \sum_{t\in S_2}\{F(Y_{t+1})-\Gamma_0\}^2
    -
    N_n\{\widehat\Gamma_{\mathbf X\to\mathbf Y\mid\mathbf Z}-\Gamma_0\}^2,
\]
and the second term is \(o_P(V_n)\). Hence
\[
    \frac{N_n\widehat\sigma_n^2}{V_n}\overset{P}{\to}1.
\]
Consequently, by Slutsky's theorem, we conclude that
\[
    \frac{\sqrt{N_n}
    \left(
        \widehat\Gamma_{\mathbf X\to\mathbf Y\mid\mathbf Z}
        -
        \Gamma_0
    \right)}
    {\widehat\sigma_n}
    =
    \frac{
        \sum_{t\in S_2}\{F(Y_{t+1})-\Gamma_0\}
    }{\sqrt{V_n}}
    \left(
        \frac{N_n\widehat\sigma_n^2}{V_n}
    \right)^{-1/2}
    \overset{d}{\to}
    N(0,1).
\]
\end{proof}

\newpage
\subsection{Proof of Theorem \ref{Theorem3}} \label{proof_of_thm3}

\begin{customthm}{\ref{Theorem3}}
Consider a time series following the non-negative SRE model \eqref{sre2} that satisfies Assumptions (B1), (B2), and (B4). Then, the estimator \(\hat{\Gamma}_{\textbf{X}\to\textbf{Y}\mid \textbf{Z}}\) defined in equation (\ref{Gamma_hat_equation}), with \(S \equiv S_1\) satisfies
\begin{equation}
\tag{\ref{consistency_theorem2}}
\hat{\Gamma}_{\textbf{X}\to\textbf{Y}\mid\textbf{Z}} \overset{P}{\to} 1 \text{   as }n\to\infty\iff  \Gamma_{\textbf{X}\to\textbf{Y}\mid\mathcal{C}} = 1.
\end{equation}
\end{customthm}
\begin{proof}
First, we add some notation:
\begin{itemize}
    \item Recall that $Y_{t+1}=A_{1,t+1}^yZ_t + A_{2,t+1}^yX_t + A_{3,t+1}^yY_t+B_{t+1}^y$, and $\varepsilon_t^Y:=(A^y_{1,t},A^y_{2,t},A^y_{3,t},B^y_t)^\top.$
    \item $D_t:=\left\{\begin{pmatrix}Y_t\\ \mathbf Z_t\end{pmatrix}\le\boldsymbol{\tau}\right\},$
    \item $  I_{n,t}:=\mathbf 1\{X_t\ge\tau_n^X,\ D_t\},$
    \item $  N_n:=|S_1|=\sum_{t=1}^n I_{n,t}.$ By assumption, \(N_n\overset{P}{\to}\infty\).

\end{itemize}

\textbf{Claim 1} under (B2) and (B4),
\[
    \Gamma_{\mathbf X\to\mathbf Y\mid\mathcal C}=1
    \iff
    A^y_{2,t+1}>0\ \text{a.s.}.
\]
This has been proven in Theorem~\ref{Theorem1}. 

\textbf{Claim 2.} For every bounded measurable function \(g\) of \(\varepsilon^y_{t+1}\),
\[
    \frac1{N_n}\sum_{t\in S_1}g(\varepsilon^y_{t+1})
    \overset{P}{\to}
    \mathbb{E}[g(\varepsilon^y_{t+1})].
\]

\textbf{Proof of Claim 2.}
By assumption, there exists $C<\infty$ such that
\(|g(\varepsilon^y_{t+1})-\mathbb{E}[g(\varepsilon^y_{t+1})]|\le C\).  By (B2), \(\varepsilon^y_{t+1}\) is independent of the past up to $t$, hence 
\[
    M_n(g):=\sum_{t=1}^n I_{n,t}\{g(\varepsilon^y_{t+1})-\mathbb{E}[g(\varepsilon^y_{t+1})]\}
\]
is a bounded martingale-difference sum. Therefore, we use the conditional Hoeffding bound \citep{hoeffding1963probability,azuma1967weighted}: for every \(\delta>0\) and every integer \(R\ge1\),
\[
    P(|M_n(g)|>\delta N_n)
    \le
    P(N_n<R)+2\exp\{-\delta^2R/(2C^2)\}.
\]
Since \(N_n\overset{P}{\to}\infty\), first letting \(n\to\infty\) and then
\(R\to\infty\) yields \(M_n(g)/N_n\overset{P}{\to}0\), which proves the claim.

\textbf{Proof of theorem, implication} ''$\impliedby$'' Assume
\(\Gamma_{\mathbf X\to\mathbf Y\mid\mathcal C}=1\). Then
\(A^y_{2,t+1}>0\) a.s. Since the SRE is non-negative, for every fixed \(v\) and all large
\(n\), every \(t\in S_1\) satisfies \(X_t\ge\tau_n^X\ge v\), and hence
\[
    F(Y_{t+1})
    \ge
    L_t(v):=F(A^y_{2,t+1}v+B^y_{t+1}).
\]
Moreover, \(\mathbb{E}[L_t(v)]\to1\) as \(v\to\infty\), by dominated convergence. Given
\(\varepsilon>0\), choose \(v\) such that \(\mathbb{E}[L_t(v)]>1-\varepsilon\). Using Claim 2,
\[
    \hat{\Gamma}_{\mathbf X\to\mathbf Y\mid\mathbf Z}
    \ge
    \frac1{N_n}\sum_{t\in S_1}L_t(v)
    \overset{P}{\to}
    \mathbb{E}[L_t(v)]>1-\varepsilon .
\]
Since \(0\le\hat{\Gamma}_{\mathbf X\to\mathbf Y\mid\mathbf Z}\le1\) and \(\varepsilon>0\) is arbitrary, we obtain
\(
    \hat{\Gamma}_{\mathbf X\to\mathbf Y\mid\mathbf Z}
    \overset{P}{\to}1 .
\)

\textbf{Proof of theorem, implication} ''$\implies$''  For a contradiction, suppose that
\(\Gamma_{\mathbf X\to\mathbf Y\mid\mathcal C}<1\). Then \(A^y_{2,t+1}=0\) a.s. For
\(t\in S_1\), using \(Y_t\le\tau_Y\), \(\mathbf Z_t\le\boldsymbol{\tau}_Z\), and
non-negativity,
\[
    F(Y_{t+1})
    \le
    K_t:=
    F(A^y_{1,t+1}\boldsymbol{\tau}_Z
      +A^y_{3,t+1}\tau_Y
      +B^y_{t+1}).
\]
The variable \(K_t\) is a bounded measurable function of \(\varepsilon^y_{t+1}\). Using
Claim 2,
\[
    \frac1{N_n}\sum_{t\in S_1}K_t
    \overset{P}{\to}
    \mathbb{E}[K_t].
\]
Since \(F(x)<1\) for every finite \(x\), \(K_t<1\) a.s., and hence \(\mathbb{E}[K_t]<1\). Hence
\[
    P\left(
        \hat{\Gamma}_{\mathbf X\to\mathbf Y\mid\mathbf Z}\le \frac{1+\mathbb{E}[K_t]}{2} 
    \right)\to1 .
\]
Thus \(\hat{\Gamma}_{\mathbf X\to\mathbf Y\mid\mathbf Z}\not\overset{P}{\to}1\). This completes the proof.
\end{proof}

\begin{corollary}
\label{consequence_1+gamma/2}
Let the assumptions of Theorem~\ref{Theorem3} hold and suppose that
\(\Gamma_{X\to Y\mid C}<1\). Write 
\[
\widehat\Gamma_n(\tau)
:=
\frac1{|S_1|}\sum_{t\in S_1}F(Y_{t+1}),
\qquad
\widehat\Gamma^{\mathrm{baseline}}_n(\tau)
:=
\frac1{|\widetilde S_1|}
\sum_{t\in \widetilde S_1}F(Y_{t+1}), 
\]
where $\widetilde S_1:=\left\{t\leq n: \begin{pmatrix}Y_t\\ \mathbf Z_t\end{pmatrix}\le \tau\right\}.$ Then there exists \(\tau^0\) such that, for every admissible \(\tau\leq \tau^0\):
\[
P\left(
\widehat\Gamma_n(\tau)
\leq
\frac{1+\widehat\Gamma^{\mathrm{baseline}}_n(\tau)}{2}
\right)\to 1 .
\]
\end{corollary}
\begin{proof}
Recall the notation
\[
D_t(\tau):=\left\{\begin{pmatrix}Y_t\\ \mathbf Z_t\end{pmatrix}\le \tau\right\},
\qquad
S_1:=\{t\le n:X_t\ge \tau_n^X,\ D_t(\tau)\},
\qquad
\widetilde S_1:=\{t\le n:D_t(\tau)\}.
\]
By the characterization used in the proof of Theorem~\ref{Theorem3}, under
(B2) and (B4),
\[
\Gamma_{X\to Y\mid C}<1
\quad\implies \quad
A^y_{2,t+1}=0 \ \text{a.s.}
\]
Therefore, for every  \(t\in D_t(\tau)\), and using (B4), we can write
\[
Y_{t+1}=A^y_{1,t+1}\mathbf Z_t+A^y_{3,t+1}Y_t+B^y_{t+1}\le
A^y_{1,t+1}\tau_Z+A^y_{3,t+1}\tau_Y+B^y_{t+1}.
\]
Since \(F\) is nondecreasing, we always have
\[
\widehat\Gamma_n
\le
U_n(\tau)
:=
\frac1{|S_1|}\sum_{t\in S_1}
K_{t+1}(\tau), \quad \text{ where } \quad K_{t+1}(\tau)
:=
F\!\left(
A^y_{1,t+1}\tau_Z+A^y_{3,t+1}\tau_Y+B^y_{t+1}
\right).
\]
The random variable \(K_{t+1}(\tau)\) is a bounded measurable function of
\(\varepsilon^y_{t+1}\). Thus, by Claim 2 used in the
proof of Theorem~\ref{Theorem3},
\[
U_n(\tau)\overset{P}{\longrightarrow}
e(\tau)
:=
\mathbb E\!\left[
F\!\left(
A^y_{1,t}\tau_Z+A^y_{3,t}\tau_Y+B^y_t
\right)
\right].
\]
It remains to choose \(\tau\) so that \(e(\tau)\) is strictly below the desired midpoint. 

Let \(s_Y:=\inf\operatorname{supp}(Y_t)\) and
\(s_Z:=\inf\operatorname{supp}(\mathbf Z_t)\). In the non-negative SRE case,
\(s_Y\) and \(s_Z\) are finite. Define
\[
e_0
:=
\mathbb E\!\left[
F\!\left(
A^y_{1,t}s_Z+A^y_{3,t}s_Y+B^y_t
\right)
\right].
\]
Since \(F(x)<1\) for every finite \(x\), we have \(e_0<1\). Moreover, by
right-continuity and monotonicity of \(F\), together with dominated convergence,
\[
e(\tau)\longrightarrow e_0
\qquad\text{as}\qquad
\tau\downarrow (s_Y,s_Z).
\]
Hence we may choose \(\tau^0\) sufficiently close to \((s_Y,s_Z)\) such that for every admissible
\(\tau\le \tau^0\),
\[
e(\tau)\le e(\tau^0)<\frac{1+e_0}{2}.
\]
By the ergodic theorem,
\[
\widehat\Gamma^{\mathrm{baseline}}_n(\tau)
=
\frac{\sum_{t=1}^n \mathbf 1\{D_t(\tau)\}F(Y_{t+1})}
     {\sum_{t=1}^n \mathbf 1\{D_t(\tau)\}}
\overset{P}{\longrightarrow}
\gamma_{\mathrm{base}}(\tau):=
\mathbb E\{F(Y_{t+1})\mid D_t(\tau)\}
\]
for every admissible \(\tau\). Again using non-negativity and monotonicity of \(F\),
\[
F(Y_{t+1})
\ge
F\!\left(
A^y_{1,t+1}s_Z+A^y_{3,t+1}s_Y+B^y_{t+1}
\right).
\]
Since \(\varepsilon^y_{t+1}\) is independent of the past, the right-hand side is
independent of \(D_t(\tau)\). Consequently, \(
\gamma_{\mathrm{base}}(\tau)\ge e_0.
\) Thus, for every admissible \(\tau\le\tau^0\),
\[
e(\tau)
<
\frac{1+e_0}{2}
\le
\frac{1+\gamma_{\mathrm{base}}(\tau)}{2}.
\]
Combining this with the convergence and
\(U_n(\tau)\overset{P}{\to}e(\tau)\) and the strict population inequality above gives
\[
P\left(
U_n(\tau)\le
\frac{1+\widehat\Gamma^{\mathrm{baseline}}_n(\tau)}{2}
\right)\to 1.
\]
Since \(\widehat\Gamma_n(\tau)\le U_n(\tau)\), we conclude that also
\(
P\left(
\widehat\Gamma_n(\tau)
\le
\frac{1+\widehat\Gamma^{\mathrm{baseline}}_n(\tau)}{2}
\right)\to 1.
\)
This proves the claim.
\end{proof}

\newpage
\subsection{Proof of Lemma \ref{Classification_lemma}}
\label{proof_of_classicication_lemma}

\begin{lemma}
\label{lemma:S2_baseline_consistency}
Let the assumptions of Theorem~\ref{Theorem2v2} hold. Let
\[
    \widetilde S_{2,n}
    :=
    \left\{
        t\in\{1,\ldots,n-1\}:
        (Y_t,\mathbf Z_t)\in B_{(y_0,\mathbf z_0)}(r_n)
    \right\},
\]
Then 
\[
    \widehat\Gamma^{\baseline}_{2,n}
    :=
    \frac{1}{|\widetilde S_{2,n}|}
    \sum_{t\in\widetilde S_{2,n}} F(Y_{t+1})  \overset{P}{\longrightarrow}
    \Gamma^{\baseline}_0
    :=
    \mathbb E[
        F(Y_{t+1})
        \mid
        Y_t=y_0,\mathbf Z_t=\mathbf z_0
    ].
\]
\end{lemma}

\begin{proof}
The proof can follow directly from following line by line the proof of Theorem~\ref{Theorem2v2}. 

Since \(S_{2,n}\subseteq \widetilde S_{2,n}\) and since \(|S_{2,n}|\overset{P}{\to}\infty\) we also have \(
    |\widetilde S_{2,n}|\overset{P}{\to}\infty .
\)

Put \(\widetilde W_t:=(Y_t,\mathbf Z_t)\) and
\(\widetilde w_0:=(y_0,\mathbf z_0)\). By the structural equation,
\[
    Y_{t+1}
    =
    h_Y(X_t,\widetilde W_t,\varepsilon^Y_{t+1}),
    \qquad
    \varepsilon^Y_{t+1}\perp\!\!\!\perp \sigma(X_s,Y_s,\mathbf Z_s:s\leq t).
\]
Define
\[
    q(\widetilde w)
    :=
    \mathbb E[
        F(Y_{t+1})\mid \widetilde W_t=\widetilde w
    ].
\]
The continuity and density assumptions from Theorem~\ref{Theorem2v2} imply that
\(q\) is continuous at \(\widetilde w_0\). Hence
\[
    \mathbb E[
        F(Y_{t+1})
        \mid
        \widetilde W_t\in B_{\widetilde w_0}(r_n)
    ]
    \longrightarrow
    q(\widetilde w_0)
    =
    \Gamma^{\baseline}_0 .
\]
Repeating the martingale-decomposition and local-ergodic argument from the proof of
Theorem~\ref{Theorem2v2}, but without the restriction \(X_t\geq\tau_n^X\), gives
\[
    \widehat\Gamma^{\baseline}_{2,n}
    -
    \mathbb E[
        F(Y_{t+1})
        \mid
        \widetilde W_t\in B_{\widetilde w_0}(r_n)
    ]
    \overset{P}{\longrightarrow}0 .
\]
Combining the last two displays proves the claim.
\end{proof}

\begin{customlem}{\ref{Classification_lemma}}
Let the assumptions from Theorem~\ref{Theorem2v2} hold. Then
Algorithm~\ref{Algorithm1} with \(S=S_2\) is consistent; that is, the output is
correct with probability tending to one as \(n\to\infty\).

Let the assumptions from Theorem~\ref{Theorem3} hold. Then there exists
\(\boldsymbol{\tau}_0\in\mathbb R^{1+d}\) such that, for all admissible
\(\boldsymbol{\tau}\leq\boldsymbol{\tau}_0\), Algorithm~\ref{Algorithm1} with
\(S=S_1\) and hyper-parameter \(\boldsymbol{\tau}\) gives the correct output with
probability tending to one as \(n\to\infty\).
\end{customlem}

\begin{proof}
\textbf{First case} \(S=S_2\). By Theorem~\ref{Theorem2v2},
\(
    \widehat\Gamma_{2,n}
    \overset{P}{\longrightarrow}
    \Gamma_0
    :=
    \Gamma_{\mathbf X\to\mathbf Y\mid\mathcal C_0},
\)
and by Lemma~\ref{lemma:S2_baseline_consistency},
\(
    \widehat\Gamma^{\baseline}_{2,n}
    \overset{P}{\longrightarrow}
    \Gamma^{\baseline}_0
    :=
    \mathbb E[
        F(Y_{t+1})\mid Y_t=y_0,\mathbf Z_t=\mathbf z_0
    ]<0.
\)

\begin{itemize}
    \item If \(\mathbf X\overset{ext}{\to}\mathbf Y\), then \(\Gamma_0=1\). Hence
\[
    \widehat\Gamma_{2,n}
    -
    \frac{1+\widehat\Gamma^{\baseline}_{2,n}}{2}
    \overset{P}{\longrightarrow}
    \frac{1-\Gamma^{\baseline}_0}{2}
    >0.
\]
Thus Algorithm~\ref{Algorithm1} returns
\(\mathbf X\overset{ext}{\to}\mathbf Y\) with probability tending to one.

 \item If \(\mathbf X\overset{ext}{\not\to}\mathbf Y\), then \(
    \Gamma_0=\Gamma^{\baseline}_0.
\) Consequently,
\[
    \widehat\Gamma_{2,n}
    -
    \frac{1+\widehat\Gamma^{\baseline}_{2,n}}{2}
    \overset{P}{\longrightarrow}
    -\frac{1-\Gamma^{\baseline}_0}{2}
    <0.
\] Thus Algorithm~\ref{Algorithm1} returns
\(\mathbf X\overset{ext}{\not\to}\mathbf Y\) with probability tending to one.
\end{itemize}

\textbf{Second case} \(S=S_1\). The baseline estimator satisfies
\[
    \widehat\Gamma^{\baseline}_{1,n}(\boldsymbol{\tau})
    \overset{P}{\longrightarrow}
    \gamma_{\baseline}(\boldsymbol{\tau})
    :=
    \mathbb E[
        F(Y_{t+1})
        \mid
        (Y_t,\mathbf Z_t)\leq\boldsymbol{\tau}
    ]<1,
\]
by the ergodic theorem. 

\begin{itemize}
 \item If \(\mathbf X\overset{ext}{\to}\mathbf Y\), then
Theorem~\ref{Theorem3} gives
\(
    \widehat\Gamma_{1,n}(\boldsymbol{\tau})
    \overset{P}{\longrightarrow}
    1,
\)
so Algorithm~\ref{Algorithm1} returns
\(\mathbf X\overset{ext}{\to}\mathbf Y\) with probability tending to one.

 \item If \(\mathbf X\overset{ext}{\not\to}\mathbf Y\), then
Corollary~\ref{consequence_1+gamma/2} gives
\(\boldsymbol{\tau}_0\in\mathbb R^{1+d}\) such that, for every admissible
\(\boldsymbol{\tau}\leq\boldsymbol{\tau}_0\),
\[
    \mathbb P\!\left(
        \widehat\Gamma_{1,n}(\boldsymbol{\tau})
        \leq
        \frac{
            1+
            \widehat\Gamma^{\baseline}_{1,n}(\boldsymbol{\tau})
        }{2}
    \right)
    \to 1.
\]
Thus Algorithm~\ref{Algorithm1} returns
\(\mathbf X\overset{ext}{\not\to}\mathbf Y\) with probability tending to one.
\end{itemize}
\end{proof}

\subsection{Proof of Lemma \ref{lemma_path_diagram}}
\label{proof_of_lemma_path_diagram}

\begin{customlem}{\ref{lemma_path_diagram}}
Let \((\textbf{X}^1, \dots, \textbf{X}^m)\) be a collection of time series. Assume that, for each distinct pair \(i, j \in \{1, \dots, m\}\), Algorithm~\ref{Algorithm1} is consistent and that 
\begin{equation}\label{ertty}
\Gamma_{\textbf{X}^i\to\textbf{X}^j\mid\mathcal{C}} = 1 \implies  \Gamma_{\textbf{X}^i\to\textbf{X}^j\mid \emptyset} = 1. 
     \end{equation}
Note that these conditions are satisfied under the assumptions of Lemma~\ref{Classification_lemma} and Theorem~\ref{Theorem1}. Then, Algorithm~\ref{Algorithm2} is consistent, meaning that \(P({\hat{\mathcal{G}}} = \mathcal{G}) \to 1\) as \(n \to \infty\).

Furthermore, if, for each distinct pair $i, j\in{1, \dots, m}$, 
\begin{equation}\label{ertty2}\Gamma_{\textbf{X}^i\to\textbf{X}^j\mid\mathcal{C}} = 1 \iff  \Gamma_{\textbf{X}^i\to\textbf{X}^j\mid \emptyset} = 1
\end{equation}
then $P({\hat{\mathcal{G}}^{\rm P}} = \mathcal{G})\to 1$ as $n\to\infty$, and Step 2 of the algorithm is asymptotically not necessary. 

\end{customlem}

\begin{proof}
Under condition (\ref{ertty2}), we observe the following equivalences:
\begin{equation*}
\begin{split}
\Gamma_{\textbf{X}^i \to \textbf{X}^j \mid \emptyset} = 1 & \iff \Gamma_{\textbf{X}^i \to \textbf{X}^j \mid \mathcal{C}} = 1 \\
& \iff \textbf{X}^i \toverset{ext}{\to} \textbf{X}^j \\
& \iff (i, j) \in \mathcal{G}.
\end{split}
\end{equation*}
Since Algorithm~\ref{Algorithm1} is consistent, we obtain an edge \(i \to j\) in \({\hat{\mathcal{G}}^{\rm P}}\) if and only if \(\Gamma_{\textbf{X}^i \to \textbf{X}^j \mid \emptyset} = 1\) with probability approaching 1 as \(n \to \infty\). Hence, \(P({\hat{\mathcal{G}}^{\rm P}} = \mathcal{G}) \to 1\) as \(n \to \infty\).

Regarding the first statement, considering condition (\ref{ertty}), we infer that \(P({\hat{\mathcal{G}}^{\rm P}} \supseteq \mathcal{G}) \to 1\) as \(n \to \infty\) since \(\Gamma_{\textbf{X}^i \to \textbf{X}^j \mid \emptyset} = 1\) for every edge \(i \to j\) in \(\mathcal{G}\). Consequently, for \(A := \pa_{{\hat{\mathcal{G}}^{\rm P}}}(i) \cap \pa_{{\hat{\mathcal{G}}^{\rm P}}}(j)\) and \(B := \pa_{\mathcal{G}}(i) \cap \pa_{\mathcal{G}}(j)\), we have \(A \supseteq B\) with probability approaching 1 as \(n \to \infty\). Thus,
\[ \Gamma_{\textbf{X}^i \to \textbf{X}^j \mid \mathcal{C}} = 1 \iff \Gamma_{\textbf{X}^i \to \textbf{X}^j \mid \textbf{W}^A} = 1 \iff \Gamma_{\textbf{X}^i \to \textbf{X}^j \mid \textbf{W}^B} = 1. \]
Therefore, since Algorithm~\ref{Algorithm1} is consistent, an edge \(i \to j\) exists in \(\hat{\mathcal{G}}\) as \(n \to \infty\) if and only if \(\Gamma_{\textbf{X}^i \to \textbf{X}^j \mid \mathcal{C}} = 1\). This completes the proof.
\end{proof}

\newpage
\subsection{Proof of Proposition \ref{no_free_lunch}}
\label{proof_of_no_free_lunch}

\begin{customprop}{\ref{no_free_lunch}}[No-free-lunch: time-series version]
Let \(n\in\mathbb N\), let \(\alpha\in(0,1)\), and let
\(\psi_n:\mathbb R^{3n}\to\{0,1\}\) be any test of
\[
    H_0^G:\mathbf X\toverset{G}{\not\to}\mathbf Y\mid\mathbf Z .
\]
If \(\psi_n\) has level \(\alpha\) uniformly over the unrestricted Granger non-causality null, then it cannot have nontrivial finite-sample power even against the embedded alternatives \(\mathcal A_0\); that is, 
\[
    \sup_{Q\in\mathcal P_0^G} Q(\psi_n=1)\le \alpha  \quad \implies \quad       \sup_{P\in\mathcal A_0}   P(\psi_n=1)\le \alpha. 
\]
\end{customprop}

\begin{proof}
We prove the result by reducing the time-series testing problem to the
i.i.d. conditional-independence testing problem of
\citet{HARDNESS_OF_CONDITIONAL_TESTING}.

We use the following auxiliary notation: Let \(\mathcal E_0\) be the
class of absolutely continuous laws \(R\) of a triple \((A,B,C)\), and define
\[
    \mathcal N_{\mathrm{CI}}
    :=
    \{R\in\mathcal E_0:A\indep B\mid C\},
    \qquad
    \mathcal A_{\mathrm{CI}}
    :=
    \mathcal E_0\setminus\mathcal N_{\mathrm{CI}} ,
\qquad
    \mathcal A_{\mathrm{CI}}^{\perp}
    :=
    \{R\in\mathcal A_{\mathrm{CI}}: A\indep C\}.
\]

For \(R\in\mathcal E_0\), let \((A_t,C_t)_{t\in\mathbb Z}\) be i.i.d. with
law \(R_{AC}\), and write 
\[
    B_t=f_R(A_t,C_t,\eta_{t+1})
\]

for some measurable function \(f_R\) and \((\eta_t)_{t\in\mathbb Z}\)  i.i.d.
Uniform\((0,1)\), independent of \((A_t,C_t)_{t\in\mathbb Z}\). Using this notation, we have \((A_t,B_t,C_t)\sim R\). 

Define the time series
\[
    X_t=A_t,
    \qquad
    Z_t=C_t,
    \qquad
    Y_t=B_{t-1}=f_R(A_{t-1},C_{t-1},\eta_t).
\]
Let \(P_R\) denote its law. Notice that \(P_R\in\Xi_0\). Moreover,
for \(R\in\mathcal A_{\mathrm{CI}}^{\perp}\), the same construction has the
structural form required in the definition of \(\mathcal A_0\): take
\(\varepsilon_t^Z=C_t\), \(\varepsilon_t^X=A_t\),
\(\varepsilon_t^Y=\eta_t\), and \(f=f_R\). The required mutual independence of
the innovation sequences follows from \(A\indep C\) and from the independence
of \((\eta_t)\).

For every \(R\in\mathcal E_0\),
\[
    X_t=A_t,
    \qquad
    Y_{t+1}=B_t,
    \qquad
    Z_t=C_t .
\]
Let
\[
    D_t:=\sigma\big(B_{s},C_s:s<t\big),
    \qquad
    E_t:=\sigma\big(A_s:s<t\big).
\]
Then \(\mathcal C_t^{-\mathbf X}=\sigma(D_t,C_t)\),
\(\mathbf X_{\past(t)}=\sigma(E_t,A_t)\), and \((D_t,E_t)\) is independent of
\((A_t,B_t,C_t)\). Hence, conditionally on \((D_t,C_t)\), the conditional law
of \((E_t,A_t,B_t)\) factorizes as the conditional law of \(E_t\) given
\(D_t\) times the conditional law of \((A_t,B_t)\) given \(C_t\). Therefore,
conditional independence of \(Y_{t+1}=B_t\) from
\(\mathbf X_{\past(t)}=\sigma(E_t,A_t)\) given
\(\mathcal C_t^{-\mathbf X}=\sigma(D_t,C_t)\) is equivalent to conditional
independence of \(B_t\) from \(A_t\) given \(C_t\). That is,
\[
    \mathbf X\toverset{G}{\not\to}\mathbf Y\mid\mathbf Z
    \quad\text{under }P_R
    \qquad\Longleftrightarrow\qquad
    A\indep B\mid C
    \quad\text{under }R .
\]
Thus, if \(R\in\mathcal N_{\mathrm{CI}}\), then
\(P_R\in\mathcal P_0^G\). If \(R\in\mathcal A_{\mathrm{CI}}^{\perp}\), then
\(P_R\in\mathcal A_0\).

Now let \(\psi_n\) satisfy
\[
    \sup_{Q\in\mathcal P_0^G}Q(\psi_n=1)\le \alpha .
\]
From \(\psi_n\), define a test of conditional independence based on
\(n+1\) i.i.d. observations
\[
    V_i=(A_i,B_i,C_i),\qquad i=0,\ldots,n,
\]
by
\[
\begin{split}
    \varphi_{n+1}(V_0,\ldots,V_n)
    :=
    \psi_n\big(
        &(A_1,B_0,C_1),
          (A_2,B_1,C_2),
          \ldots,
          (A_n,B_{n-1},C_n)
    \big).
\end{split}
\]
The argument of \(\psi_n\) has the same distribution as
\[
    ((X_1,Y_1,Z_1),\ldots,(X_n,Y_n,Z_n))
\]
under \(P_R\). Hence, using the \(n\)-sample convention from the main text,
\[
    R^{\otimes(n+1)}(\varphi_{n+1}=1)
    =
    P_R(\psi_n=1).
\]

For every \(R\in\mathcal N_{\mathrm{CI}}\), we have
\(P_R\in\mathcal P_0^G\). Therefore
\[
    R^{\otimes(n+1)}(\varphi_{n+1}=1)
    =
    P_R(\psi_n=1)
    \le \alpha .
\]
Thus \(\varphi_{n+1}\) is a valid level-\(\alpha\) test of
\(A\indep B\mid C\) over \(\mathcal E_0\). By the no-free-lunch theorem of
\citet{HARDNESS_OF_CONDITIONAL_TESTING},
\[
    R^{\otimes(n+1)}(\varphi_{n+1}=1)
    \le \alpha
    \qquad
    \text{for every }R\in\mathcal A_{\mathrm{CI}} .
\]
In particular, the same bound holds for every
\(R\in\mathcal A_{\mathrm{CI}}^{\perp}\).

It remains only to identify these embedded alternatives with \(\mathcal A_0\).
Let \(P\in\mathcal A_0\), and define
\[
    A_t:=X_t,
    \qquad
    B_t:=Y_{t+1},
    \qquad
    C_t:=Z_t .
\]
By the definition of \(\mathcal A_0\), the triples \((A_t,B_t,C_t)\) are
i.i.d., \(A_t\indep C_t\), and the Granger-causality condition is equivalent,
by the argument above, to \(A\not\!\indep B\mid C\). Hence
\(R:=\mathcal L(A_t,B_t,C_t)\) belongs to
\(\mathcal A_{\mathrm{CI}}^{\perp}\). Moreover,
\[
    ((X_1,Y_1,Z_1),\ldots,(X_n,Y_n,Z_n))
    =
    ((A_1,B_0,C_1),\ldots,(A_n,B_{n-1},C_n)),
\]
in distribution. Therefore
\[
    P(\psi_n=1)
    =
    R^{\otimes(n+1)}(\varphi_{n+1}=1)
    \le \alpha .
\]
Since \(P\in\mathcal A_0\) was arbitrary, this proves
\[
    \sup_{P\in\mathcal A_0}P(\psi_n=1)\le \alpha .
\]
\end{proof}

\begin{customcorollary}{\ref{corollary_no_free_lunch}}
Let \(\psi_n:\mathbb R^{3n}\to\{0,1\}\) be any test of
\[
    H_0^{ext}:
    \Gamma_{\mathbf X\to\mathbf Y\mid\mathbf Z}<1, \qquad \text{ against} \qquad 
    H_1^{ext}:
    \Gamma_{\mathbf X\to\mathbf Y\mid\mathbf Z}=1.
\]
For every embedded alternative \(P\in\mathcal A_0\) holds 
\[
\sup_{Q\in\mathcal P_0^{ext}}Q(\psi_n=1)\le \alpha \qquad 
\text{ then } \qquad 
    P(\psi_n=1)\le \alpha. 
\]
\end{customcorollary}

\begin{proof}
Let \(\psi_n\) satisfy
\[
    \sup_{Q\in\mathcal P_0^{ext}}Q(\psi_n=1)\le \alpha .
\]
On the common domain on which the extremal coefficient is defined, Granger
non-causality implies extremal non-causality. Thus
\(\mathcal P_0^G\subseteq\mathcal P_0^{ext}\) on this domain, and therefore
\[
    \sup_{Q\in\mathcal P_0^G}Q(\psi_n=1)
    \le
    \sup_{Q\in\mathcal P_0^{ext}}Q(\psi_n=1)
    \le \alpha .
\]
Proposition~\ref{no_free_lunch} now yields
\[
    P(\psi_n=1)\le \alpha
    \qquad
    \text{for every }P\in\mathcal A_0 .
\]
\end{proof}

\newpage
\subsection{Proofs of Propositions \ref{proposition_with_lag_1} and \ref{proposition_with_lag_2}} 
\label{proof_of_proposition_with_lag}

\begin{customprop}{\ref{proposition_with_lag_1}}
Consider the data-generating process (\ref{structural_generation_lagged}). Then for every $p\in\mathbb{N}$, 
\begin{equation*}
    \begin{split}
     \Gamma_{\textbf{X}\to\textbf{Y}\mid\mathcal{C}}(p) = 1 &\implies \textbf{X}\overset{{\rm tail}(p)}{\longrightarrow}\textbf{Y}    
    \implies \textbf{X}\toverset{Sims}{\to}\textbf{Y}.  
    \end{split}
\end{equation*}
\end{customprop}

\begin{proof}
The proof is mostly analogous to the proofs of Propositions \ref{Proposition_CTC_iff_1} and \ref{Proposition_Granger_equivalent_tail}. 

\begin{itemize}
    \item ``$  \Gamma_{\textbf{X}\to\textbf{Y}\mid\mathcal{C}}(p) = 1 \implies \textbf{X}\overset{{\rm tail}(p)}{\longrightarrow}
 \textbf{Y} $'':  We have $$\lim_{v\to \infty}\mathbb{E}[\max\{F(Y_{t+1}), \dots, F(Y_{t+p})\}\mid \mathcal{C}^{-{X_t}}_t]< 1,$$
since $F(x)<1$ for all $x\in\mathbb{R}$. Hence, if   $\Gamma_{\textbf{X}\to\textbf{Y}\mid\mathcal{C}}(p) = 1$, then 
\begin{equation*}
    \begin{split}
\lim_{v\to \infty}\mathbb{E}[\max\{F(Y_{t+1}), \dots, F(Y_{t+p})\}\mid& X_t>v, \mathcal{C}^{-{X_t}}_t] = 1\\&\neq \lim_{v\to \infty}\mathbb{E}[\max\{F(Y_{t+1}), \dots, F(Y_{t+p})\}\mid \mathcal{C}^{-{X_t}}_t],        
    \end{split}
\end{equation*}
what we wanted to prove.

\item``$\textbf{X}\overset{{\rm tail}(p)}{\longrightarrow}\textbf{Y} \implies \textbf{X}\toverset{Sims}{\to}\textbf{Y} $'': If   $\textbf{Y}_{\future(t)} \indep X_{t}\mid \mathcal{C}^{-{X_t}}_{t}$, then 
\begin{equation*}
    \begin{split}
\lim_{v\to \infty}\mathbb{E}[\max\{F(Y_{t+1}), \dots, F(Y_{t+p})\}\mid& X_t>v, \mathcal{C}^{-{X_t}}_t]\\&=     \lim_{v\to \infty}\mathbb{E}[\max\{F(Y_{t+1}), \dots, F(Y_{t+p})\}\mid \mathcal{C}^{-{X_t}}_t].        
    \end{split}
\end{equation*}
 Hence  $X\toverset{Sims}{\not\to}Y$ implies $X\overset{{\rm tail}(p)}{\not\to}Y$. 
\end{itemize}
\end{proof}

\begin{customprop}{\ref{proposition_with_lag_2}}
Consider the structural time series
\[
\begin{split}
   \mathbf Z_t
   &= h_Z(X_{t-1},\ldots,X_{t-q_Z},
          Y_{t-1},\ldots,Y_{t-q_Z},
          \mathbf Z_{t-1},\ldots,\mathbf Z_{t-q_Z},
          \varepsilon^Z_t),\\
   X_t
   &= h_X(X_{t-1},\ldots,X_{t-q_X},
          Y_{t-1},\ldots,Y_{t-q_X},
          \mathbf Z_{t-1},\ldots,\mathbf Z_{t-q_X},
          \varepsilon^X_t),\\
   Y_t
   &= h_Y(X_{t-1},\ldots,X_{t-q_Y},
          Y_{t-1},\ldots,Y_{t-q_Y},
          \mathbf Z_{t-1},\ldots,\mathbf Z_{t-q_Y},
          \varepsilon^Y_t).
\end{split}
\]
Assume that \(h_X,h_Y,h_Z\), are upper-tail preserving. Assume further that,
for every \(t\) and every \(m\ge 1\),
\[
   (\varepsilon^X_{t+1},\ldots,\varepsilon^X_{t+m},
    \varepsilon^Y_{t+1},\ldots,\varepsilon^Y_{t+m},
    \varepsilon^Z_{t+1},\ldots,\varepsilon^Z_{t+m})
   \indep X_t \mid \mathcal C_t^{-X_t}.
\]
Then, if
\[
   \ell_t:=\min\{s\ge 1:
   Y_{t+s}\not\indep X_t\mid \mathcal C_t^{-X_t}\}
\]
exists, then
\[
   \Gamma_{\mathbf X\to \mathbf Y\mid\mathcal C}^t(p)=1
   \qquad\text{for every }p\ge \ell_t .
\]
\end{customprop}

\begin{proof}
WLOG \(t=0\) to simplify notation and write
\[
    \mathcal C:=\mathcal C_0^{-X_0}.
\]
All statements below are understood under a regular conditional law given \(\mathcal C\), and hence hold \(\mathcal C\)-almost surely. 

Call a random variable \(V\) ``active'' if
\[
   \forall c\in\mathbb R:\qquad
   \lim_{v\to\infty}
   P(V>c\mid X_0>v,\mathcal C)=1 .
\]
Call a vector of variables \(N\) ''irrelevant'' if \(N\indep X_0\mid \mathcal C\).

\textbf{Simple observations. }The variable \(X_0\) is trivially ``active''. All variables contained in \(\mathcal C\), in particular
\[
   X_{-1},X_{-2},\ldots,\qquad
   Y_0,Y_{-1},\ldots,\qquad
   \mathbf Z_0,\mathbf Z_{-1},\ldots,
\]
are ``irrelevant'', because they are \(\mathcal C\)-measurable. By assumption, every finite block of future
noise variables is also conditionally independent of \(X_0\) given \(\mathcal C\).

\textbf{Induction step. }We prove by induction over \(k\ge 0\) that each coordinate of
\[
   X_1,\ldots,X_k,\qquad
   Y_1,\ldots,Y_k,\qquad
   \mathbf Z_1,\ldots,\mathbf Z_k
\]
is either active or irrelevant. The claim is trivially true for \(k=0\).

Assume it has been proved up to time \(k\). Consider one coordinate
of \(X_{k+1}\), \(Y_{k+1}\), or \(\mathbf Z_{k+1}\). Denote it by \(V_{k+1}\). We may
write
\(
   V_{k+1}=h(A,N),
\)
where \(A\) is the vector of active lagged arguments entering the
corresponding structural equation, and \(N\) collects the irrelevant
lagged arguments, the \(\mathcal C\)-measurable arguments, and the
noise variable at time \(k+1\).

If \(h\) is constant in the coordinates \(A\), then \(V_{k+1}\) is a
measurable function only of irrelevant variables, \(\mathcal C\), and the
noise variables. Therefore \(V_{k+1}\) can be added to the jointly irrelevant
vector, so \(V_{k+1}\indep X_0\mid\mathcal C\).

If \(h\) is not constant in the coordinates \(A\), we show that
\(V_{k+1}\) is active. Fix \(c\in\mathbb R\) and \(\eta>0\). Since \(N\)
is conditionally independent of \(X_0\), there exists a compact set
\(K\) such that
\[
   P(N\in K\mid X_0>v,\mathcal C)
   =P(N\in K\mid \mathcal C)>1-\eta
\]
for all \(v\). By the upper-tail preservation assumption, there exists
\(M\in\mathbb R\) such that
\[
   h(a,n)>c
   \qquad
   \text{whenever } a_i\ge M\ \text{for all active coordinates }i,
   \text{ and } n\in K .
\]
By the induction hypothesis, every coordinate of \(A\) is active, and
therefore
\[
   \lim_{v\to\infty}
   P\bigl(\min_i A_i>M\mid X_0>v,\mathcal C\bigr)=1 .
\]
Consequently,
\[
\begin{split}
   \liminf_{v\to\infty}
   P(V_{k+1}>c\mid X_0>v,\mathcal C)
   &\ge
   \liminf_{v\to\infty}
   P\bigl(\min_i A_i>M,\ N\in K
       \mid X_0>v,\mathcal C\bigr)\ge 1-\eta .
\end{split}
\]
Since \(\eta>0\) was chosen arbitrary,
$   \lim_{v\to\infty}
   P(V_{k+1}>c\mid X_0>v,\mathcal C)=1 .$

Thus \(V_{k+1}\) is active. This completes the induction.

\textbf{Final step.} Now assume that for some \(s\le p\),
\[
   Y_s\not\indep X_0\mid\mathcal C .
\]
\(Y_s\) is not irrelevant, and hence by induction, \(Y_s\) is active, so for every \(c\in\mathbb R\),
\[
   \lim_{v\to\infty}
   P(Y_s>c\mid X_0>v,\mathcal C)=1 .
\]
Consequently, as $v\to\infty$,
\[
\begin{split}
   P\!\left(
      \max\{Y_1,\ldots,Y_p\}>c
      \mid X_0>v,\mathcal C
   \right)
   &\ge
   P(Y_s>c\mid X_0>v,\mathcal C)\to 1
   \qquad
   \forall c\in\mathbb R .
\end{split}
\]
Let \(F\) be the distribution function used in the definition of
\(\Gamma_{\mathbf X\to\mathbf Y\mid\mathcal C}(p)\). For any
\(\varepsilon>0\), choose \(c\) such that \(F(c)>1-\varepsilon\). Then
\[
\begin{split}
&\liminf_{v\to\infty}
E\!\left[
   \max\{F(Y_1),\ldots,F(Y_p)\}
   \mid X_0>v,\mathcal C
\right]  \\
&\qquad\ge
(1-\varepsilon)
\lim_{v\to\infty}
P\!\left(
   \max\{Y_1,\ldots,Y_p\}>c
   \mid X_0>v,\mathcal C
\right)
=1-\varepsilon .
\end{split}
\]
Sending \(\varepsilon\downarrow0\) gives
\[
   \Gamma_{\mathbf X\to\mathbf Y\mid\mathcal C}^t(p)=1 .
\]
This proves the result.
\end{proof}

\newpage
\subsection{Proof of Proposition~\ref{theorem_instantaneous}}
\label{section_proof_instantaneous}

\begin{customprop}{\ref{theorem_instantaneous}}
Fix \(t\in\mathbb Z\). Let \(
\tilde{\mathcal C}_t:=\mathcal C_t^{-\{X_t,Y_t\}}
\) denote the admissible information available at time \(t\), excluding \(X_t\) and \(Y_t\).
Suppose that, conditionally on \(\tilde{\mathcal C}_t\), the contemporaneous structural equations are
\[
X_t=\mu_X(\tilde{\mathcal C}_t)+\varepsilon_t^X,
\qquad
Y_t=\mu_Y(\tilde{\mathcal C}_t)+\beta X_t+\varepsilon_t^Y,
\]
where \(\mu_X\) and \(\mu_Y\) are \(\tilde{\mathcal C}_t\)-measurable and finite almost surely, and $\beta>0$. 

Assume that \(\varepsilon_t^X\), \(\varepsilon_t^Y\), and \(\tilde{\mathcal C}_t\) are mutually independent,  \(\varepsilon_t^X, \varepsilon_t^Y\in RV(\alpha)\) are compatible and define the $p=0$ coefficient 
\[
\Gamma^{t,\mathrm{inst}}_{\mathbf X\to\mathbf Y\mid\tilde{\mathcal C}}
:=
\lim_{v\to\infty}
\mathbb E\!\left[
F_Y(Y_t)\mid X_t>v,\tilde{\mathcal C}_t
\right],
\]
assuming that the limit exists a.s. Then,
\[
\Gamma^{t,\mathrm{inst}}_{\mathbf X\to\mathbf Y\mid\tilde{\mathcal C}}=1, \quad \Gamma^{t,\mathrm{inst}}_{\mathbf Y\to\mathbf X\mid\tilde{\mathcal C}}<1,
\quad
\quad\text{a.s.}
\]
\end{customprop}

\begin{proof}
Fix \(t\in\mathbb Z\) and condition on \(\tilde{\mathcal C}_t\). Since \(\mu_X(\tilde{\mathcal C}_t)\) and \(\mu_Y(\tilde{\mathcal C}_t)\) are \(\tilde{\mathcal C}_t\)-measurable, they may be treated as constants. For notational simplicity, write $\mu_X=\mu_X(\tilde{\mathcal C}_t), \mu_Y=\mu_Y(\tilde{\mathcal C}_t)$. Then
\[
X_t=\mu_X+\varepsilon_t^X,
\qquad
Y_t=\mu_Y+\beta X_t+\varepsilon_t^Y .
\]
Define the notation \[
\mathbb P(\varepsilon_t^X>u)\sim c_X u^{-\alpha}L(u),
\qquad
\mathbb P(\varepsilon_t^Y>u)\sim c_Y u^{-\alpha}L(u),
\qquad u\to\infty,
\]
where \(c_X,c_Y>0\) and \(L\) is slowly varying. 

\textbf{Direction} \(X_t\to Y_t\). Since \(\beta>0\) and \(\varepsilon_t^Y\ge 0\) almost surely, \( Y_t\ge \mu_Y+\beta X_t
\). Hence, on the event \(\{X_t>v\}\) for $v\to\infty$, \(
Y_t\ge \mu_Y+\beta v\to\infty .
\) In other words, we have
\[
\liminf_{v\to\infty}
\mathbb E\!\left[
F_Y(Y_t)\mid X_t>v,\tilde{\mathcal C}_t
\right]
\ge
\lim_{v\to\infty}F_Y(\mu_Y+\beta v)
=1.
\]
Since \(F_Y\le 1\), it follows that \(
\Gamma^{t,\mathrm{inst}}_{\mathbf X\to\mathbf Y\mid\tilde{\mathcal C}}=1
\) a.s.

\textbf{Direction} \(Y_t\to X_t\).  Substituting the structural equation for \(X_t\)
into the equation for \(Y_t\), we obtain
\[
Y_t
=
\underbrace{\mu_Y+\beta\mu_X}_{d}
+
\underbrace{\beta\varepsilon_t^X}_{U}
+
\underbrace{\varepsilon_t^Y}_{V}.
\]
Using the regular variation
\[
\mathbb P(U>u)
=
\mathbb P(\varepsilon_t^X>u/\beta)
\sim
\beta^\alpha c_X u^{-\alpha}L(u), \qquad \mathbb P(V>u)
\sim
c_Y u^{-\alpha}L(u).
\]
Since regularly varying distributions are subexponential, the single-large-jump principle gives
\[
\mathbb P(Y_t>v\mid\tilde{\mathcal C}_t)
=
\mathbb P(U+V>v-d)
\sim
\mathbb P(U>v)+\mathbb P(V>v).
\]
Therefore,
\[
\mathbb P(Y_t>v\mid\tilde{\mathcal C}_t)
\sim
\left(\beta^\alpha c_X+c_Y\right)v^{-\alpha}L(v).
\]
Since \(U=\beta\varepsilon_t^X\), the event \(\{U>v\}\) implies \(X_t\to\infty\). Hence
\[
\mathbb E\!\left[
F_X(X_t)\mathbf 1_{\{U>v\}}\mid\tilde{\mathcal C}_t
\right]
\sim
\mathbb P(U>v).
\]
On the other hand, \(V=\varepsilon_t^Y\) is independent of \(X_t\) conditionally on
\(\tilde{\mathcal C}_t\). Thus the contribution to the event \(\{Y_t>v\}\) coming from
the innovation \(V\) is weighted by the conditional mean
\[
m_X(\tilde{\mathcal C}_t)
=
\mathbb E\!\left[
F_X(X_t)\mid\tilde{\mathcal C}_t
\right]\overset{a.s.}{<}1
\]
Using again the single-large-jump principle,
\[
\mathbb E\!\left[
F_X(X_t)\mathbf 1_{\{Y_t>v\}}\mid\tilde{\mathcal C}_t
\right]
\sim
\left\{
\beta^\alpha c_X+c_Y\,m_X(\tilde{\mathcal C}_t)
\right\}
v^{-\alpha}L(v).
\]
Dividing this asymptotic expression by \(\mathbb P(Y_t>v\mid\tilde{\mathcal C}_t)\), we obtain
\[
\Gamma^{t,\mathrm{inst}}_{\mathbf Y\to\mathbf X\mid\tilde{\mathcal C}}
=
\frac{
\beta^\alpha c_X+c_Y\,m_X(\tilde{\mathcal C}_t)
}{
\beta^\alpha c_X+c_Y
}\overset{a.s.}{<} 1. 
\]
This proves the result.
\end{proof}

\end{document}